\documentclass{article}

\usepackage{fullpage}
\usepackage[]{algorithm2e}
\usepackage[noend]{algpseudocode}

\usepackage{amsthm}
\usepackage[utf8]{inputenc} 
\usepackage[T1]{fontenc}    
\usepackage{hyperref}       
\usepackage{url}            
\usepackage{booktabs}       
\usepackage{amsfonts}       
\usepackage{nicefrac}       
\usepackage{microtype}      

  \makeatletter
\def\BState{\State\hskip-\ALG@thistlm}
\makeatother

  \usepackage{mathtools}

\usepackage{titlesec}

\usepackage{tikz}
\usepackage{pgfplots}
\usetikzlibrary{pgfplots.groupplots}

\titleformat{\paragraph}
{\normalfont\normalsize\bfseries}{\theparagraph}{1em}{}
\titlespacing*{\paragraph}
{0pt}{3.25ex plus 1ex minus .2ex}{1.5ex plus .2ex}

\usepackage{movie15}

\usepackage{cite}
\usepackage{caption}
\usepackage[bottom,hang,flushmargin]{footmisc} 

\usepackage{hyperref}
\definecolor{darkred}{RGB}{150,0,0}
\definecolor{darkgreen}{RGB}{0,150,0}
\definecolor{darkblue}{RGB}{0,0,200}
\hypersetup{colorlinks=true, linkcolor=darkred, citecolor=darkgreen, urlcolor=darkblue}

\numberwithin{equation}{section}

\usepackage{graphicx,amsmath,amsfonts,amssymb,bm,hyperref,url,breakurl,epsfig,epsf,color,MnSymbol,mathbbol,fmtcount,semtrans,caption,subcaption,multirow}
\usepackage[font = scriptsize]{subcaption}
\newtheorem{theorem}{Theorem}[section]

\newtheorem{assumption}{Assumption}

\newtheorem{lemma}[theorem]{Lemma}
\newtheorem{corollary}[theorem]{Corollary}
\newtheorem{proposition}[theorem]{Proposition}
\newtheorem{definition}[theorem]{Definition}

\newcommand{\eps}{\varepsilon}

\newcommand{\beq}{\begin{equation}}
\newcommand{\eeq}{\end{equation}}

\newcommand{\cov}{{{\text{\bf{cov}}}}}
\newcommand{\clconv}{{{\text{${\overline{\bf{\text{conv}}}}$}}}}
\newcommand{\cl}{{{\text{\bf{cl}}}}}
\newcommand{\map}{{{\text{\bf{map}}}}}
\newcommand{\var}{{{\text{\bf{var}}}}}

\newcommand{\nn}{\nonumber}

\newcommand{\A}{{\mtx{A}}}
\newcommand{\bt}{\bigotimes}
\newcommand{\bd}{\bigodot}
\newcommand{\Ub}{{\mtx{U}}}

\newcommand{\B}{{{\mtx{B}}}}

\newcommand{\Sb}{{{\mtx{S}}}}

\newcommand{\diag}{\text{diag}}
\newcommand{\Lc}{{\cal{L}}}

\newcommand{\Pb}{{\mtx{P}}}
\newcommand{\Qb}{{\mtx{Q}}}
\newcommand{\Cb}{{\mtx{C}}}
\newcommand{\Hb}{{\mtx{H}}}

\newcommand{\sigmap}{\sigma'_{\ob}}
\newcommand{\bSi}{{\boldsymbol{{\Sigma}}}}
\newcommand{\bSit}{{\boldsymbol{{\tilde{\Sigma}}}}}
\newcommand{\bSih}{{\boldsymbol{{\hat{\Sigma}}}}}

\newcommand{\Db}{{\mtx{D}}}
\newcommand{\db}{{\vct{d}}}

\newcommand{\Iden}{{\mtx{I}}}
\newcommand{\M}{{\mtx{M}}}

\newcommand{\order}[1]{{\cal{O}}\left(#1\right)}
\newcommand{\rmax}[1]{{\bf{r}_{\max}(#1)}}

\newcommand{\rmin}[1]{{\bf{r}_{\min}(#1)}}
\newcommand{\z}{{\vct{z}}}
\newcommand{\zb}{{\bar{\z}}}
\newcommand{\upsilonb}{{\bar{\upsilon}}}
\newcommand{\lzero}{L_0}
\newcommand{\lip}{L}
\newcommand{\BB}{\Omega}

\newcommand{\Zbb}{{\bar{Z}}}
\newcommand{\tn}[1]{\|{#1}\|_{\ell_2}}
\newcommand{\tf}[1]{\|{#1}\|_{F}}
\newcommand{\te}[1]{\|{#1}\|_{\psi_1}}
\newcommand{\tsub}[1]{\|{#1}\|_{\psi_2}}

\newcommand{\Cc}{\mathcal{C}}

\newcommand{\Rc}{\mathcal{R}}

\newcommand{\babeta}{{\bar{\beta}}}
\newcommand{\balpha}{{\bar{\alpha}}}
\newcommand{\bgamma}{{\bar{\gamma}}}
\newcommand{\Bc}{\mathcal{B}}
\newcommand{\Sc}{\mathcal{S}}
\newcommand{\nb}{\bar{n}}

\newcommand{\pa}{{\partial}}
\newcommand{\Nn}{\mathcal{N}}

\newcommand{\vb}{\vct{v}}

\newcommand{\vbb}{\vct{\bar{v}}}

\newcommand{\abb}{\mtx{\bar{a}}}
\newcommand{\bbb}{\mtx{\bar{b}}}
\newcommand{\w}{\vct{w}}

\newcommand{\ob}{\mtx{o}}
\newcommand{\obh}{\mtx{\hat{o}}}

\newcommand{\li}{\left<}
\newcommand{\ri}{\right>}
\newcommand{\s}{\vct{s}}
\newcommand{\ab}{\vct{a}}
\newcommand{\bb}{\vct{b}}
\newcommand{\ub}{\vct{u}}
\newcommand{\ubb}{\bar{\vct{u}}}
\newcommand{\h}{\vct{h}}
\newcommand{\g}{{\vct{g}}}

\newcommand{\dd}{{\vct{d}}}

\newcommand{\Zb}{\mtx{Z}}

\newcommand{\Tc}{\mathcal{T}}
\newcommand{\Fc}{\mathcal{F}}

\newcommand{\bL}{\bar{L}}

\newcommand{\kb}{\vct{k}}
\newcommand{\xh}{\hat{\x}}

\newcommand{\wb}{\bar{\w}}
\newcommand{\Ws}{\W^\star}
\newcommand{\ws}{{\wb^\star}}
\newcommand{\wss}{{\w^\star}}

\newcommand{\FC}{{\bf{FC}}}
\newcommand{\x}{\vct{x}}
\newcommand{\rb}{\vct{r}}
\newcommand{\y}{\vct{y}}

\newcommand{\W}{\mtx{W}}
\newcommand{\bgl}{{\big |}}

\definecolor{emmanuel}{RGB}{255,127,0}

\newcommand{\Kb}{{\bar{K}}}
\newcommand{\KB}{{\bf{K}}}

\newcommand{\somelg}{\Kb\log (4\Kb)}
\newcommand{\pb}{{\vct{p}}}

\newcommand{\R}{\mathbb{R}}
\newcommand{\Pro}{\mathbb{P}}

\newcommand{\E}{\operatorname{\mathbb{E}}}

\newcommand{\grad}[1]{{\nabla\Lc(#1)}}
\newcommand{\gradc}[1]{{\nabla\Lc_{CNN}(#1)}}
\newcommand{\gradf}[1]{{\nabla\Lc_{FC}(#1)}}
\newcommand{\e}{\mathrm{e}}

\newcommand{\vct}[1]{\bm{#1}}
\newcommand{\mtx}[1]{\bm{#1}}

\newcommand{\Pc}{{\cal{P}}}
\newcommand{\X}{{\mtx{X}}}
\newcommand{\Y}{{\mtx{Y}}}
\newcommand{\Vb}{{\mtx{V}}}

\newcommand{\covdim}[1]{{\bf{\text{cover}}}(#1)}

\newcommand{\figcoef}{0.5}

\title{Learning Compact Neural Networks with Regularization}
\author{Samet Oymak\footnote{University of California, Riverside, CA. Work done at The Voleon Group, Berkeley, CA. Email: oymak@ece.ucr.edu}}

\begin{document}
\maketitle

\begin{abstract} Proper regularization is critical for speeding up training, improving generalization performance, and learning compact models that are cost efficient. We propose and analyze regularized gradient descent algorithms for learning shallow neural networks. Our framework is general and covers weight-sharing (convolutional networks), sparsity (network pruning), and low-rank constraints among others. We first introduce covering dimension to quantify the complexity of the constraint set and provide insights on the generalization properties. Then, we show that proposed algorithms become well-behaved and local linear convergence occurs once the amount of data exceeds the covering dimension. Overall, our results demonstrate that near-optimal sample complexity is sufficient for efficient learning and illustrate how regularization can be beneficial to learn over-parameterized networks.
\end{abstract}

\section{Introduction}

Deep neural networks (DNN) find ubiquitous use in large scale machine learning systems. Applications include speech processing, computer vision, natural language processing, and reinforcement learning \cite{krizhevsky2012imagenet,graves2013speech,hinton2012deep,silver2016mastering}. DNNs can be efficiently trained with first-order methods and provide state of the art performance for important machine learning benchmarks such as ImageNet and TIMIT \cite{russakovsky2015imagenet, graves2013speech}. They also lie at the core of complex systems such as recommendation and ranking models and self-driving cars \cite{covington2016deep,wang2015collaborative,bojarski2016end}.

The abundance of promising applications bring a need to understand the properties of deep learning models. Recent literature shows a growing interest towards theoretical properties of complex neural network models. Significant questions of interest include efficient training of such models and their generalization abilities. Typically, neural nets are trained with first order methods that are based on (stochastic) gradient descent. The variations include Adam, Adagrad, and variance reduction methods \cite{kingma2014adam,duchi2011adaptive,johnson2013accelerating}. The fact that SGD is highly parallellizable is often crucial to training large scale models. Consequently, there is a growing body of works that focus on the theoretical understanding of gradient descent algorithms \cite{lee2016gradient,sol2017,zhong2017recovery,tian2017analytical,soltanolkotabi2017learning,panigrahy2018convergence,ge2017learning,zhong2017learning,safran2017spurious,janzamin2015beating} and the generalization properties of DNNs \cite{kawaguchi2017generalization,zhang2016understanding,hardt2015train,bartlett2017spectrally,neyshabur2017pac,konstantinos2017pac}.

In this work, we propose and analyze regularized gradient descent algorithms to provably learn compact neural networks that have space-efficient representation. This is in contrast to existing theory literature where the focus is mostly fully-connected networks (FNN). Proper regularization is a critical tool for building models that are compact and that have better generalization properties. This is achieved by reducing degrees of freedom of the model. Sparsifying and quantizing neural networks lead to storage efficient compact models that will be building blocks intelligent mobile devices \cite{han2015deep,han2015learning,courbariaux2016binarized,denton2014exploiting,jin2016training,dong2017learning,aghasi2017net}. The pruning idea has been around for many years \cite{hassibi1993second,CunLe} however it gained recent attention due to the growing size of the state of the art DNN models. Convolutional neural nets (CNN) are also compact models that efficiently utilize their parameters by weight sharing \cite{krizhevsky2012imagenet}.

We study neural network regularization and address both generalization and optimization problems with an emphasis on one hidden-layer networks. We introduce a machinery to measure the impact of regularization, namely the covering dimension of the constraint set. We show that covering dimension controls generalization properties as well as the optimization landscape. Hence, regularization can have substantial benefit over training unconstrained (e.g.~fully-connected) models and can help with training over-parameterized networks.

Specifically, we consider the networks parametrized as $y=\ob^T\sigma(\W\x)$ where $\x\in\R^p$ is the input data, $\W\in\R^{h\times p}$ is the weight matrix, $\ob\in\R^h$ is the output layer and $h\leq p$. We assume $\W\in\Cc$ for some constraint set $\Cc$. We provide insights on the generalization and optimization performance by studying the tradeoff between the constraint set and the amount of training data ($n$) as follows.

\noindent $\bullet$ {\bf{Generalization error:}} We study the Rademacher complexity and show that good generalization is achieved when data size $n$ is larger than the sum of the covering dimension of $\Cc$ and the number of hidden nodes $h$.

\noindent $\bullet$ {\bf{Regularized first order methods:}} We propose and analyze regularized gradient descent algorithms which incorporates the knowledge of $\Cc$ to iterations. We show that problem becomes well conditioned (around ground truth parameters) once the data size exceeds the covering dimension of the constraint set. This implies the local linear convergence of first order methods with {\em{near-optimal sample complexity}}. Recent results (as well as our experiments) indicate that it is not possible to do much better than this as random initialization can get stuck at spurious local minima \cite{zhong2017recovery,safran2017spurious}.

\noindent $\bullet$ {\bf{Application to CNNs:}} We apply our results to CNNs and obtain improved global convergence guarantees when combined with the tensor initialization of \cite{zhong2017learning}.  We also improve existing local convergence results on unconstrained problem (compared to \cite{zhong2017recovery}).

\noindent $\bullet$ {\bf{Insights on layerwise learning:}} To extend our approach to deep networks, we consider learning an intermediate layer of a deep network given all others. We assume a random activation model which decouples the activations from input data in a similar fashion to Choromonska et al \cite{choromanska2015loss}. Under this simplified model, \emph{global linear convergence} occur with minimal data.


\subsection{Related Works}
Our results on the optimization landscape are closely related to the recent works on provably learning shallow neural nets \cite{sol2017,zhong2017recovery,tian2017analytical,soltanolkotabi2017learning,panigrahy2018convergence,ge2017learning,oymak2018end,zhong2017learning,safran2017spurious,arora2014provable,lee2016gradient,mei2016landscape}. Janzamin et al. proposed tensor decomposition to learn shallow networks \cite{janzamin2015beating}. Tian \cite{tian2017analytical} studies the gradient descent algorithm to train a model assuming population gradient. Soltanolkotabi et al. \cite{sol2017} focuses on training of shallow networks when they are over-parameterized and analyzes the global landscape for quadratic loss. More recently Ge et al. \cite{ge2017learning} shows global convergence of gradient descent by designing a new objective function instead of using $\ell_2$-loss. 

Our algorithmic results are closest to those of Zhong et al. \cite{zhong2017recovery}. Similar to us, authors focus on learning weights of a ground truth model where the input data is Gaussian. They propose a tensor based initialization followed by local gradient descent for learning one hidden-layer FNN. While we analyze a more general class of problems, when specialized to their setup, we improve their sample complexity and radius of convergence for local convergence. For instance, they need $\order{h^2p}$ samples to learn a FNN whereas we require $\order{hp}$ which is proportional to the {\em{degrees of freedom}} of the weight matrix. 

Growing list of works \cite{du2017gradient,du2017convolutional,brutzkus2017globally,zhong2017learning,oymak2018input} investigate CNNs with a focus on nonoverlapping filter assumption. Unlike these, we formalize CNN as a low-dimensional subspace constraint and show sample optimal local convergence even with multiple kernels and overlapping structure. As discussed in Section \ref{cnn sec}, we also improve the global convergence bounds of \cite{zhong2017learning}.

Generalization properties of deep networks recently attracted significant attention \cite{kawaguchi2017generalization,zhang2016understanding,hardt2015train,bartlett2017spectrally,neyshabur2017pac,konstantinos2017pac}. Our results are closer to \cite{bartlett2017spectrally,neyshabur2017pac,konstantinos2017pac} which studies the problem in a learning theory framework. \cite{bartlett2017spectrally,neyshabur2017pac} provide generalization bounds for deep FNNs based on spectral norm of the individual layers. More recently, \cite{konstantinos2017pac} specializes such bounds to CNNs. Our result differs from these in two ways. First, our bound reflects the impact of regularization and secondly, we avoid the dependencies on input data length by taking advantage of the Gaussian data model.

Finally, our approach borrows ideas from recent line of work on nonconvex optimization. These include low-rank factorization and sparse approximation literature \cite{tu2015low,oymak2017sharp,soltanolkotabi2017structured,ge2017no,sun2015complete,mei2016landscape} as well as standard techniques \cite{Vers,talagrand2006generic}.

\section{Problem Statement}
Here, we describe the general problem formulation. Our aim is learning neural networks that efficiently utilize their parameters by using gradient descent and proper regularization. For most of the discussion, the input/output $(y_i,\x_i)_{i=1}^n$ relation is given by
\[
y_i=\ob^T\sigma(\Ws \x_i).
\]
Here $\ob\in\R^h$ is the vector that connects hidden to output layer and $\Ws\in\R^{h\times p}$ is the weight matrix that connects input to hidden layer. Assuming $\ob$ is known we are interested in learning $\Ws$ which has $hp$ degrees of freedom. The associated loss function for the regression problem is
\[
\Lc(\W)=\frac{1}{2n}\sum_{i=1}^n (y_i-\ob^T\sigma(\W \x_i))^2.
\]
 Starting from an initial point $\W_0$, gradient descent algorithms learns $\Ws$ using the following iterations
 \[
 \W_{i+1}=\W_i-\mu \grad{\W_i}.
 \]
If we have a prior on $\Ws$, such as sparse weights, this information can be incorporated by projecting $\W$ on the constraint set. Suppose $\Ws$ lies in a constraint set $\Cc$. Denote the projection on $\Cc$ by $\Pc_\Cc(\cdot)$. Starting from an initial point $\W_0$, the {\bf{Projected Gradient Descent (PGD) algorithm}} is characterized by the following iterations
\begin{align}
 \W_{i+1}=\Pc_\Cc(\W_i-\mu \grad{\W_i}).\label{pgd algo}
\end{align}
Our goal will be to understand the impact of $\Cc$ on generalization as well as the properties of the PGD algorithm.

\subsection{Compact Models and Associated Regularizers}
In order to learn parameter-efficient compact networks, practical approaches include weight-sharing, weight pruning, and quantization as explained below.
\begin{itemize}
\item {\bf{Convolutional model (weight-sharing):}} Suppose we have a CNN with $k$ kernels of width $b$. Each kernel is shifted and multiplied with length $b$ patches of the input data i.e.~same kernel weights are used many times across the input. In Section \ref{cnn sec}, we formulate this as an FNN subject to a subspace constraint where the constraint $\Cc$ is a $kb$ dimensional subspace.
\item {\bf{Sparsity:}} Weight matrix $\Ws$ has at most $s$ nonzero weights out of $hp$ entries.
\item {\bf{Quantization:}} Weights are restricted to be discrete values. In the extreme case, entries of $\Ws$ are $\pm1$.
\item {\bf{Low-rank approximation:}} Weight matrix $\Ws$ obeys $\text{rank}(\Ws)\leq r$ for some $r\leq h$.
\end{itemize}

We also consider convex regularizers which can yield smoother optimization landscape (e.g. subspace, $\ell_1$). Convexified version of sparsity constraint is $\ell_1$ regularization. Parametrized by $\tau>0$, the constraint set is given by
\[
\Cc=\{\W\in\R^{h\times p}~\bgl~\|\W\|_1\leq \tau\}
\]
Similarly, the convexified version of low-rank projection is the nuclear norm regularization, which corresponds to the $\ell_1$ norm of singular values \cite{RechtFazel}.

\noindent Finally, we remark that our results can be specialized to the {\bf{unconstrained problem}} where the constraint set is $\Cc=\R^{h\times p}$ and PGD reduces to gradient descent.

\noindent {\bf{Notation:}} Throughout the paper, $h$ denotes the number of hidden nodes, $p$ denotes the input dimension, and $n$ denotes the number of data points unless otherwise stated. $\s_{\min}(\cdot),\s_{\max}(\cdot)$ returns the minimum/maximum singular values of a matrix. $\kappa(\Vb)$ returns the condition number of the matrix $\s_{\max}(\Vb)/\s_{\min}(\Vb)$. Similarly, for a vector $\vb$, $\kappa(\vb)=\max_{i}|\vb_i|/\min_{i}|\vb_i|$. Frobenius norm and spectral norm are denoted by $\|\cdot\|_F,~\|\cdot\|$ respectively. $c,C>0$ denote absolute constants. $\Nn(0,\Iden_d)$ will denote a vector in $\R^d$ with i.i.d.~standard normal entries. $\var[\cdot]$ returns the variance of a random variable.

\section{Main Results}
We first introduce covering numbers to quantify the impact of regularization.

\subsection{Covering Dimension}

If constraint set $\Cc$ is a $d$-dimensional subspace (e.g.~$\Cc=\R^{h\times p}$), weight matrices $\W\in\Cc$ has $d$ degrees of freedom. This model applies to convolutional and unconstrained problems. For subspaces, the dimension $d$ is sufficient to capture the problem complexity and our main results apply when the data size $n$ obeys $n\geq \order{d}$.
For other constraint types such as sparsity and matrix rank, we consider the constraint set given by
\[
\Cc=\{\W\in\R^{h\times p}~\bgl~\Rc(\W)\leq \tau\}
\]
where $\Rc$ is the regularizer function such as $\ell_1$ norm. To capture the impact of regularizer, we define \emph{feasible ball} which is the set of feasible directions given by
\begin{align}
\Tc=\Bc^{h\times p}\bigcap \cl\left(\left\{\alpha\Ub\in\R^{h\times p}~\bgl~\Ws+\Ub\in \Cc,~\alpha\geq 0\right\}\right)\label{feasible ball}
\end{align}
where $\cl(\cdot)$ is the set closure and $\Bc^{h\times p}$ is the unit Frobenius norm ball. For instance, when $\Rc$ is the $\ell_0$ norm, $\Tc$ is a subset of $\tau+\|\Ws\|_0$ sparse weight matrices.

Covering number is a standard way to measure the complexity of a set \cite{shalev2014understanding}. We will quantify the \emph{impact of regularization} by using ``covering dimension'' which is defined as follows. 
\begin{definition}[Covering dimension] \label{covdim}Let $T\subset\Bc^{h\times p}$ and $C>0$ be an absolute constant. Covering dimension of $T$ is denoted by $\covdim{T}$ and is defined as follows. Suppose there exists a set $S$ satisfying
\begin{itemize}
\item $T\subset \clconv(S)$ where $\clconv(S)$ is the minimal closed convex set containing $S$.
\item Radius of $S$ obeys $\sup_{\vb\in S}\tn{\vb}\leq C$.
\item For all $\eps>0$, $\ell_2$ $\eps$-covering number of $S$ obeys $N_\eps(S)\leq (1+\frac{B}{\eps})^s$ for some $s\geq 0, B>1$ and all $\eps>0$.
\end{itemize}
Then, $ \covdim{T}\leq s\log B$. Hence $\covdim{T}$ is the infimum of all such upper bounds.
\end{definition}

As illustrated in Table \ref{table summary}, covering dimension captures the degrees of freedom for practical regularizers. This includes sparsity, low-rank, and weight-sharing constraints discussed previously. Note that Table \ref{table summary} is obtained by setting $\tau=\Rc(\Ws)$. In practice, a good choice for $\tau$ can be found by using cross validation. It is also known that the performance of PGD is robust to choice of $\tau$ (see Thm $2.6$ of \cite{oymak2017sharp}). For unstructured constraint sets without a clean covering number, one can use stronger tools from geometric functional analysis. In Appendix \ref{perturbed width sec}, we discuss how more general complexity estimates can be achieved by using {\em{Gaussian width}} of $\Tc$ \cite{Cha} and establish a connection to covering dimension.

Our results will apply in the regime $n\gtrsim \covdim{\Tc}$ where $n$ is the number of data points. This will allow sample size to be proportional to the {\em{degrees of freedom}} of the constraint space implying data-efficient learning. Now that we can quantify the impact of regularization, we proceed to state our results.

\begin{table}
\begin{center}
    \begin{tabular}{ | l | l | l | p{10cm} |}
    \hline
    {\bf{Constraint}} & {\bf{Weight matrix model}} & {\bf{$\covdim{\Tc}$}}  \\ \hline
    None & $\Ws\in\R^{h\times p}$ & $hp$\\ \hline
    Convolutional & $k$ kernels of $b$ width & $kb$\\ \hline
    Sparsity $\|\cdot\|_0$ & $s$ nonzero weights & $s\log(6hp/s)$ \\\hline
    $\ell_1$ norm $\|\cdot\|_1$ & $s$ nonzero weights & $s\log(6hp/s)$ \\\hline
    Subspace & $\Ws\in S$,~$\text{dim}(S)=k$ & $k$\\\hline
    Matrix rank & $\text{rank}(\Ws)\leq r$ & $rh$\\
    \hline
    \end{tabular}
    \end{center}
    \caption{The list of low-dimensional models and corresponding covering dimensions (up to a constant factor) for the constraint sets $\Cc=\{\W~\bgl~\Rc(\W)\leq \Rc(\Ws)\}$. If constraint is set membership such as subspace, $\Rc(\W)=0$ inside the set and $\infty$ outside.}\label{table summary}
    \end{table}
    
\subsection{Generalization Properties}

To provide insights on generalization, we derive the Rademacher complexity of regularized neural networks with $1$-hidden layer. To be consistent with the rest of the paper, we focus on Gaussian data distribution. Rademacher complexity is a useful tool that measures the richness of a function class and that allows us to give generalization bounds. Given sample size $n$, let $\rb\in\R^n$ be an i.i.d.~Rademacher vector. Let $\{\x_i\}_{i=1}^n$ are input data points that are i.i.d.~with $\x_i\sim\Nn(0,\Iden_p)$. Finally, let $\Fc$ be the class of neural nets we analyze. Then, Rademacher complexity of $\Fc$ with respect to Gaussian data with $n$ samples is given by
\[
\text{Rad}(\Fc)=\frac{1}{n}\E_{\{\x_i\}_{i=1}^n }[\E_{\rb}[\sup_{f\in \Fc}\sum_{i=1}^n\rb_if(\x_i)]]
\]

 The following lemma provides the result on Rademacher complexity of networks with low-covering numbers.
\begin{lemma} \label{lemma rad}Suppose the activation function $\sigma$ is $\lip$-Lipschitz. Consider the class of one hidden-layer networks $\Fc$ where $f\in \Fc$ is parametrized by its input matrix $\W$ and output vector $\ob$ and satisfies
\begin{itemize}
\item input/output relation is $f_{\ob,\W}(\x)=\ob^T\sigma(\W\x)$,
\item $\|\W\|\leq R_{\W}$ and $\W\in \Cc$ where $\eps$-covering number of $\Cc$ obeys $N_{\eps}(\Cc)\leq (1+B/\eps)^s$ for some $B>0,~s\geq0$,
\item $\tn{\ob}\leq R_{\ob}$.
\end{itemize}
For Gaussian input data $ \{\x_i\}_{i=1}^n\sim \Nn(0,\Iden_p)^n$, Rademacher complexity of class $\Fc$ is bounded by
\[
\text{Rad}(\Fc)\leq \lip R_{\ob}R_{\W} \order{\frac{(h+s)\log (n+p)+s\log (1+\frac{B}{R_{\W}})}{n}}^{1/2}
\]
\end{lemma}
This result obeys typical Rademacher complexity bounds however the ambient dimension $hp$ is replaced by the total degrees of freedom which is given in terms of $h+s\log B$. Furthermore, unlike \cite{bartlett2017spectrally,neyshabur2017pac}, we do not have dependence on the length of the input data which is $\E[\tn{\x}]\approx \sqrt{p}$. This is because we take advantage of the Gaussianity of input data which allows us to escape from the worst-case analysis that suffer from $\E[\tn{\x}]$. Combined with standard learning theory results \cite{shalev2014understanding}, this bound shows that \emph{empirical risk minimization} achieves small generalization error as soon as $n\sim \order{h+s\log B}$ samples. Observe that $\order{s}$ components of $\text{Rad}(\Fc)$ relate to the covering dimension of $\Cc$ and become dominant as soon as $s\geq h$.

We remark that typically $B\sim\order{R_{\W}}$. For instance, if $\Cc$ is a $B$ scaled unit $\ell_2$ ball, in order to ensure it contains $R_{\W}$ scaled spectral ball $\{\W~\bgl~ \|\W\|\leq R_{\W}\}$, we need to pick $B=\sqrt{h}R_{\W}$.

Our main results are dedicated to the properties of the PGD algorithm where the aim is to learn compact neural nets efficiently. We show that Rademacher complexity bounds are highly consistent with the sample complexity requirements of PGD which is governed by the local optimization landscape such as positive-definiteness of the Hessian matrix.

\subsection{Local Convergence of Regularized Training}
A crucial ingredient of the convergence analysis of PGD is the positive-definiteness of Hessian along restricted directions dictated by $\Tc$ \cite{negahban2012restricted}. Denoting Hessian at the ground truth $\Ws$ by $\Hb_{\Ws}$, we investigate its restricted eigenvalue,
\[
H(\Ws,\Tc)=\inf_{\vb\in \Tc} \vb^T\Hb_{\Ws}\vb
\]
in the regime $h\leq p$. Positivity of $H(\Ws,\Tc)$ will ensure that the problem is well conditioned around $\Ws$ and is locally convergent. However, radius of convergence is not guaranteed to be large. Below, we present a summary of our results to provide basic insights about the actual technical contribution while avoiding the exact technical details.\\
$\bullet$ {\bf{Sample size:}} Whether the constraint set $\Cc$ is {\emph{convex or nonconvex}}, we have $H(\Ws,\Tc)>0$ as soon as
\[
n\geq \order{\covdim{\Tc}}.
\]
This implies \emph{sample optimal} local convergence for subspace, sparsity and rank constraints among others.

\noindent$\bullet$ {\bf{Radius of convergence:}} Basin of attraction for the PGD iterations \eqref{pgd algo} are $\order{h^{-1}}$ neighborhood of $\Ws$ i.e.~we require
\[
\tf{\W_0-\Ws}\leq \order{h^{-1}\tf{\Ws}}.
\]
As there are more hidden nodes, we require a tighter initialization. However, the result is independent of $p$. 

\noindent$\bullet$ {\bf{Rate of convergence:}} Within radius of convergence, weight matrix distance $\tf{\W_{i}-\Ws}^2$ reduces by a factor of
\[
\rho= 1-\order{\frac{1}{\max\{1,n^{-1}p\log p\}h\log p}},
\]
at each iteration, which implies \emph{linear convergence}. As long as the problem is not extremely overparametrized (i.e. $n\geq p\log p$), ignoring log terms, rate of convergence is $1-\order{1/h}$. This implies accurate learning in $\order{h\log\eps^{-1}}$ steps given target precision $\eps$.

We are now in a place to state the main results. We place the following assumptions on the activation function for our results. It is a combination of smoothness and nonlinearity
conditions. 

\newcommand{\actdefine}{
\begin{assumption} [Activation function] \label{actassume} $\sigma(\cdot)$ obeys following properties:
\begin{itemize}
\item $\sigma(\cdot)$ is differentiable, $\sigma'(\cdot)$ is an $\lip$-Lipschitz function and $|\sigma'(0)|\leq \lzero$ for some $L,\lzero>0$.
\item Given $g\sim \Nn(0,1)$ and $\theta>0$, define $\zeta(\theta)$ as 
\begin{align}
\zeta(\theta)=\min\{\var[\sigma'(\theta g)]-\E[\sigma'(\theta g)g]^2,~\var[\sigma'(\theta g)g]-\E[\sigma'(\theta g)g^2]^2\}
\end{align}
where expectations are taken with respect to $g$. $\zeta(\theta)$ obeys $\zeta(\theta)>0$.
\end{itemize}
\end{assumption}
}
\actdefine

\noindent Example functions that satisfy the assumptions are
\begin{itemize}
\item Sigmoid and hyperbolic tangent,
\item Error function $\sigma(x)=\int_0^x \exp(-t^2)dt$,
\item Squared ReLU $\sigma(x)=\max\{0,x\}^2$,
\item Softplus $\sigma(x)=\log(1+\exp(x))$ (for sufficiently large $\theta$, see Appendix \ref{softplus appendix}).
\end{itemize}

While ReLU does not satisfy the criteria, a smooth ReLU approximation such as softplus works. In general, definition of $\zeta(\cdot)$ reveals that our assumptions are satisfied if $\sigma$ i) is nonlinear, ii) is increasing, iii) has bounded second derivative, and iv) has symmetric first derivative (see Theorem $5.3$ of \cite{zhong2017recovery}).

The $\zeta(\theta)$ quantity is a measure of the nonlinearity of the activation function. It will be used to control the minimum eigenvalue of Hessian. A very similar quantity is used by \cite{zhong2017recovery} where they have an extra term which is not needed by us. This implies, our $\zeta(\theta)$ is positive under milder conditions.

\begin{definition}[Critical quantities]\label{crit quant} $\Theta$ will be used to lower bound $H(\Ws,\Tc)$ and $\BB$ will control the learning rate. They are defined as follows
\[
\Theta=\frac{\lip^2\s^2_{\max}\kappa^{2}(\ob)\kappa^{h+2}(\Ws)}{\zeta(\s_{\min})},~~~\BB=h(\log p+\frac{\lzero^2}{\lip^2\s^2_{\max}}).
\]
\end{definition}
$\Theta$ will be a measure of the conditioning of the problem. It is essentially unitless and obeys $\Theta\geq 1$ since $\lip^2\s^2_{\max}\geq \zeta(\s_{\min})$. $\BB$ will be inversely related to the radius of convergence and learning rate. If $\lzero=0$ (e.g. quadratic activation), $\BB$ simplifies to ${h\log p}$
\subsubsection{Restricted Eigenvalue of Hessian}
Our first result is a sample complexity bound for the restricted positive definiteness of the Hessian matrix at $\Ws$. It implies that problem is locally well-conditioned with minimal data ($n\sim\covdim{\Tc}$).
\begin{theorem} \label{rest hess thm}Suppose $\Cc$ is a closed set that includes $\Ws$, $h\leq p$, and let $\{\x_i\}_{i=1}^n$ be i.i.d.~$\Nn(0,\Iden_p)$ data points. Set $\upsilonb=C\Theta\log^2( C\Theta)$ and suppose
\[
n\geq\order{{(\sqrt{\covdim{\Tc}}+t)^2}{\upsilonb^4}}.
\]
With probability $1-\exp(-n/\upsilonb^2)-2\exp(-\order{\min\{t\sqrt{n},t^2\}})$, we have that\footnote{If $\Ws$ has orthogonal rows, $\kappa^{h+2}(\Ws)$ term can be removed from $\Theta$ by using a more involved analysis that uses an alternative definition of $\zeta$. The reader is referred to the Appendix \ref{section hessian nonlinear}.}
\[
H(\Ws,\Tc)\geq \frac{\zeta(\s_{\min})\ob_{\min}^2}{\kappa^{h+2}(\Ws)\upsilonb^3}.
\]
\end{theorem}
\begin{proof}[Proof sketch]
This result is a corollary of Theorem \ref{rsv hessian main}. Given a data point $\x\in\R^p$, we define $d(\x)=\ob\bd\sigma'(\Ws\x)\in\R^h$ where $\sigma'$ is the entrywise derivative and $\bd$ entrywise product. Then, define $\rho(\x)=d(\x)\bt\x\in\R^{hp}$ where $\bt$ is the Kronecker product. At ground truth, we have 
\begin{align}
H_{\Ws}=n^{-1}\sum_{i=1}^n\rho(\x_i)\rho(\x_i)^T.\label{rest eig matrix}
\end{align}
After showing $\bSi=\E[\rho(\x)\rho(\x)^T]$ is positive definite, we need to ensure that
\begin{align}
H(\Ws,\Tc)\geq \order{\s_{\min}(\bSi)}\label{small deviation}
\end{align}
 with finite sample size $n$. This boils down to a high-dimensional statistics problem. We first show that $\rho(\x)$ has subexponential tail for $\x\sim\Nn(0,\Iden)$ i.e. for all unit vectors $\vb$, $\Pro(|\vb^T(\rho(\x)-\E[\rho(\x)])|\geq t)\leq 2\exp(-C_{\sigma,\ob,\Ws}t)$. Next, Theorem \ref{rsv mean} provides a novel restricted eigenvalue result for random matrices with subexponential rows as in \eqref{rest eig matrix}.  This is done by combining Mendelson's small-ball argument with tools from generic chaining \cite{Mendel1,talagrand2006generic}. Careful treatment is necessary to address the facts that $\rho(\x)$ is \emph{not zero-mean} and \emph{its tail depends on} $\sigma,\ob,\Ws$. Our final result Theorem \ref{rsv hessian main} ensures \eqref{small deviation} with $n\geq \order{\covdim{\Tc}}$ samples where $\order{}$ has the dependencies on the aforementioned variables.
\end{proof}

\subsubsection{Linear Convergence of PGD}

Our next result utilizes Theorem \ref{rest hess thm} to characterize PGD around $\order{1/h}$ neighborhood of the ground truth. 
\begin{theorem} \label{good thm}Suppose $\Cc$ is a convex and closed set that includes $\Ws$ and let $\{\x_i\}_{i=1}^n$ be i.i.d.~$\Nn(0,\Iden_p)$ data points.. Set $\upsilonb=C\Theta\log^2( C\Theta)$ and suppose
\begin{align}
n\geq\order{{(\sqrt{\covdim{\Tc}}+t)^2}{\upsilonb^4}},\label{n bound}
\end{align}
Set $q=\max\{1,8n^{-1}p\log p\}$. Define learning rate $\mu$ and rate of convergence $\rho$ as
\begin{align}
\mu=\frac{1}{6q\ob_{\max}^2\lip^2\BB},~\rho=1-\frac{1}{12q\upsilonb^4\BB}
\end{align}
Given $\W$ (independent of data points), consider the PGD iteration
\[
\hat\W=\Pc_\Cc(\W-\mu\grad{\W})
\]
Suppose $\W$ satisfies $\tf{\W-\Ws}\leq \order{\frac{\tf{\Ws}}{q\sqrt{h\BB\log p}\upsilonb^{4}}}$. Then, $\hat\W$ obeys 
\[
\tf{\hat\W-\Ws}^2\leq \rho \tf{\W-\Ws}^2,
\]
with probability $1-P$ where $P=\exp(-n/\upsilonb^2)+2\exp(-\order{\min\{t\sqrt{n},t^2\}})+8(n\exp(-p/2)+np^{-10}+\exp(-qn/4p))$.
\end{theorem}

\subsubsection{Convergence to the Ground Truth}
Theorem \ref{good thm} shows the improvement of a single iteration. Unfortunately, it requires the existence of fresh data points at every iteration. Once the initialization radius becomes tighter ($\order{p^{-1/2}h^{-1}}$ rather than $\order{h^{-1}}$), we can show a uniform convergence result that allows $\W$ to depend on data points. Combining both, the following corollary shows that repeated applications of projected gradient converges in $\order{h\log \eps^{-1}}$ steps to $\eps$ neighborhood of $\Ws$ using $\order{\covdim{\Tc}h\log^2 p}$ samples. This is in contrast to related works \cite{zhong2017recovery,zhong2017learning} which always require fresh data points.

\begin{theorem} \label{good cor main} Consider the setup of Theorem \ref{good thm}. Let $K=\order{q\upsilonb^4\BB\log p}$. Given $\bar{n}= Kn$ independent data points (where $n$ obeys \eqref{n bound}), split dataset into $K$ equal batches. Starting from a point $\tf{\W_0-\Ws}\leq \order{\frac{\tf{\Ws}}{q\sqrt{h\BB\log p}\upsilonb^{4}}}$, apply the PGD iterations
\[
\W_{i+1}=\W_i-\mu\nabla \Lc_{\min\{i,K\}}(\W_i),
\] where $\Lc_i$ is the loss function associated with $i$th batch. With probability $1-KP$, all $\W_i$ for $i\geq 1$ obey
\[
\tf{\W_i-\Ws}^2\leq \rho^i\tf{\W_0-\Ws}^2.
\]
\end{theorem}

\section{Application to Convolutional Neural Nets}\label{cnn sec}

We now illustrate how convolutional neural networks can be treated under our framework. To describe shallow CNN, suppose we have $k$ kernels $\{\kb_i\}_{i=1}^k$ each with width $b$. Set $\KB=[\kb_1~\dots~\kb_k]^T\in\R^{k\times b}$. Denote stride size by $s$ and set $r=\lfloor p/s\rfloor$.

To describe our argument, we introduce some notation specific to the convolutional model. Let $\vb^{\ell}\in\R^b$ denote $i$th subvector of $\vb\in\R^p$ corresponding to entries from $\ell s+1$ to $\ell s+b$ for $0\leq \ell\leq r-1$. Also given $\bb\in\R^b$, let $\vb=\text{map}_{\ell}(\bb)\in\R^p$ be the vector obtained by mapping $\bb$ to the $i$th subvector i.e.
\[
\vb^j=\begin{cases}\bb~\text{if}~j=\ell\\0~\text{else}\end{cases}
\]

For each data point $\x_j$, we consider its $r$ subvectors $\{\x_{j}^l\}_{l=1}^r$ and filter each subvector with each of the kernels. Then, the input/output relation has the following form (assuming output layer weights $\ob_{i,l}$)
\[
y_{CNN}(\KB,\x_j)=\sum_{i=1}^k\sum_{l=1}^r \ob_{i,l}\sigma(\kb_i^T\x_{j}^l)
\]
Given labels $\{y_i\}_{i=1}^n$, the gradient of $\ell_2^2$-loss with respect to $\kb_i$ and $j$th label, takes the form
\begin{align}
\nabla{\Lc_{\text{CNN,j}}}(\kb_i)=\sum_{l=1}^r \ob_{i,l}(y_{CNN}(\KB,\x_j)-y_j)\sigma'(\kb_i^T\x_{j,l})\x_{j,l}\label{cnn grad}
\end{align}
We will show that a CNN can be transformed into a fully-connected network combined with a subspace constraint. This will allow us to apply Theorem \ref{good thm} to CNNs which will yield near optimal local convergence guarantees. We start by writing convolutional model as a fully-connected network. 
\begin{definition}[Convolutional weight matrix structure]\label{conv struct} Set $h=kr$. Given kernels $\{{\kb}_i\}_{i=1}^k$, we construct the fully-connected weight matrix $\W=\FC(\KB)\in\R^{h\times p}$ as follows: Representing $\{1,\dots,h\}$ as cartesian product of $\{1,\dots,k\}$ and $\{1,\dots,r\}$, define the $h=kr$ rows $\{\w_{i,j}\}_{(i,j)=(1,1)}^{(k,r)}$ of the weight matrix $\W$ as
\[
\w_{i,j} = \map_j(\kb_i)
\]
for $1\leq i\leq r$ and $1\leq l\leq k$. Finally let $\Cc$ be the space of all convolutional weight matrices defined as 
\[\Cc=\{\FC(\KB)~\bgl~\{\kb_i\}_{i=1}^k\in\R^b\}.\]
\end{definition}

\noindent This model yields a matrix $\W$ that has double structure:
\begin{itemize}
\item Each row of $\W$ has at most $b=p/r$ nonzero entries.
\item For fixed $i$, the weight vectors $\{\w_{i,l}\}_{l=1}^r$ are just shifted copies of each other.
\end{itemize}
This implies the total degrees of freedom is same as $\{\kb_i\}_{i=1}^k$ and $\Cc$ is a $kb$ dimensional subspace.
\begin{lemma} \label{conv space}Convolutional weight matrix space $\Cc$ is a $kb$ dimensional linear subspace.
\end{lemma}

Next, given $\W=\FC(\KB)$, observe the equality of the predictions i.e.
\[
y_{\FC}(\W,\x_j)=\sum_{i=1}^k\sum_{l=1}^r \ob_{i,l}\sigma(\w_{i,l}^T\x)=y_{CNN}(\KB,\x_j)
\]
Similarly for $\W=\FC(\KB)$, one can also show the equality of the CNN gradient and projected FNN  gradient.
\begin{lemma} \label{conv grad}Recall convolutional gradient \eqref{cnn grad}. Given kernels $\KB=[\kb_1~\dots~\kb_k]^T$ and corresponding FNN weight matrix $\W=\FC(\KB)$, we have
\[
\Pc_\Cc(\gradf{\W})=\frac{1}{r}\FC(\gradc{\KB})
\]
Consequently, setting $\W=\FC(\KB)$, and considering CNN and FNN gradient iterations
\[
\hat\KB = \KB-\frac{\mu}{r}\gradc{\KB},~\hat\W=\Pc_\Cc(\W-\mu\gradf{\W}),
\]
We have the equality $\FC(\hat\KB)=\hat\W$.
\end{lemma}
With this lemma at hand, we obtain the following corollary of Theorem \ref{good thm}.
\begin{corollary} \label{main conv thm}Let $\{\x_i\}_{i=1}^n$ be i.i.d.~$\Nn(0,\Iden_p)$ data points. Given $k$ kernels $\KB^\star=[\kb^\star_1~\dots~\kb_k^\star]^T$ and generate the labels
\[
y_j=y_{CNN}(\KB^\star,\x_j)
\]
Assume $\FC(\KB^\star)$ is full row-rank and let $\Theta,\BB,\upsilonb,q,\mu,\rho,P$ be same as in Theorem \ref{good thm} defined with respect to the matrix $\FC(\KB^\star)$. Suppose $n\geq\order{(\sqrt{kb}+t)^2/{\upsilonb^4}}$
and consider the convolutional iteration
\[
\hat{\KB}=\KB-\frac{\mu}{r}\gradc{\KB}.
\]
Suppose the initial point $\KB=[\kb_1~\dots~\kb_k]^T$ satisfies $\tf{\KB-\KB^\star}\leq \order{\frac{\tf{\KB^\star}}{q\sqrt{h\BB\log p}\upsilonb^{4}}}$. Then, with $1-P$ probability,
\[
\tf{\hat\KB-\KB^\star}^2\leq \rho \tf{\KB-\KB^\star}^2.
\]
\end{corollary}
\begin{proof} Convolutional constraint space $\Cc$ is a $kb$ dimensional subspace and we apply Theorem \ref{good thm} on learning fully-connected weight matrix corresponding to CNN to get a convergence result on $\Ws=\FC(\KB^\star)$. We use the fact that $\Ws\in\Cc$ and apply Lemma \ref{conv grad} which states that projected gradient iterations are equivalent to the gradient iterations of CNN. Consequently results for $\Ws$ can be translated to $\KB^\star$ via $\hat\W=\FC(\hat\KB)$. To conclude, we use the facts that $\Cc$ is a subspace, $\covdim{\Cc}=kb$, and $\frac{\tf{\Ws-\W}}{\tf{\Ws}}=\frac{\tf{\KB^\star-\KB}}{\tf{\KB^\star}}$.
\end{proof}
This corollary can be combined with the results of \cite{zhong2017learning} to obtain a globally convergent CNN learning algorithm using $n\sim \order{\text{poly}(k,t,\log p)}$ samples. The algorithm of \cite{zhong2017learning} uses gradient descent around the local neighborhood of $\KB^\star$ following a tensor initialization. For local convergence \cite{zhong2017learning} needed $n\geq \order{p}$ samples whereas we show that there is no dependence on the data length $p$.




\subsection{Minimum singular value of the convolutional weight matrix}\label{section cnn min sing}
Clearly, Corollary \ref{main conv thm} implicitly assume the condition number of the convolutional weight matrix $\FC(\KB^\star)$. Luckily, we can give closed-form and explicit bounds for this in two scenarios. The bound will be in terms of the intrinsic properties of the kernels $\{\kb^\star_i\}_{i=1}^k$. The first scenario is when kernels $\{\kb^\star_i\}_{i=1}^k\}$ are non-overlapping. The second scenario is when there is a single kernel $\KB^\star=\kb^\star_1$ and overlap is allowed. The following lemma illustrates this bounds.
\begin{lemma} Given $\KB^\star=[\kb^\star_1~\dots~\kb_k^\star]^T$, define the FCNN weight matrix $\Ws=\FC(\KB^\star)$.
\begin{itemize}
\item If CNN is non-overlapping (i.e.~$s\geq b$), $\FC(\KB^\star)$ is full row-rank if and only if $\KB^\star$ is full row-rank (which requires $k\leq b$). Furthermore,
\[
\s_{\min}(\Ws)=\s_{\min}(\KB^\star),~\s_{\max}(\Ws)=\s_{\max}(\KB^\star)
\]
\item Suppose there is a single kernel $\KB^\star\in \R^b$. Zero-pad $\KB^\star$ to obtain $\wss\in\R^p$, which is the first row of $\Ws$. Let ${\bf{f}}=\text{DFT}(\wss)$ be the discrete Fourier transform of $\wss$. Independent of stride $s$, we have that
\[
\max_{1\leq i\leq p}|{\bf{f}}_i|\geq \s_{\max}(\Ws),~ \s_{\min}(\Ws)\geq\min_{1\leq i\leq p}|{\bf{f}}_i|
\]
\end{itemize}
\end{lemma}
\begin{proof} The first result follows from the fact that when $s\geq b$ (non-overlapping CNN), $\Ws$ can be reorganized as block diagonal matrix with $r$ blocks where each block is $\KB^\star$. For the second result, observe that $\Ws=\FC(\KB^\star)$ is a subsampled circulant matrix formed out of $\wss$. In particular,
\[
\Ws=\Sb\Ws_{full}~~~\text{where}~~~\Ws_{full}={\bf{F}}\diag(\text{DFT}(\wss)){\bf{F}}^H={\bf{F}}\diag({\bf{f}}){\bf{F}}^H.
\]
Here ${\bf{F}}$ is usual DFT matrix scaled by $1/\sqrt{p}$ and $\Sb$ is the row-selection matrix that samples the $1,1+s,\dots,1+(r-1)s$'th rows. Singular values of $\Ws_{full}$ are simply the absolute values of ${\bf{f}}$. Since $\Ws$ is obtained by row-subsampling, we have
\[
\max_{1\leq i\leq p}|{\bf{f}}_i|= \s_{\max}(\Ws_{full})\geq \s_{\max}(\Ws)\geq \s_{\min}(\Ws)\geq \s_{\min}(\Ws_{full})=\min_{1\leq i\leq p}|{\bf{f}}_i|
\]
\end{proof}

\section{Learning deep networks layerwise}\label{sec layers}
So far, we demonstrated the local linear convergence of the PGD algorithm for shallow networks. A natural question is whether similar framework can be extended to deeper networks. Specifically, we are interested in learning a particular layer of a deep network given all other layers. To address this, we consider a \emph{simplified model} where we assume activations and the input data is independent in a similar fashion to Choromanska et al. \cite{choromanska2015loss}. At each layer, activation function will randomly modulate its input by $\pm 1$ multiplication. While this model is not realistic, we believe it provides valuable insight on what to expect for deeper networks. Below is the definition of the deep random activation model we study.

\begin{definition} [Random activation model] \label{randact}Consider an $D$ hidden layer network characterized by $\ob\in \R^{h_D},\Ws_i\in\R^{h_i\times h_{i-1}}$ for $0\leq i\leq D-1$ where the input/output $(y_i,\x_i)$ relation is given by
\[
y_i=\ob^T\sigma_{D,i}(\Ws_{D-1}(\sigma_{D-1,i}(\dots\sigma_{1,i}(\Ws_0\x_i)))).
\]
Here $\{\sigma_{l,i}(\cdot)\}_{(l,i)=(1,1)}^{(D,n)}$ are parametrized by the independent vectors $\{\rb_{l,i}\}_{(l,i)=(1,1)}^{(D,n)}$ which have independent Rademacher entries. In particular, $\sigma_{i,l}$ obeys $\sigma_{D,i}(\vb)=\rb_{l,i}\bd\vb$ for all $\vb$ where $\bd$ is the entrywise (Hadamard) product. 
\end{definition}
This model greatly simplifies the analysis because activations are decoupled from the input data. We next introduce a quantity that captures the condition number of for learning a particular layer.
\begin{definition} [Network condition number] \label{random act crit} Given a matrix $\Vb$ with rows $\{\vb_i\}_{i=1}^d$, let 
\[
\kappa_{\text{row}}(\Vb)={\|\Vb\|}~/~{\min\{\tn{\vb_i}\}_{i=1}^d}.
\]
Define the condition number of the $\ell$th layer as
\[
\bar{\kappa}_\ell=\kappa(\ob)\prod_{ j=0}^{\ell-1} \kappa_{\text{row}}(\Ws_j)\prod_{ j=\ell+1}^{D} \kappa_{\text{row}}({\Ws_j}^T).
\]
\end{definition}
In words, $\kappa_{\text{row}}(\cdot)$ is the spectral norm normalized by the smallest row length. $\bar{\kappa}_{\ell}$ is essentially the multiplication of row condition numbers of all matrices except the $\ell$th matrix. The following theorem is our main result on layerwise learning of DNNs and provides a global convergence guarantee that is based on the condition number ${\bar{\kappa}}_{\ell}$ and constraint set dimension $\covdim{\Tc}$.

\begin{theorem}[Learning $\ell$th layer of DNN] \label{rand act main thm}Consider the random activation model described in Definitions \ref{randact} and \ref{random act crit}. Suppose $\{\rb_{l,i}\}_{(l,i)=(1,1)}^{(D,n)}$, $\ob$ and $\{\Ws_i\}_{i\neq \ell}$ are known and $\Cc$ is a closed and convex set. We estimate $\Ws_{\ell}\in \Cc$ via PGD iterations
\begin{align}
\W_{i+1}=\Pc_{\Cc}(\W_{i}-\mu\grad{\W_i})\label{deep pgd}
\end{align}
where the gradient is with respect to the $\ell$th layer. Define $q=\max\{1,n^{-1}h_{\ell+1}h_{\ell}\log (h_{\ell+1}h_{\ell})\}$ and $\upsilonb=C\bar{\kappa}^2_{\ell}\log^2(C\bar{\kappa}^2_{\ell})$ for some large constant $C>0$. Pick learning rate $\mu=\frac{1}{6q\bgamma_{\ell}}$. Assuming
\[
n\geq  \max\{\upsilonb^4(\sqrt{\covdim{\Tc}}+t)^2\}
\]
and starting from an arbitrary point $\W_0$, all PGD iterations \eqref{deep pgd} obey
\[
\tn{\W_{i+1}-\Ws}^2\leq (1-\frac{1}{q\upsilonb^{4}})^i\tn{\W_0-\Ws}^2
\]
with probability $1-\exp(-qn/h_{\ell+1}h_{\ell})-n\exp(-\order{h_{\ell}})-n\exp(-\order{h_{\ell+1}})-\exp(-n/\upsilonb^2)-2\exp(-\order{\min\{t\sqrt{n},t^2\}})$.
\end{theorem}
Compared to Theorem \ref{good thm}, we observe that in the overdetermined regime $n\geq \order{h_{\ell+1}h_{\ell}\log (h_{\ell+1}h_{\ell})}$, the rate of convergence becomes $1-\order{1}$ which is independent of the dimensions $h_{\ell},h_{\ell+1}$. This is due to the fact that random activations result in a better conditioned problem.

\section{Numerical Results}\label{simulation}
To support our theoretical findings, we present numerical performance of sparsity and convolutional constraints for neural network training. We consider synthetic simulations where $\ob$ is a vector of all ones and weight matrix $\Ws\in\R^{h\times p}$ is sparse or corresponds to a CNN.

\begin{figure}[t]
    \begin{subfigure}[b]{\figcoef\textwidth}
        \includegraphics[width=1.0\textwidth]{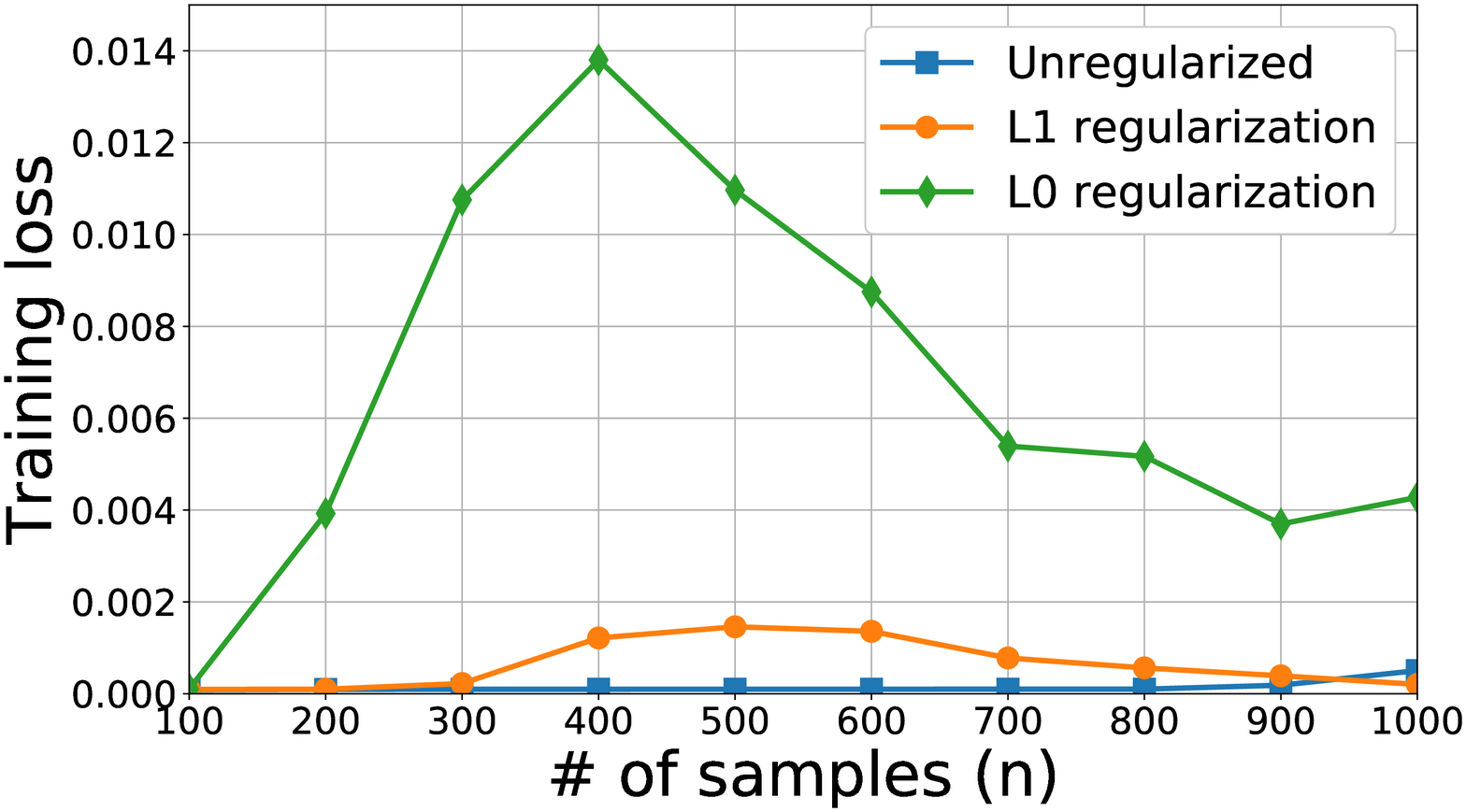}
        \label{fig:train2}
    \end{subfigure}~
 \begin{subfigure}[b]{\figcoef\textwidth}
        \includegraphics[width=1.0\textwidth]{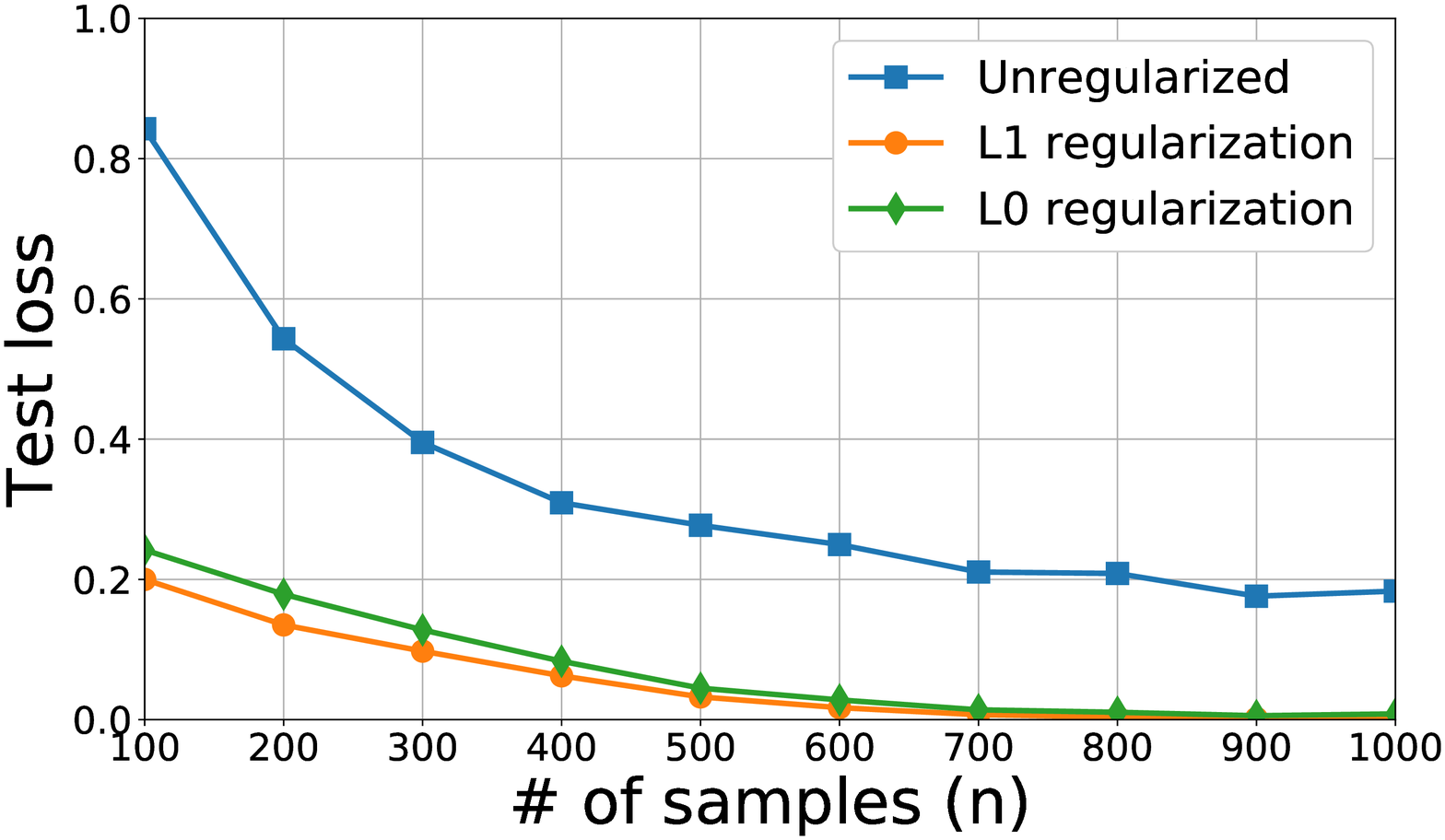}
        \label{fig:test2}
    \end{subfigure} \\
    \begin{subfigure}[b]{\figcoef\textwidth}
        \includegraphics[width=1.0\textwidth]{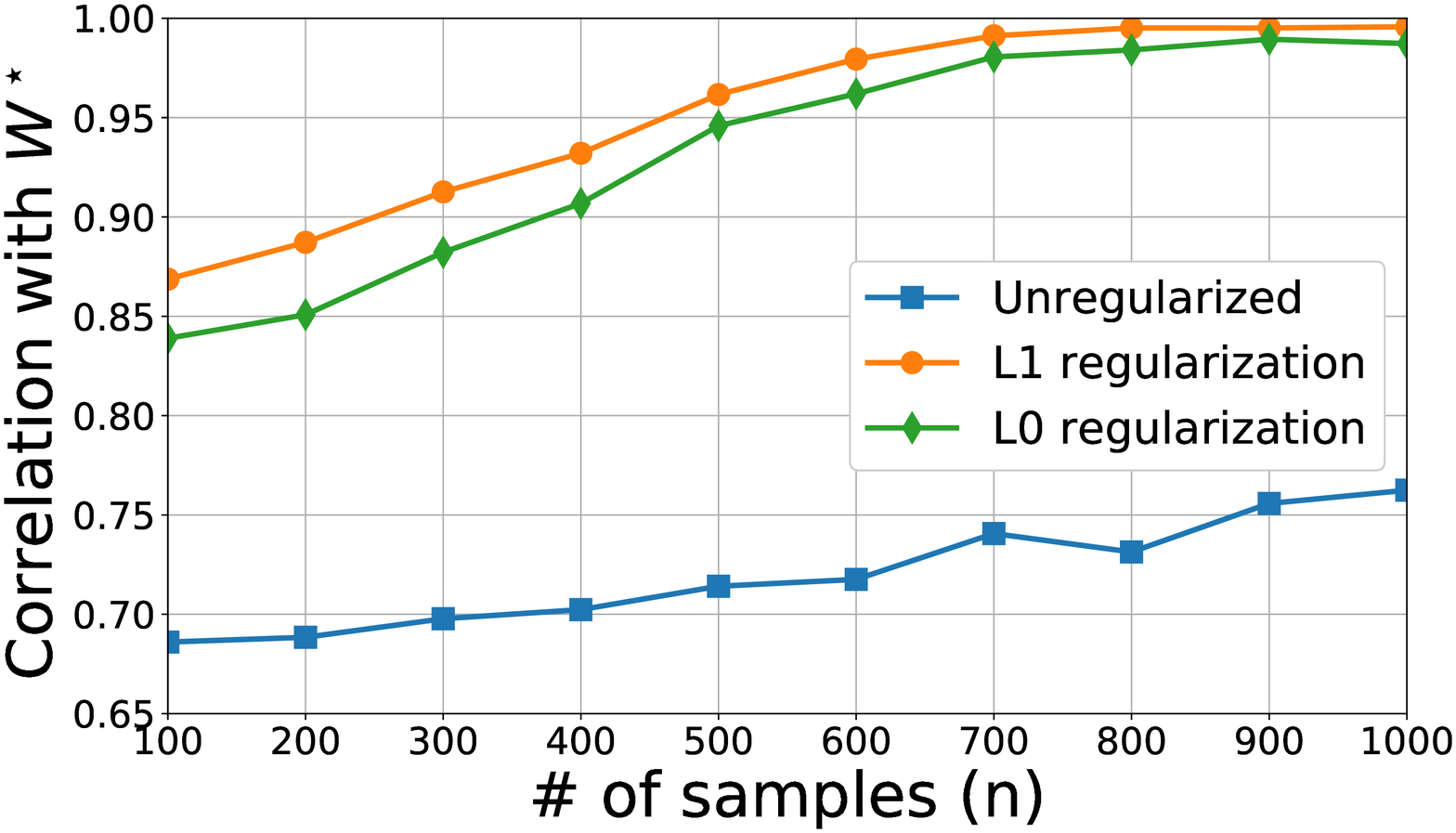}
        \label{fig:rec2}
    \end{subfigure}~
 \begin{subfigure}[b]{\figcoef\textwidth}
        \includegraphics[width=1.0\textwidth]{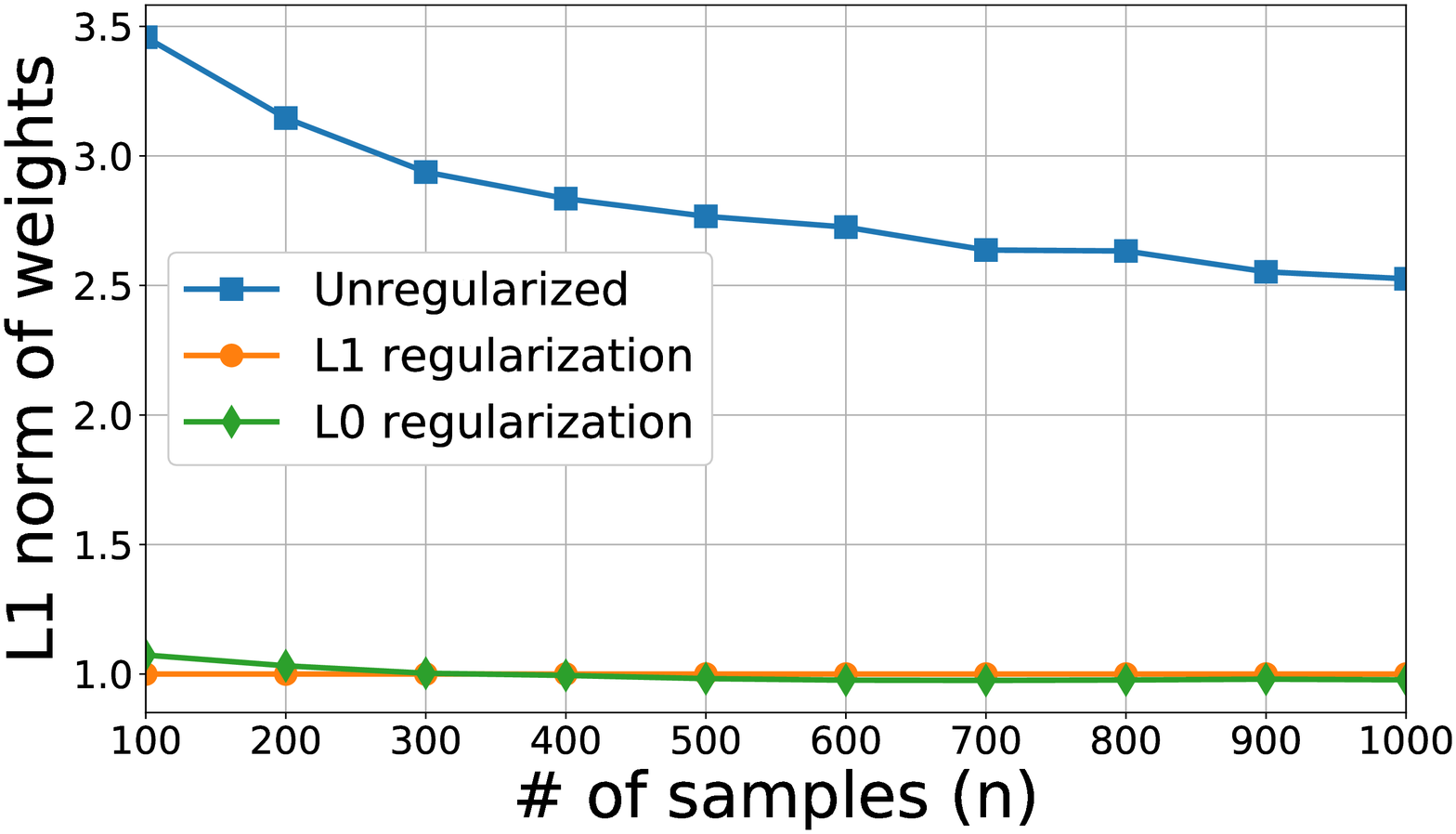}
        \label{fig:l12}
    \end{subfigure}\caption{Experiments with good initialization $\W_0=\Ws+\Zb$.}\label{fig:good}
\end{figure}

\subsection{Sparsity Constraint}
We generate $\Ws$ matrices with exactly $s$ nonzero entries at each row and nonzero pattern is distributed uniformly at random. Each entry of $\Ws$ is $\Nn(0,\frac{p}{hs})$ to ensure $\E[\tn{\Ws\x}^2]=\tn{\x}^2$. We set the learning rate to $\mu=5$. We verified that smaller learning rate leads to similar results with slower convergence. We declare the estimate $\hat\W$ to be the output of PGD algorithm after $2000$ iterations. We consider two sets of simulations using ReLU activations. 
\begin{itemize}
\item {\bf{Good initialization:}} We set $\W_0=\Ws+\Zb$ where $\Zb$ has i.i.d.~$\Nn(0,\frac{1}{h})$ entries. Note that noise $\Zb$ satisfies $\E[\|\Zb\|_F^2]=\E[\|\Ws\|_F^2]$.
\item {\bf{Random initialization:}} We set $\W_0=\Zb$ where $\Zb$ has i.i.d.~$\Nn(0,\frac{1}{h})$ entries. 
\end{itemize}
Each set of experiments consider three algorithms.
\begin{itemize}
\item {\bf{Unconstrained:}} Only uses gradient descent.
\item {\bf{$\ell_1$-regularization:}} Projects $\W$ to $\ell_1$ ball scaled by the $\ell_1$ norm of $\Ws$.
\item {\bf{$\ell_0$-regularization:}} Projects $\W$ to set of $sh$ sparse matrices.
\end{itemize}

\begin{figure}[t]
    \begin{subfigure}[b]{\figcoef\textwidth}
        \includegraphics[width=\textwidth]{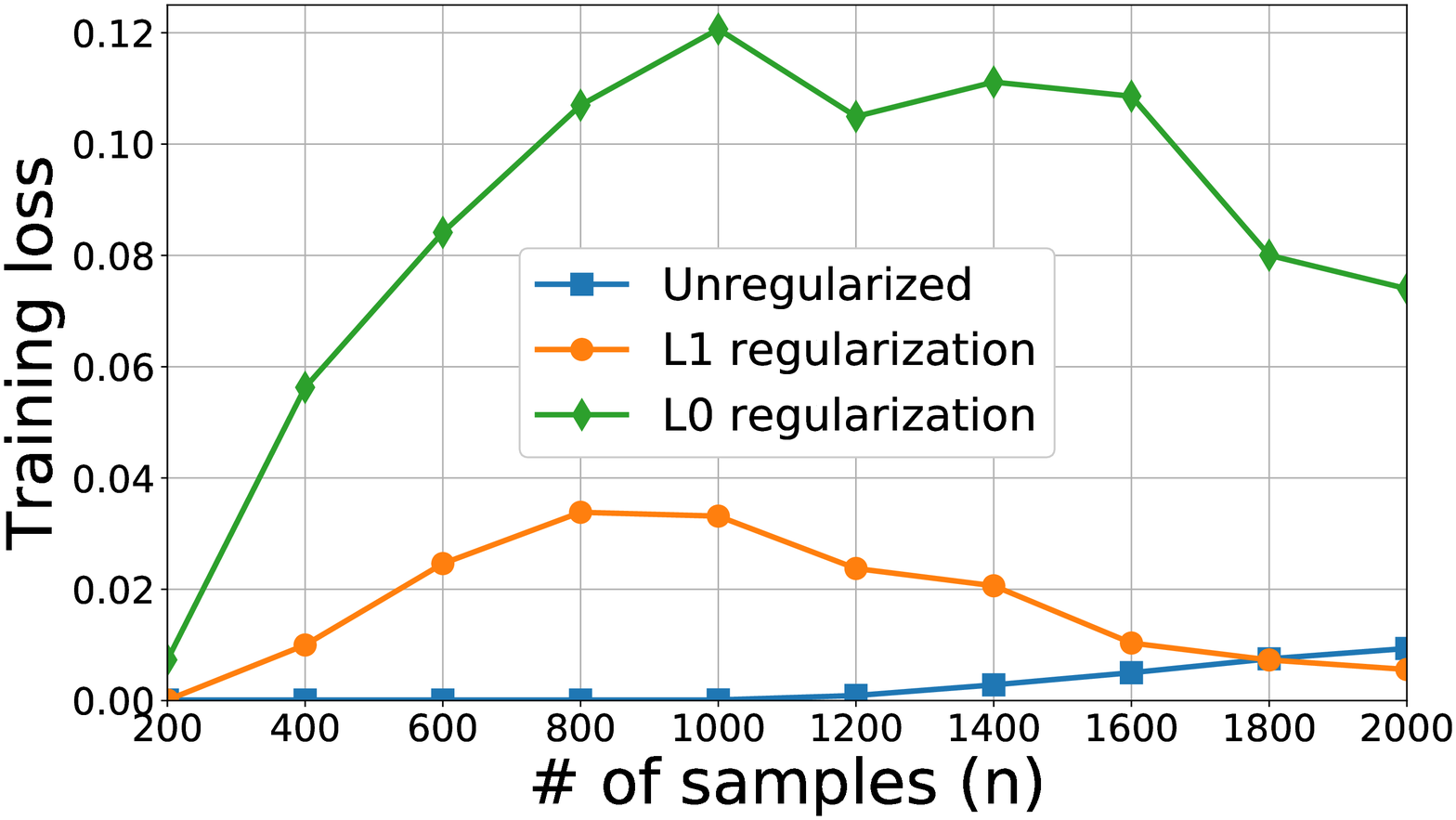}
        \label{fig:train1}
    \end{subfigure}~
 \begin{subfigure}[b]{\figcoef\textwidth}
        \includegraphics[width=\textwidth]{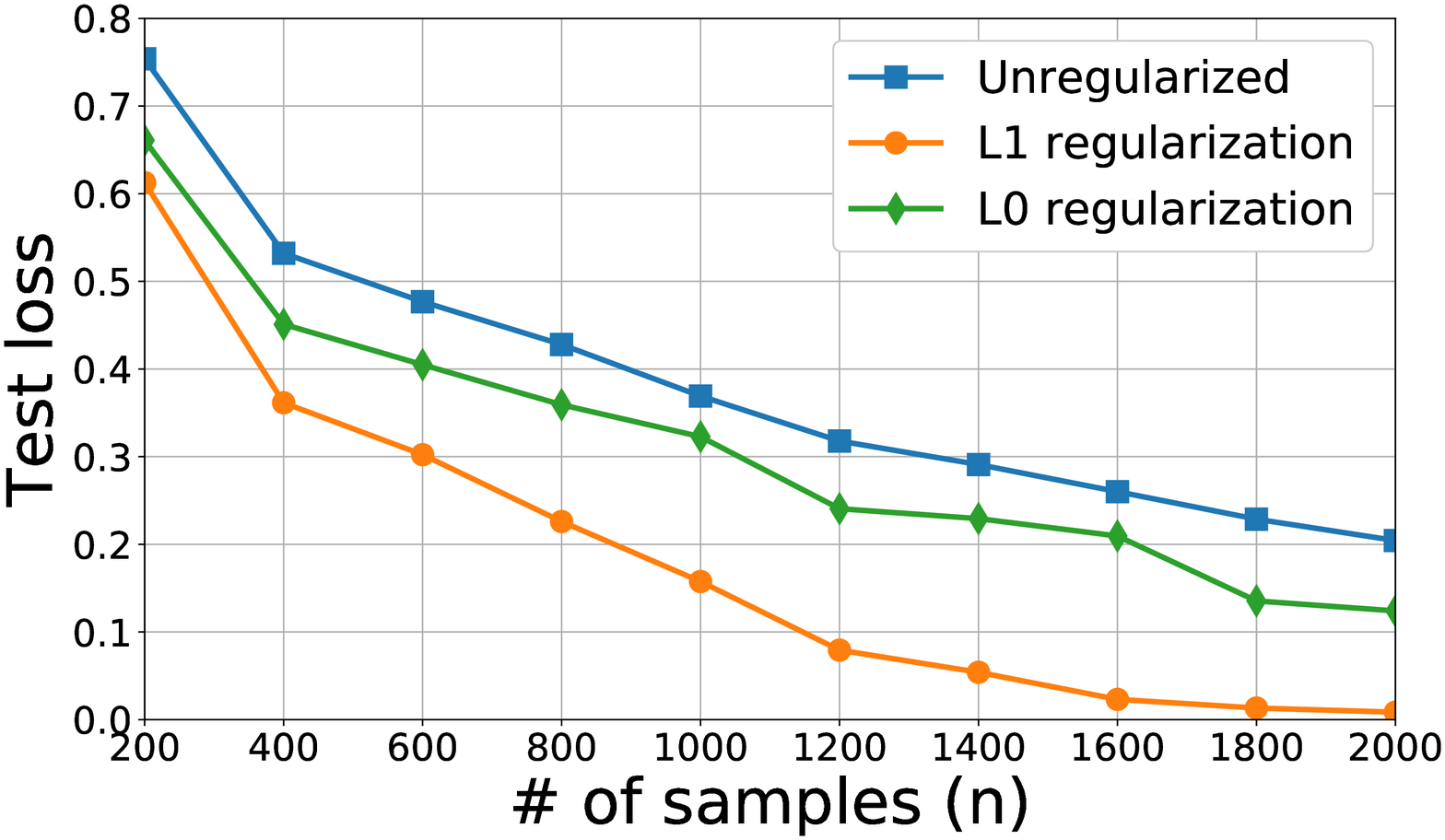}
        \label{fig:test1}
    \end{subfigure} \\
    \begin{subfigure}[b]{\figcoef\textwidth}
        \includegraphics[width=\textwidth]{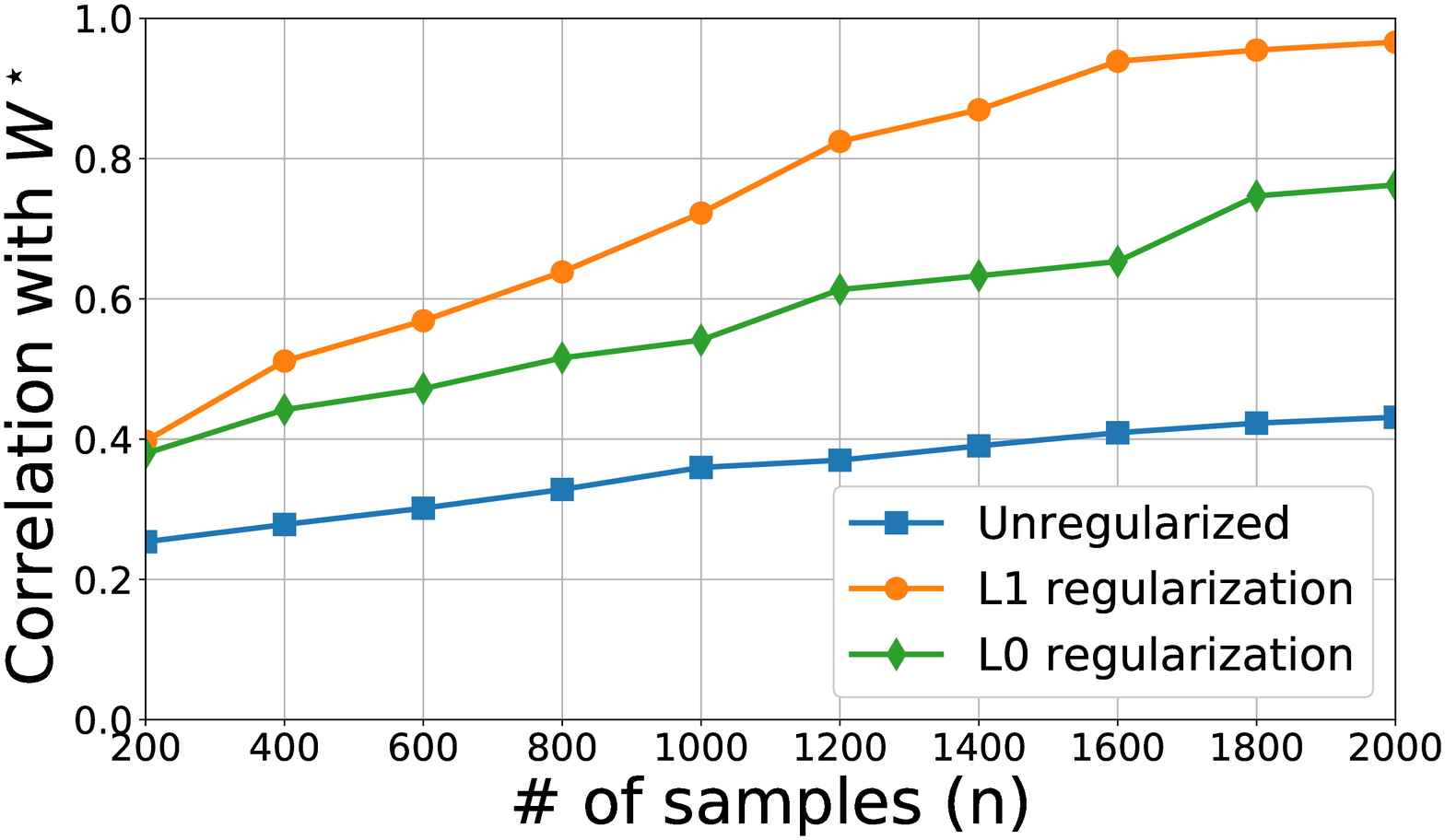}
        \label{fig:rec1}
    \end{subfigure}~
 \begin{subfigure}[b]{\figcoef\textwidth}
        \includegraphics[width=\textwidth]{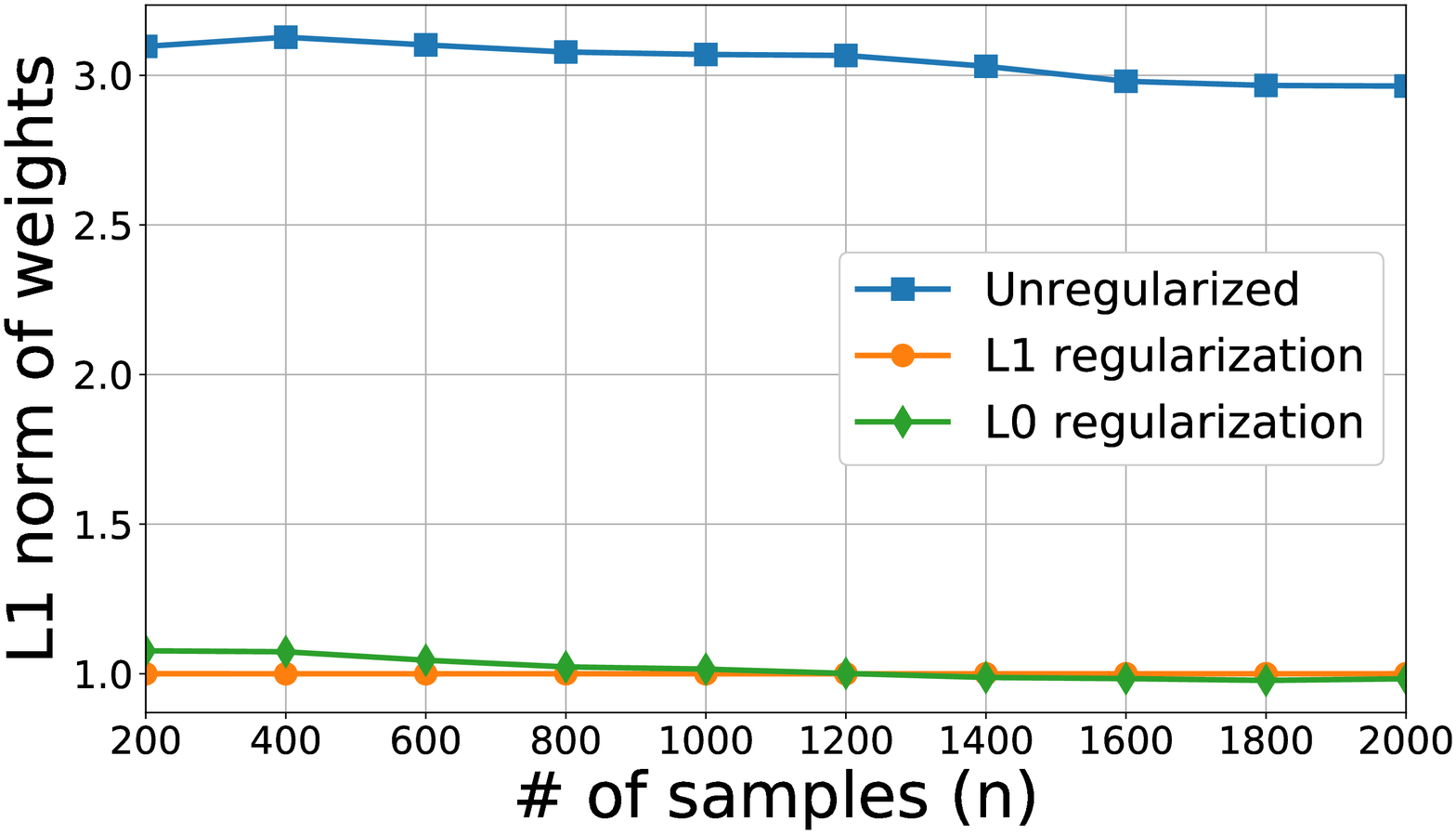}
        \label{fig:l11}
    \end{subfigure}  \caption{Experiments with random initialization $\W_0=\Zb$.} \label{fig:random}
\end{figure}
For our experiments, we picked $p=80$, $h=20$ and $s=p/10=8$. For training, we use $n$ data points which varies from $100$ to $1000$. Test error is obtained by averaging $n_{\text{test}}=1000$ independent data points. For each point in the plots, we averaged the outcomes of $20$ random trials. The total degrees of freedom is the number of nonzeros equal to $sh=160$. Our theorems imply good estimation via $\order{sh\log p/s}$ data points when initialization is sufficiently close. Figure \ref{fig:good} summarizes the outcome of the experiments with good initialization. Suppose $y$ is the label and $\hat{y}$ is the prediction. We define the (normalized) test and train losses as the ratio of empirical variances that approximates the population $\frac{\var[y-\hat{y}]}{\var[y]}$. Centering (i.e.~variance) is used to eliminate the contribution of trivial but large $\E[y]$ term due to nonnegative ReLU outputs. First, we observe that $\ell_1$ is slightly better than $\ell_0$ constraint however both approach $\approx0$ test loss when $n\geq 600$. Unregularized model has significant test error for all $100\leq n\leq 1000$ while perfectly overfitting training set for all $n$ values. We also consider the recovery of ground truth $\Ws$. Since there is permutation invariance (permuting rows of $\W$ doesn't change the prediction), we define the correlation between $\Ws$ and $\hat \W$ as follows,
\[
\text{corr}(\Ws,\hat\W)=\frac{1}{h}\sum_{i=1}^h \max_{1\leq j\leq h} \frac{\li\w^\star_i,\hat\w_j\ri}{\tn{\w^\star_i}\tn{\hat\w_j}}
\]
where $\w_i$ is the $i$th row of $\W$. In words, each row of $\Ws$ is matched to the highest correlated row from $\hat\W$ and correlations are averaged over $h$ rows. Observe that, if $\hat\W$ and $\Ws$ have matching permutations, $\text{corr}(\Ws,\hat\W)=1$. We see that $\text{corr}(\Ws,\hat\W)\approx 1$ once $n\geq 600$ which is the moment test error hits $0$.

Figure \ref{fig:random} summarizes the outcome of the experiments with random initialization. In this case, we vary $n$ from $200$ to $2000$ but the rest of the setup is the same. We observe that unlike good initialization, $\ell_0$ test error and $1-\text{corr}(\Ws,\hat\W)$ does not hit $0$ and $\ell_1$ approaches $0$ only at $n=2000$. On the other hand, both metrics demonstrate the clear benefit of sparsity regularization. The performance gap between $\ell_1$ and $\ell_0$ is surprisingly high however it is consistent with Theorem \ref{good thm} which only applies to convex regularizers. The performance difference between good and random initialization implies that initialization indeed plays a big role not only for finding the ground truth solution $\Ws$ but also for achieving good test errors.

\begin{figure}[t]
    \begin{subfigure}[b]{\figcoef\textwidth}
        \includegraphics[width=\textwidth]{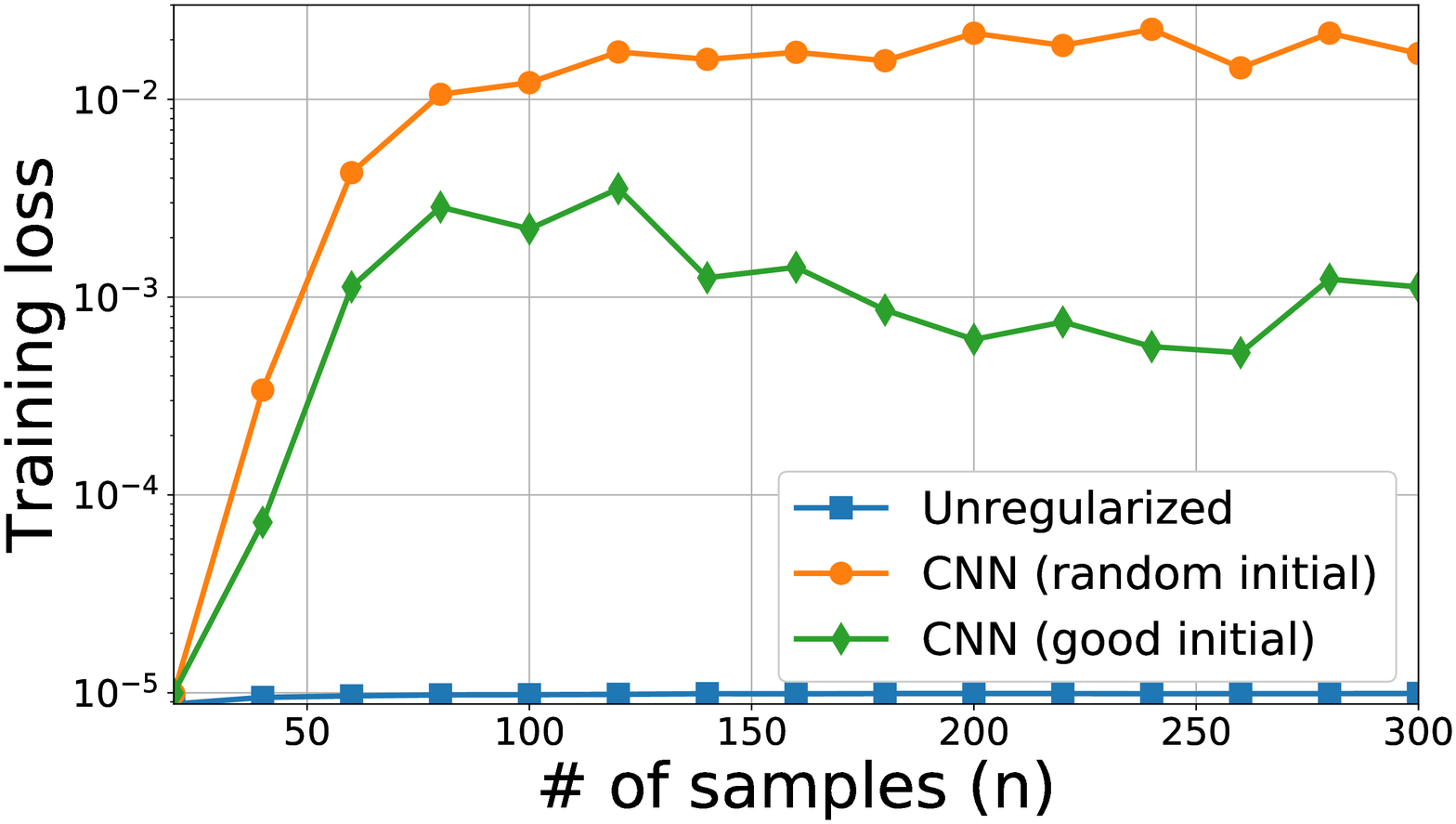}
        \label{fig:train3}
    \end{subfigure}~
 \begin{subfigure}[b]{\figcoef\textwidth}
        \includegraphics[width=\textwidth]{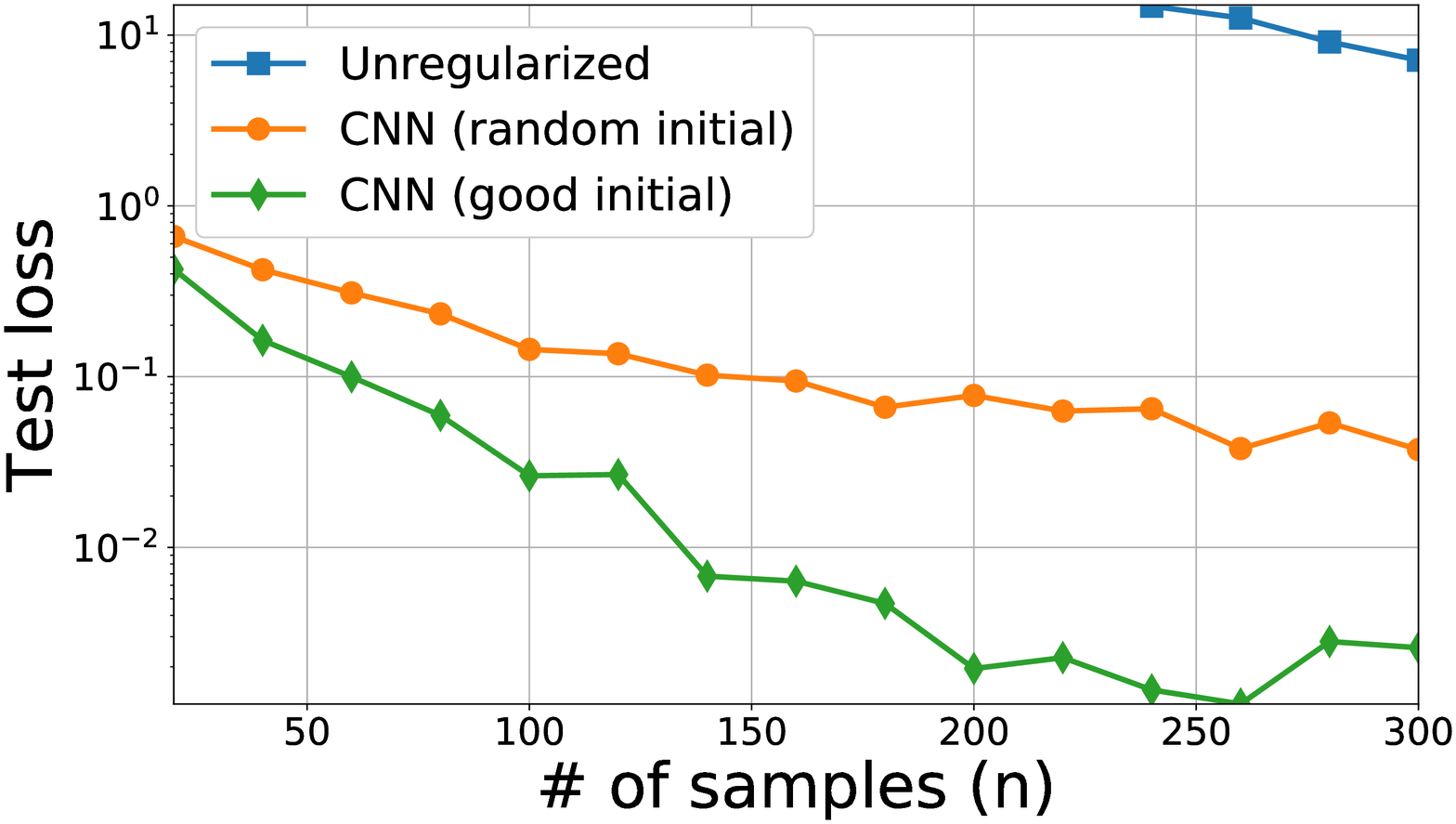}
        \label{fig:test3}
    \end{subfigure} \\
    \begin{center}\begin{subfigure}[b]{\figcoef\textwidth}
        \includegraphics[width=\textwidth]{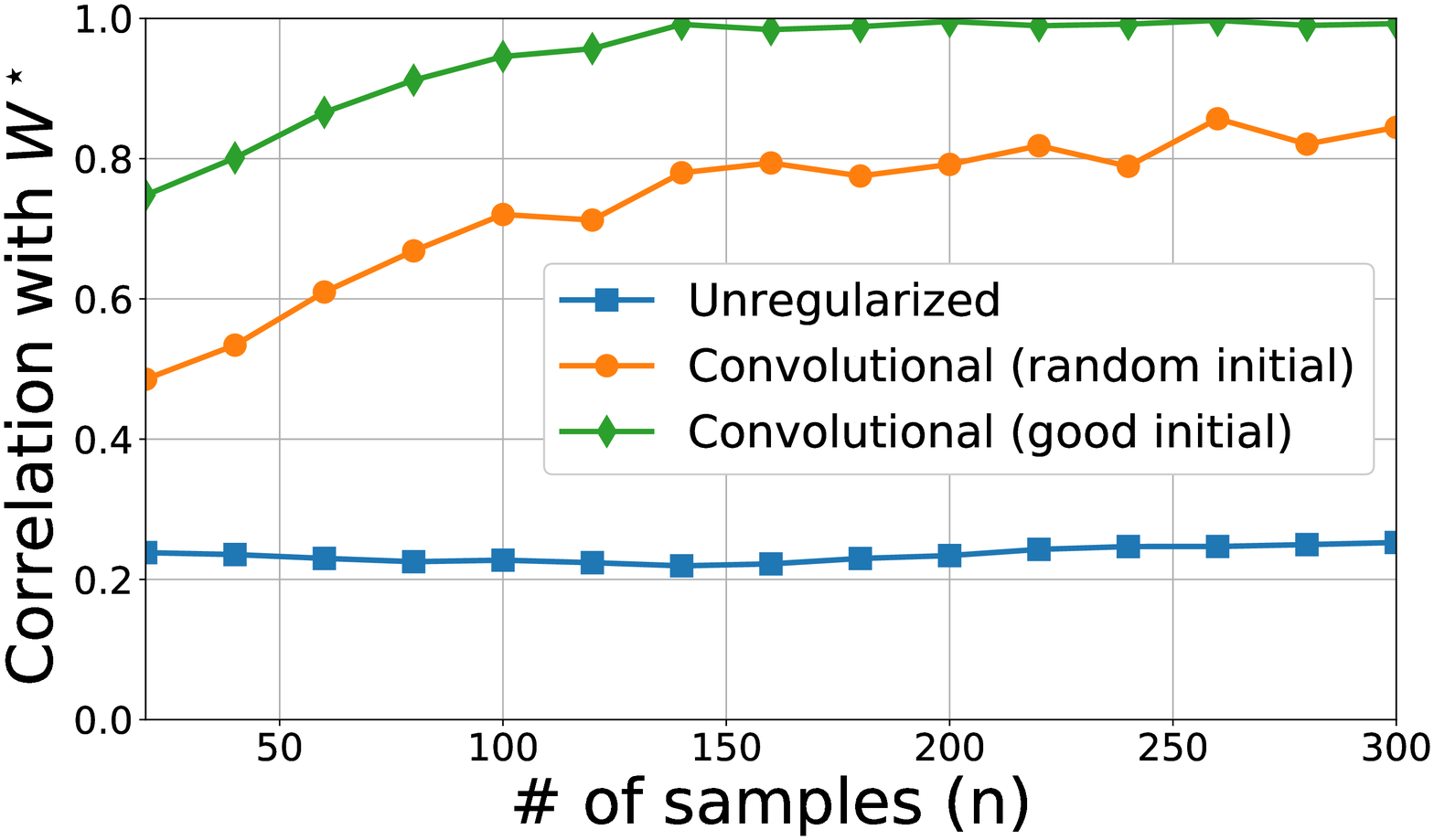}
        \label{fig:rec3}
    \end{subfigure}\end{center}

    \caption{Experiments for convolutional constraint.}   \label{fig:cnn}
\end{figure}

\subsection{Convolutional Constraint}

For the CNN experiment, we picked the following configuration. Problem parameters are input dimension $p=81$, kernel width $b=15$, stride $s=6$, number of kernels $k=4$ and learning rate $\mu=1$. We did not use zero-padding hence $r=(p-b)/s+1=12$. This implies $kr=48$ hidden layers for fully connected representation. The subspace dimension and degrees of freedom is $kb=60$. We generate kernel entries with i.i.d.~$\Nn(0,\frac{p}{hb})$ and the random matrix $\Zb$ with i.i.d. $\Nn(0,\frac{p}{bk})$ entries. The noise variance is chosen higher to ensure $\E[\tf{\Pc_{\Cc}(\Zb)}^2]=\E[\tf{\FC(\KB)}^2]$ i.e. $\Zb$ projected onto convolutional space has the same variance as the kernel matrix. We compare three models.
\begin{itemize}
\item Unconstrained model with $\W_0=\Zb$ initialization: Uses only gradient descent.
\item CNN subspace constraint with $\W_0=\Zb$ initialization: Weights are shared via CNN backpropagation.
\item CNN subspace constraint with with $\W_0=\Ws+\Zb$ initialization.
\end{itemize}
Figures \ref{fig:good} illustrates the outcome of CNN experiments. Unconstrained model barely makes it into the test loss figure due to low signal-to-noise ratio. Focusing on CNN constraints, we observe that good initialization greatly helps and quickly achieves $\approx 0$ test error. However random initialization has respectable test and correlation performance and gracefully improves as the data amount $n$ increases.

\section{Conclusions}

In this work, we studied neural network regularization in order to reduce the storage cost and to improve generalization properties. We introduced covering dimension to quantify the impact of regularization and the richness of the constraint set. We proposed projected gradient descent algorithms to efficiently learn compact neural networks and showed that, if initialized reasonably close, PGD linearly converges to the ground truth while requiring minimal amount of training data. The sample complexity of the algorithm is governed by the covering dimension. We also specialized our results to convolutional neural nets and demonstrated how CNNs can be efficiently learned within our framework. Numerical experiments support the substantial benefit of regularization over training fully-connected neural nets.





 Global convergence of the projected gradient descent appears to be a more challenging problem. In Section \ref{simulation}, we observed that gradient descent with random initialization can get stuck at local minima. For fully-connected networks, this is a well-known issue and the best known global convergence results are based on tensor initialization \cite{zhong2017recovery,safran2017spurious,janzamin2015beating}. Hence, it would be interesting to develop data-efficient initialization algorithms that can take advantage of the network priors (weight-sharing, sparsity, low-rank).

Another direction is exploring whether results and analysis of this work can be extended to deep networks. A reasonable starting point is extending the layer-wise learning approach presented in Section \ref{sec layers} to more realistic activation functions. An obvious technical challenge is the loss of i.i.d.~input distribution as we move to deeper layers. Finally, while this work addressed the generalization problem for shallow networks, generalization for deeper networks and generalization properties of (regularized) gradient descent (e.g.~when it converges to a local minima and how good it is) are intriguing directions when it comes to training compact neural nets.


\small{
\bibliographystyle{plain}
\bibliography{Bibfiles}
}
\pagebreak
\newpage
\appendix
\section{Perturbed Gaussian width}\label{perturbed width sec}

Almost all of our analysis will be in terms of ``perturbed width'' which is used to capture the geometry of a set $T$. We will replace covering dimension with perturbed width for our technical arguments. Our results will apply in the sample size regime $n\geq \omega_n^2(T)$.  Here, we introduce perturbed width and show how covering dimension bounds can be reduced to perturbed width bounds.

\begin{definition}[Perturbed Gaussian width] \label{pert width}Let $C>0$ be an absolute constant. Define $\text{rad}(S)=\sup_{\vb\in S}\tn{\vb}$. Given a set $T\subset\Bc^d$ and an integer $n\geq 1$, we define perturbed width $\omega_n(T)$ as
\[
\omega_n(T)=\min_{\clconv(S)\supseteq T,~\text{rad}(S)\leq C}\omega(S)+\frac{\gamma_1(S)}{\sqrt{n}}
\]
where $\omega(S)$ is the Gaussian width and $\gamma_1(S)$ is Talagrand's $\gamma_1$ functional (see Definition \ref{gamma functional}) with $\ell_2$ distance. $\omega(S)$ is given by
\[
\omega(S)=\E[\sup_{\vb\in S} \vb^T\g]
\]
where $\g$ is a Gaussian vector with i.i.d. entries. 
\end{definition}
Gaussian width is frequently utilized in statistics and optimization community to capture the degrees of freedom of the problem \cite{Cha,oymak2016fast,amelunxen2014living}. For instance, if $S$ is the unit ball in $\R^d$, then $\omega^2(S)\approx d$. If $S$ is also composed of $s$ sparse elements, then $\omega^2(S)\lesssim 2s\log (6d/s)$ \cite{Cha}.

Perturbed width includes the additional term $n^{-1/2}\gamma_1(\cdot)$. $\gamma_1(\cdot)$ is closely related to Gaussian width and both arise from the generic chaining argument. The reason for the naming ``perturbed'' becomes clear when we consider $n\rightarrow \infty$ which yields $\omega_n(T)\rightarrow\omega(T)$. Below, we show that $\omega_n^2(T)\lesssim \covdim{T}$ for typical sets and is proportional to the degrees of freedom.

\subsection{Covering dimension to perturbed width}

\begin{lemma} $\omega^2(T),\gamma_1(T)\leq \order{\covdim{T}}$. Hence $\omega_n^2(T)\leq\order{ \covdim{T}} $ for $n\geq \order{\covdim{T}}$.
\end{lemma}
\begin{proof} The proof is based on Lemma \ref{simple sum}. Let $S$ be a set obeying Definition \ref{covdim} with covering dimension $s\log B\leq 2\covdim{T}$. The Gaussian width and $\gamma_1$  functional of $S$ can be upper bounded as
\[
\omega(S)\leq \order{\sqrt{s\log B}},~\gamma_1(S)\leq \order{{s\log B}}.
\] 
When $n\geq s\log B$, we have that,
\[
\omega(S)\leq \order{\sqrt{\covdim{T}}},~\frac{\gamma_1(S)}{\sqrt{n}}\leq\order{ \sqrt{s\log B}}= \order{\sqrt{\covdim{T}}} 
\]
Combination yields $\omega_n(T)\leq \omega_n(S)=\order{\sqrt{\covdim{T}}} $.
\end{proof}

Luckily, the constraint sets of interest, such as sparse and low-rank weight matrices admit good covering numbers. This ensures that
\[
\gamma_1(T)\approx \omega^2(T)\approx \covdim{T}\approx \text{degrees of freedom}
\]

Perturbed width can be calculated for arbitrary and unstructured sets as well. In particular, it is rather trivial to show that $\gamma_1(\Tc)$ can be bounded in terms of $\omega(\Tc)$. The following is a corollary of Lemmas \ref{arbit reg} and \ref{simple sum}.
\begin{lemma}[Bounding $\gamma_1$ term]\label{gamma bound} Denote $\eps$-covering number of $T\subset\Bc^d$ with respect to $\ell_2$ distance as $N_\eps(T)$. Suppose $N_\eps(T)\leq (B/\eps)^s$ for some numbers $B\geq 2,s\geq 2$ independent of $\eps$. Then,
\[
\gamma_1(T)\leq \order{\max\{\omega(T)\sqrt{s\log s\log B},~s\log B\}}
\]
\end{lemma}

For generic constraint sets without a good covering number bound, Lemma \ref{gamma bound} yields the following looser bound (using $T\subset\Bc^d$)
\[
\omega_n(\Tc)\leq C\cdot\sqrt{\omega(\Tc)d\log d}
\]
which will imply a sample complexity requirement of $n\geq  \order{\omega(\Tc)\sqrt{d\log d}}\geq \omega_n^2(T)$ for our main results.

%

\section{Proof of main theorem} \label{secmainproof}
Next sections are dedicated to the proofs of our main technical results. For these subsequent sections, we introduce further notation that will simplify our life. 

\noindent {\bf{Notation:}} Outer product between two matrices $\X\in\R^{n_1\times n_2},\Y\in\R^{d_1\times d_2}$ is denoted by $\X\bt\Y\in \R^{d_1n_1\times d_2n_2}$. This matrix is defined as
\[
(\X\bt\Y)_{i_1d_1+j_1,i_2d_2+j_2}=\X_{i_1,i_2}\Y_{j_1,j_2}.
\]
Given two vectors $\x,\y$ of identical size, $(\x\bd\y)_i=\x_i\y_i$. Given $\X\in\R^{h\times p}$, $\bar{\x}\in\R^{hp}$ will denote the vector obtained by putting rows of $\X$ on top of each other. Given a matrix $\Ub$, its $(i,j)$th entry, $i$th row and $j$th column is given by $\Ub_{(i,j)}$, $\Ub_{(i,:)}$, $\Ub_{(:,j)}$. Given a random vector $\x$, $\bSi(\x)$ returns its covariance. Let $\Bc^{d},\Sc^{d}$ denote the unit ball and sphere in $\R^d$. Given a set $T$, $\Delta(T)$ will denote its $\ell_2$ diameter.

We also define the restricted singular value (RSV) and restricted eigenvalue (RE) of a matrix as follows.
\begin{definition}[Restricted singular value] \label{def rsv}Given a matrix $\M$ and a set $\Cc$, the restricted singular value (for all $\M$) and the restricted eigenvalue (only for square $\M$) are defined as
\[
\sigma(\M,\Cc)=\inf_{\vb\in \Cc}\frac{\tn{\M\vb}}{\tn{\vb}},~\lambda(\M,\Cc)=\inf_{\vb\in \Cc}\frac{\vb^T\M\vb}{\tn{\vb}^2}.
\]
\end{definition}


\subsection{Proof strategy}
We now go over the proof strategy and introduce the main ideas. Our goal is to construct the weight matrix $\Ws$ via PGD iterations \eqref{pgd algo}. Towards this goal, given an initial point $\Ub$, we consider the single gradient iteration
\begin{align}
\hat\Ub=\Pc_{\Cc}(\Ub-\mu\grad{\Ub})\label{pgd main}
\end{align}
and study the estimation error $\tf{\Ws-\Ub}$ as a function of $\Ws$ and $\Ub$.

To simplify the subsequent notation, we introduce the following shortcut notations. Let $\sigmap$ be the operation that takes a vector $\x\in\R^h$ as input and returns a vector with entries $\ob_i\sigma'(\x_i)$. Given matrices $\W,\Ub\in\R^{h\times p}$ and a vector $\g\in\R^p$, define $\db(\cdot)\in\R^h$ functions as
\[
\dd(\W;\g)=\sigmap(\W\g), ~\dd(\W,\Ub;\g)_i=\frac{\ob_i(\sigma(\W_{(i,:)}\g)-\sigma(\Ub_{(i,:)}\g))}{\W_{(i,:)}\g-\Ub_{(i,:)}\g}
\]
for $1\leq i\leq h$. Next, define $\rho(\cdot)\in\R^{hp}$ function as
\[
\rho(\W;\g)=\dd(\W;\g)\bt \g,~ \rho(\W,\Ub;\g)=\dd(\W,\Ub;\g)\bt \g.
\]
We now study the gradient descent algorithm. Let us focus on the loss associated with $j$th sample
\[
\Lc_j(\Ub)=\frac{1}{2}(\ob^T\sigma(\Ws\x_j)-\ob^T\sigma(\Ub\x_j))^2.
\]
Differentiation yields
\begin{align}
\nabla\Lc_j(\Ub)&=(\ob^T\sigma(\Ub\x_j)-\ob^T\sigma(\Ws\x_j))\sigmap(\Ub\x_j)\bt \x_j\\
&=(\ob^T\sigma(\Ub\x_j)-\ob^T\sigma(\Ws\x_j))\rho(\Ub;\x_j)\\
&=\rho(\Ub;\x_j)\rho(\Ub,\Ws;\x_j)^T(\ubb-\ws)\\
\end{align}
Consequently, using the fact that empirical loss $\Lc=n^{-1}\sum_{i=1}^n\Lc_j$, the overall gradient takes the form
\begin{align}
\grad{\Ub}&=n^{-1}\sum_j\nabla\Lc_j(\Ub)\\
&=n^{-1}\sum_{j=1}^n(\ob^T\sigma(\Ub\x_j)-\ob^T\sigma(\Ws\x_j))\rho(\Ub;\x_j)\\
&=n^{-1}\sum_{j=1}^n\rho(\Ub;\x_j)\rho(\Ub,\Ws;\x_j)^T(\ubb-\ws)
\end{align}
Observe that, the final line is a product of a matrix and $\ubb-\ws$. We will decompose this matrix into three pieces and connect it to the Hessian at $\Ws$ as follows.
\begin{align}
H_1(\Ub,\Ws)&=n^{-1}\sum_{j=1}^n\rho(\Ws;\x_j)\rho(\Ws;\x_j)^T\label{hessian h1}\\
H_2(\Ub,\Ws)&=n^{-1}\sum_{j=1}^n\rho(\Ws;\x_j)(\rho(\Ub,\Ws;\x_j)-\rho(\Ws;\x_j))^T\\
H_3(\Ub,\Ws)&=n^{-1}\sum_{j=1}^n(\rho(\Ub;\x_j)-\rho(\Ws;\x_j))^T\rho(\Ub,\Ws;\x_j)\\
H(\Ub,\Ws)&=n^{-1}\sum_{j=1}^n\rho(\Ub;\x_j)\rho(\Ub,\Ws;\x_j)^T
\end{align}
Observe that $\grad{\Ub}=H(\ubb-\ws)$ where $H=\sum_{i=1}^3 H_i(\Ub,\Ws)$ and $H_1$ is the Hessian at the ground truth $\Ws$.
%
For the following discussion, we sometimes drop the $(\Ub,\Ws)$ subscript from the $H_i$'s when it is clear from the context.
$H_2,H_3$ will be viewed as perturbations over the ground truth Hessian $H_1$. Consequently, our strategy will be to argue that they are small. This is done by Theorem \ref{spec bounds}. The other crucial component is arguing that Hessian $H(\Ub,\Ws)$ is positive definite over the constraint set $\Tc$. This will be done by obtaining a bound for the restricted eigenvalue of the matrix $H(\Ub,\Ws)$ (see Theorem \ref{rsv hessian main}). The proof will be completed by obtaining such estimates and applying Lemma \ref{combine all} to combine them to get a high probability convergence guarantee.

The following lemma provides the deterministic condition for convergence based on the definitions above.
\begin{lemma}\label{combine all} Recall \eqref{pgd main}. Suppose $\Cc$ is a closed and convex constraint set and the following bounds hold for $\Ub\in\Cc$.
\begin{itemize} 
\item {\bf{Small perturbation:}} $\|H_2+H_3\|\leq \eps$.
\item {\bf{Bounded spectrum:}} $\alpha\Iden_{hp}\succeq H_1$.
\item {\bf{Restricted eigenvalue:}} $\lambda(H_1,\Tc)\geq \beta$.
\end{itemize}
Assume $\beta\geq 10\eps$ and use learning rate $\mu=\frac{1}{\alpha}$ , PGD estimate $\hat\Ub$ satisfies the bound
\[
\tf{\hat\Ub-\Ws}^2\leq (1-\frac{\beta}{2\alpha})\tf{\Ub-\Ws}^2
\]
\end{lemma}
\begin{proof} Restating \eqref{pgd main}, we have that
\[
\hat\Ub=\Pc_\Cc(\Ub-\mu\grad{\Ub})
\]
Define the error matrix $\Zb=\Ub-\Ws$ and $\hat\Zb=\hat\Ub-\Ws$. Using convexity of $\Cc$ (hence projection on $\Cc$ contracts distance), this implies that
\[
\tf{\hat\Zb}^2\leq \tf{\Pc_\Cc(\Ub-\mu\grad{\Ub})-\Ws}\leq \tf{\Zb-\mu\grad{\Ub}}
\]
Now, recalling $\grad{\Ub}=H(\Ub,\Ws)\zb$ we will expand the right hand side, in particular
\[
\tf{\Zb-\mu\grad{\Ub}}^2=\tf{\Zb}^2-2\mu\li\zb,H(\Ub,\Ws)\zb\ri+\mu^2\zb^TH(\Ub,\Ws)^TH(\Ub,\Ws)\zb.
\]
Decomposing the middle term,
\[
\li\zb,H(\Ub,\Ws)\zb\ri\geq \zb^T H_1\zb-\eps\tn{\zb}^2.
\]
Decomposing the third term (denote $H_r=H_2+H_3$),
\begin{align}
\zb^TH(\Ub,\Ws)^TH(\Ub,\Ws)\zb&=\zb^TH_1H_1\zb+2\zb^T H_1H_r\zb+\tn{H_r\zb}^2\\
&\leq \alpha \zb^TH_1\zb+2\alpha\eps\tn{\zb}^2+\eps^2\tn{\zb}^2.
\end{align}
where we used the fact that $H_1\preceq \alpha\Iden$ is positive semidefinite. Combining the latest two bounds, we obtain
\begin{align}
\tf{\hat\Zb}^2&\leq \tn{\zb}^2-2\mu (\zb^T H_1\zb-\eps\tn{\zb}^2)+\mu^2(\alpha\zb^TH_1\zb+2\alpha\eps\tn{\zb}^2+\eps^2\tn{\zb}^2)\\
&\leq \tn{\zb}^2-(2\mu-\mu^2\alpha) \zb^T H_1\zb+(2\mu\eps+\mu^2\eps^2+2\mu^2\alpha\eps)\tn{\zb^2}\\
&\leq \tn{\zb}^2(1-(2\mu-\mu^2\alpha) \beta+(2\mu\eps+\mu^2\eps^2+2\mu^2\alpha\eps))
\end{align}
Setting $\mu=1/\alpha$, we obtain
\[
\tf{\hat\Zb}^2\leq \tf{\Zb}^2(1-\frac{\beta-4\eps}{\alpha}+\frac{\eps^2}{\alpha^2})
\]
Using the fact that $\alpha\geq \beta\geq 10\eps$, we obtain that
\[
\tf{\hat\Zb}^2\leq \tf{\Zb}^2(1-\frac{3\beta}{5\alpha}+\frac{\eps}{10\alpha})\leq  \tf{\Zb}^2(1-\frac{3\beta}{5\alpha}+\frac{\beta}{100\alpha})\leq \tf{\Zb}^2\frac{\beta}{2\alpha}
\]
\end{proof}


\subsection{Proof of convergence}
This theorem states our main result on convergence of projected gradient algorithm with convex regularizers. We first revisit the critical quantities that will be used for the statement.
\[
\Theta:=\Theta_{\sigma,\Ws}=\frac{\lip^2\kappa^{2}(\ob)\kappa^{h+2}(\Ws)}{\zeta(\s_{\min})}~~~\text{and}~~~\BB= h(\log p+\frac{\lzero^2}{\lip^2\s^2_{\max}})
\]
%

\begin{theorem} [Proof of Theorem \ref{good thm}]\label{general thm}Suppose $\Cc$ is a convex constraint set that includes $\Ws$, $h\leq p$, and let $\{\x_i\}_{i=1}^n$ be i.i.d.~$\Nn(0,\Iden_p)$ data points.. Set $\upsilonb=C\Theta\log^2( C\Theta)$ and suppose
\[
n\geq\order{{(\omega_n(\Tc)+t)^2}{\upsilonb^4}},
\]
Set $q=\max\{1,8n^{-1}p\log p\}$. Define learning rate $\mu$ and rate of convergence $\rho$ as
\begin{align}
\mu=\frac{1}{6q\ob_{\max}^2\lip^2\BB},~\rho=1-\frac{1}{12q\upsilonb^4\BB}
\end{align}
Consider the projected gradient iteration
\[
\hat\W=\Pc_\Cc(\W-\mu\grad{\W})
\]
$\bullet$ {\bf{Convergence with large radius:}} Suppose initial point $\W$ satisfies
\[
\tf{\W-\Ws}\leq \order{\frac{\tf{\Ws}}{q\sqrt{h\BB\log p}\upsilonb^{4}}}\leq \order{\frac{\s_{\max}}{q\sqrt{\BB\log p}\upsilonb^{4}}}.
\]
Then, $\hat\W$ obeys $\tf{\hat\W-\Ws}^2\leq \rho \tf{\W-\Ws}^2$.

\noindent$\bullet$ {\bf{Uniform convergence:}} Furthermore, suppose $\W$ satisfies the tighter constraint
\[
\tf{\W-\Ws}\leq \frac{\tf{\Ws}}{q\sqrt{hp\Omega}\upsilonb^{4}} .
\]
then, starting from $\W=\W_0$, for all $i\geq 0$, $\W_{i+1}=\Pc_\Cc(\W_i-\mu\grad{\W_i})$ obeys
\[
\tf{\W_i-\Ws}^2\leq \rho^i\tf{\W-\Ws}^2
\]
Both results hold with probability $1-\exp(-n/\upsilonb^2)-2\exp(-\order{\min\{t\sqrt{n},t^2\}})+8(n\exp(-p/2)+np^{-10}+\exp(-qn/4p))$.
\end{theorem}
\begin{proof} Proof follows by substituting proper values in Lemma \ref{combine all}. First, let us address the restricted singular value condition. Let $\b\lip=\lip\ob_{\max}\s_{\max}$. Using $\upsilonb=\upsilon^{-1}$, Theorem \ref{rsv hessian main} yields
\[
\lambda(H_1,\Tc)\geq \bL^2\Theta^{-1}\upsilonb^{-3}\geq \bL^2\upsilonb^{-4}
\]
with probability $1-2\exp(-\order{\min\{t\sqrt{n},t^2\}})-\exp(-n/\upsilonb^2)$.

Next, we estimate the spectral norm of $H_1,H_2,H_3$ for $\W$. Proposition \ref{spec bounds} yields (importing variables $B,P,\Ub\rightarrow\W$)
\[
\|H_1\|\leq 6qB^2,\|H_2\|+\|H_3\|\leq \order{qBP\tf{\W-\Ws}}
\] 
Oberve that $\BB$ satisfies 
\[
\BB^{-1/2}= \order{\frac{\lip\s_{\max}}{\lip\s_{\max}\sqrt{h\log p}+{\lzero}\sqrt{h}}}=\order{\frac{\bL}{B}}.
\]
This implies that as soon as 
\[
\tf{\W-\Ws}\leq \order{\s_{\max}\frac{\upsilonb^{-4}}{ q\sqrt{\BB\log p}}}= \order{\frac{\bL^2\upsilonb^{-4}}{qBP}}
\]
applying Lemma \ref{combine all}, we achieve the convergence rate
\begin{align}
\rho=1-\frac{\bL^2\upsilonb^{-4}}{2\|H_1\|}\leq 1-\frac{1}{12q\upsilonb^4\BB}\label{iden rates}
\end{align}
by choosing the learning rate $\mu=\frac{1}{6qB^2}=\frac{1}{6q\s_{\max}^2\ob_{\max}^2\lip^2\BB}$.

For uniform convergence result, $\W_i$ is possibly dependent on $\{\x_j\}$s. Consequently, we would like to bound Hessian for all points around $\Ws$ uniformly. To achieve this, we apply Proposition \ref{prop nb} which yields the looser upper bound
\[
\|H_2\|+\|H_3\|\leq \order{\tf{\W-\Ws}\max\{1,\frac{p}{n}\}\ob_{\max}^2\lip(\lip\sqrt{\log p}\tf{\Ws}+\lzero\sqrt{h})\sqrt{p}}=\order{\tf{\W-\Ws}\bar{P}}
\]
for $p^{-1/2}\tf{\Ws}$ neighborhood of $\Ws$. We then carry out the exact same argument where the initialization requirement is
\begin{align}
\tf{\W-\Ws}&\leq \frac{\tf{\Ws}}{q\sqrt{hp\Omega}\upsilonb^{4}} \leq \frac{\s_{\max}}{q\sqrt{p\Omega}\upsilonb^{4}} \\
&\leq\order{\frac{\lip^2\ob_{\max}^2\s_{\max}^2\upsilonb^{-4}}{\max\{1,\frac{p}{n}\}\ob_{\max}^2\lip\sqrt{p}(\lzero\sqrt{h}+\lip\sqrt{h\log p}\s_{\max})}}\\
&\leq\order{\frac{\lip^2\ob_{\max}^2\s_{\max}^2\upsilonb^{-4}}{\max\{1,\frac{p}{n}\}\ob_{\max}^2\lip\sqrt{p}(\lzero\sqrt{h}+\lip\sqrt{\log p}\tf{\Ws})}}\\
&=\order{\frac{\bL^2\upsilonb^{-4}}{\bar{P}}}.
\end{align}
Recalling the lower bound on $\lambda(H_1,\Tc)$, for all $\W$ satisfying $\tf{\W-\Ws}\leq \frac{\tf{\Ws}}{q\sqrt{hp\Omega}\upsilonb^{4}}$
this implies
\[
\lambda(H_1,\Tc)\geq \bL^2\upsilonb^{-4}\geq \order{\tf{\W-\Ws}\bar{P}}\geq \|H_2\|+\|H_3\|.
\]
Consequently, we obtain identical convergence rates to \eqref{iden rates} for all $\W$ in this tighter neighborhood where $\W$ is allowed to depend on data points.

Also, observe that at each iteration, the distance $\tf{\W_i-\Ws}$ will get smaller at each iteration so Hessian perturbation bound will always be valid because we will never get out of uniform convergence radius.
\end{proof}

\subsection{Proof of Theorem \ref{good cor main}}
\begin{proof} The proof follows from Theorem \ref{general thm}. Let $K=\order{q\upsilonb^4\BB\log p}$ and $\rho$ be same as in Theorem \ref{good cor main}. Suppose $\W_0$ is initialized as described. Applying the ``large radius convergence'' result of Theorem \ref{general thm}, using $i$th data batch at $i$th gradient step, for all $1\leq i\leq K-1$, with probability $1-(K-1)P$, we have that
\[
\tf{\W_i-\Ws}^2\leq \rho^i \tf{\W_0-\Ws}^2.
\]
Observe that $K=\order{(1-\rho)^{-1}\log p}$. Hence $\rho^{K-1}=(1-(1-\rho))^{\order{(1-\rho)^{-1}\log p}}\leq 0.25^{\order{\log p}}\leq p^{-\order{1}}$. Using this, we obtain
\[
\tf{\W_{K-1}-\Ws}^2\leq \rho^{K-1}\tf{\W_0-\Ws}^2\leq \frac{\order{1}}{p}\tf{\W_0-\Ws}^2.
\]
This implies $\W_{K-1}$ is sufficiently close to $\Ws$ to apply the uniform convergence result of Theorem \ref{general thm}. Now, starting from $\W_{K-1}$, we use PGD with batch $K$ for all steps $i\geq K$ to achieve $\tf{\W_{i}-\Ws}^2\leq \rho^{i}\tf{\W_0-\Ws}^2$ for all $i\geq 0$.
\end{proof}

\section{Upper bounding spectral norms}

First, we state a basic lemma for activations with bounded second derivative.
\begin{lemma}[Activation perturbation] \label{actper} Under Assumption \ref{actassume}, $|\frac{\sigma(\vb^T\g)-\sigma(\w^T\g)}{(\vb-\w)^T\g}-\sigma'(\w^T\g)|\leq \lip|(\vb-\w)^T\g|$ for all vectors $\vb,\w,\g$.
\end{lemma}
\begin{proof} Since $\sigma'$ is $\lip$ Lipschitz, we have that $\frac{\sigma(a)-\sigma(b)}{a-b}-\sigma'(b)=\frac{\int_b^a\sigma'(x)-\sigma'(b)dx}{a-b}\leq \frac{\int_b^aL|x-b|dx}{a-b}\leq \lip|a-b|$.
\end{proof}

The next result upper bounds the spectral norm of Hessian decomposition.
\begin{proposition} \label{spec bounds}Recall the definitions of $H_1,H_2,H_3$ \eqref{hessian h1} and suppose $\{\x_i\}_{i=1}^n\sim \Nn(0,\Iden_p)^n$. Fix radius $\tf{\Ws}\geq R>0$. Define the quantities
\begin{align}
&B:=\ob_{\max}(5\lip\s_{\max}\sqrt{h\log p}+\lzero\sqrt{h})\geq \ob_{\max}(5\lip\tf{\Ws}\sqrt{\log p}+\lzero\sqrt{h})\\
&P:=\ob_{\max}5\lip\sqrt{\log p}
\end{align}
Set $q=\max\{1,8n^{-1}p\log p\}$. With probability $1-4(n\exp(-p/2)+2np^{-10}+\exp(-qn/4p))$, we have that
\begin{itemize}
\item $H_1\preceq  6qB^2$.
\item For a fixed $\Ub$ (independent of $\{\x_i\}_{i=1}^n$s) satisfying $\tf{\Ub-\Ws}\leq R$, we have $\|H_2\|\leq 6RqBP,\|H_3\|\leq 12RqBP$.
\end{itemize}
\end{proposition}
\begin{proof} The proof of both statements are based on Theorem \ref{0length}. For $H_1$, pick $f_i(\x)=\ob_i\sigma'(\w_i^T\x)$ to establish the result. For $H_2,H_3$, we write
\begin{align}
nH_2=\sum_{i=1}^n \rho(\Ws;\x_i)(\rho(\Ub,\Ws;\x_i)-\rho(\Ws;\x_i))^T\\
nH_3=\sum_{i=1}^n (\rho(\Ub;\x_i)-\rho(\Ws;\x_i))\rho(\Ub,\Ws;\x_i)^T
\end{align}

Define 
\begin{align}
&\M_1=\sum_{i=1}^n(\rho(\Ub,\Ws;\x_i)-\rho(\Ws;\x_i))(\rho(\Ub,\Ws;\x_i)-\rho(\Ws;\x_i))^T\label{m1m2m3}\\
&\M_2=\sum_{i=1}^n\rho(\Ub,\Ws;\x_i)\rho(\Ub,\Ws;\x_i)^T\\
&\M_3=(\rho(\Ub;\x_i)-\rho(\Ws;\x_i))(\rho(\Ub;\x_i)-\rho(\Ws;\x_i))^T
\end{align}

From Cauchy-Schwarz, $\|H_2\|\leq \sqrt{\|H_1\|\|\M_1\|}$, $\|H_3\|\leq \sqrt{\|\M_2\|\|\M_3\|}$. To bound these, we apply Theorem \ref{0length} as follows.
\begin{itemize}
\item For $\M_1$, pick $f_i(\x)=\ob_i(\frac{\sigma(\x^T\ws_i)-\sigma(\x^T\ubb_i)}{\x^T\ws_i-\x^T\ubb_i}-\sigma'(\x^T\ws_i))$ and use Lemma \ref{actper} to obtain \[\|\M_1\|\leq6P^2\tf{\Ws-\Ub}^2\leq 6P^2R^2.\]
\item For $\M_2$, pick $f_i(\x)=\ob_i\frac{\sigma(\x^T\ws_i)-\sigma(\x^T\ubb_i)}{\x^T\ws_i-\x^T\ubb_i}$. Using Lemma \ref{actper}, this yields $|f_i(\x)|\leq \lzero+\lip|\x_i^T\ws_i|+\lip|\x_i^T(\ubb_i-\ws_i)|$. Applying a straightforward variation of Theorem \ref{0length}, when $R\leq \tf{\Ws}$
\[
\M_2\preceq 6B_2^2~\text{where}~B_2=\ob_{\max}(5\lip(\tf{\Ws}+R)\sqrt{\log p}+\lzero\sqrt{h})\leq 2B.
\]
\item For $\M_3$, pick $f_i(\x)=\ob_i(\sigma'(\x^T\ubb_i)-\sigma'(\x^T\ws_i))$ which yields $|f_i(\x)|\leq \lip|\x^T(\ubb_i-\ws_i)|$. This yield $\|\M_3\|\leq 6P^2R^2$.
\end{itemize}

Combining the $H_1,\M_1,\M_2,\M_3$ bounds, these yield $\|H_2\|\leq 6BP$ and $\|H_3\|\leq 12BP$. The overall probability is $1-4(n\exp(-p/2)-2np^{-10}-\exp(-qn/4p))$ via union bound of success over $4$ matrices.
\end{proof}
\begin{proposition}[Bounding $H_2,H_3$ over a neighborhood] \label{prop nb}Recall the definitions of $H_2,H_3$ \eqref{hessian h1} and suppose $\{\x_i\}_{i=1}^n\sim \Nn(0,\Iden_p)^n$. Suppose $\tf{\Ub-\Ws}\leq p^{-1/2}\tf{\Ws}$. With $1-n\exp(-p/2)-np^{-10}$ probability, we have that 
\[
\|H_2\|+\|H_3\|\leq \bar{P}\tf{\Ub-\Ws}
\]
where $\bar{P}=\max\{1,\frac{p}{n}\}\ob_{\max}^2\lip(\lip\sqrt{\log p}\tf{\Ws}+\lzero\sqrt{h})\sqrt{p}$.
\end{proposition}
\begin{proof} Pick unit vectors $\abb,\bbb$ and consider
\[
\sup_{\tn{\abb}=1,\tn{\bbb}=1}\bbb^TH_2\abb,~\sup_{\tn{\abb}=1,\tn{\bbb}=1}\bbb^TH_3\abb
\]
Let $\A,\B$ be the matricized versions of $\abb,\bbb$. Also set $\Zb=\Ub-\Ws$.
\begin{align}
nH_2=\sum_{i=1}^n \rho(\Ws;\x_i)(\rho(\Ub,\Ws;\x_i)-\rho(\Ws;\x_i))^T\\
nH_3=\sum_{i=1}^n (\rho(\Ub;\x_i)-\rho(\Ws;\x_i))\rho(\Ub,\Ws;\x_i)^T
\end{align}
Form $\X\in\R^{p\times n}$ by concatenating $\x_i$'s. Let $\nb=\max\{n,p\}$. The critical observation is that with $1-\exp(-\nb/2)$ probability, $\|\X\|\leq 3\sqrt{\nb}$, hence for all matrices $\A$ (and similarly $\B$), we have
\[
\tf{\A\X}^2\leq 9\nb\tf{\A}^2.
\]
Now, using a very coarse estimate, we upper bound the individual components of empirical average matrices $H_2,H_3$. The argument follows the strategy outlined in the proof of Lemma \ref{spec bounds}.
\begin{align}
\abb^T (\rho(\Ub;\x_i)-\rho(\Ws;\x_i))&\rho(\Ub,\Ws;\x_i)^T\bbb = (\db(\Ub;\x_i)-\db(\Ws;\x_i))\A\x_i\db(\Ub,\Ws;\x_i)\B\x_i\\
&\leq \ob_{\max}^2\lip\tn{(\Ub-\Ws)\x_i}\tn{\A\x_i}(\lip\tn{\Ws\x_i}+\lip\tn{(\Ub-\Ws)\x_i}+\lzero\sqrt{h})\tn{\B\x_i}\nn\\
&\leq 2S_i
\end{align}
\begin{align}
\abb^T \rho(\Ws;\x_i)(\rho(\Ub,\Ws;\x_i)-\rho(\Ws;\x_i))^T\bbb& = \db(\Ws;\x_i)\A\x_i(\db(\Ub,\Ws;\x_i)-\db(\Ws;\x_i))\B\x_i\\
&\leq \ob_{\max}^2(\lip\tn{\Ws\x_i}+\lzero\sqrt{h})\tn{\A\x_i}(\lip\tn{(\Ws-\Ub)\x_i}\tn{\B\x_i})\nn\\
&\leq S_i
\end{align}

Observe that both statements have similar upper bounds. We are now interested in finding (a rather loose) upper bound on these quantities namely $n^{-1}\sum_i S_i$.

Denote $w_i=\frac{\tn{\Ws\x_i}}{\tf{\Ws}},~a_i=\tn{\A\x_i},~u_i=\frac{\tn{(\Ws-\Ub)\x_i}}{\tf{\Ws-\Ub}},~b_i=\tn{\B\x_i}$. We have $S_i\sim a_ib_iu_i(\lip(w_i\tf{\Ws}+u_i\tf{\Ws-\Ub})+\lzero\sqrt{h})$. With probability $1-n\exp(-p/2)$, all $\x_i$s obey $\tn{\x_i}\leq 2\sqrt{p}$ so that $a_i,b_i,u_i\leq 2\sqrt{p}$. Since $\Ws$ is independent of $\x_i$, applying subexponential Chernoff, we obtain
\[
|w_i|\leq \order{\sqrt{\log p}}
\]
with probability $1-np^{-10}$. We first bound the component $S_{u}\sim\sum_i a_ib_iu_i^2$ which yields the following maximization
\[
\sum_{i=1}^n a_ib_iu_i^2~\text{subject to}~\sum_{i=1}^n a_i^2\leq 9\nb,~a_i\leq 2\sqrt{p}~(\text{same for}~b_i,u_i).
\]
Observe that $a_ib_iu_i^2\leq 0.25(a_i^4+b_i^4+2u_i^4)$. Hence, we consider
\[
\sum_{i=1}^n a_i^4~\text{subject to}~\sum_{i=1}^n a_i^2\leq 9\nb,~a_i\leq 2\sqrt{p}
\]
Observe that
\[
\sum_{i=1}^n a_i^4\leq \sum_{i=1}^n a_i^2(\max_{i=1}^n a_i)^2\leq 9\nb(2\sqrt{p})^2=36\nb p.
\]
This yields $\sum_{i=1}^n a_ib_iu_i^2\leq 36\nb p$ subject to constraints. Hence, the first component obeys
\[
S_u\leq 36\ob_{\max}^2\lip^2\tn{\Ub-\Ws}^2\nb p
\]
Similarly, we can bound the second component $S_w\sim\sum_i a_ib_iu_i(\lip w_i\tf{\Ws}+\lzero\sqrt{h})$ as follows
\[
\sum_{i=1}^n a_ib_iu_i\leq 3^{-1}\sum_{i=1}^n a_i^3+b_i^3+u_i^3\leq 18\nb \sqrt{p},
\]
which gives
\[
S_w\leq \order{\ob_{\max}^2\lip\tf{\Ws-\Ub}(\lip\sqrt{\log p}\tf{\Ws}+\lzero\sqrt{h})\nb \sqrt{p}}
\]
Observe that if $\sqrt{p}\tn{\Ub-\Ws}\leq \tf{\Ws}$, then $S_u\leq S_w$ so that $\sum S_i\leq 2S_w$. Consequently, for all $\Ub$ obeying $\tf{\Ub-\Ws}\leq \tf{\Ws}/\sqrt{p}$, we obtain
\[
\|H_2\|+\|H_3\|\leq n^{-1}\sum_{i=1}^n2S_i\leq \order{\max\{1,\frac{p}{n}\}\ob_{\max}^2\lip\tf{\Ws-\Ub}(\lip\sqrt{\log p}\tf{\Ws}+\lzero\sqrt{h})\sqrt{p}}
\]
\end{proof}

The lemma below provides a spectral norm bound on matrices that are particular functions of Gaussian vectors.
\begin{lemma}[Bounding Spectral Norm]\label{0length} Assume $p\geq 2$. Given $\W$ with rows $\{\w_i\}_{i=1}^h$ and $\x\in\R^p$, suppose functions $f_i$ obey $|f_i(\x)|\leq \lip|\w_i^T\x|+\lzero$. Define
\[
B:=5\lip\tf{\W}\sqrt{\log p}+\lzero\sqrt{h}.
\]
Set $q=\max\{1,8n^{-1}p\log p\}$. Given i.i.d. $\Nn(0,\Iden)$ data $\{\x_i\}_{i=1}^n$, defining $\y_i=[f_1(\x_i)\x_i~\dots~f_h(\x_i)\x_i]\in\R^{hp}$, with probability $1-n(\exp(-p/2)+2p^{-10})-\exp(-qn/4p)$, we have that
\[
n^{-1}\sum_{i=1}^n\y_i\y_i^T\preceq q6B^2\Iden.
\]
\end{lemma}
\begin{proof} For $\x\in\Nn(0,\Iden)$, define the vector $\xh$ obtained by conditioning $\x$ on the events $E_1=\tn{\x}\leq 2\sqrt{p}$ and $E_2=\sup_{i=1}^h\frac{|\w_i^T\x|}{\tn{\w_i}} \leq 5\sqrt{\log p}$. $\Pro(E_1\cap E_2)\geq 1-\exp(-p/2)-2p^{-10}$. With probability $1-n(\exp(-p/2)-2p^{-10})$, all $\x_i$s satisfy $E_1,E_2$ and has the conditional distribution $\hat \x_i$. Rest of the argument will use these conditional vectors. First, observe that
\begin{align}
\sqrt{\sum_{i=1}^h f_i^2(\hat\x)}&\leq \sqrt{\sum_{i=1}^h (5\lip\sqrt{\log p}\tn{\w_i}+\lzero)^2}\\
&\leq \sqrt{\sum_{i=1}^h (5\lip\sqrt{\log p}\tn{\w_i}})^2+\sqrt{h}\lzero\label{expand me}\\
&\leq 5\lip\tf{\W}\sqrt{\log p}+\lzero\sqrt{h}:=B
\end{align}
where \eqref{expand me} follows from squaring both sides and applying Cauchy-Schwarz. This upper bound on $f_i$'s provides bounds on $\tn{\y}$ and $\E[\y\y^T]$ as follows.
\begin{align}
\tn{\y}&=\tn{\hat\x}\sqrt{\sum_{i=1}^h f_i^2(\x)}\leq B\tn{\hat\x}\leq 2\sqrt{p}B.\label{bounded y}
\end{align}
Given unit length $\abb=[\ab_1~\dots~\ab_h]$, 
\begin{align}
\E[(\y^T\abb)^2]=\E[(\sum_{i=1}^h |f_i(\hat\x)\ab_i^T\hat\x|)^2]\leq \E[\sum_{i=1}^h f_i^2(\hat\x)\sum_{i=1}^h(\ab_i^T\hat\x)^2]\leq B^2\E[\sum_{i=1}^h(\ab_i^T\hat\x)^2]
\end{align}
which follows from Cauchy-Schwarz. To bound the expectation, observe that events $E_1,E_2$ hold with at least probability $1/2$, hence
\[
\E[\sum_{i=1}^h(\ab_i^T\hat\x)^2]\leq \frac{\E[\sum_{i=1}^h(\ab_i^T\x)^2]}{\Pro(E_1\cap E_2)}\leq 2.
\]
This implies $\E[\y\y^T]\preceq 2B^2$. Now, we are at a position to apply matrix Chernoff bound as $\tn{\y_i}$ is bounded via \eqref{bounded y}. Recall $q=\max\{1,8n^{-1}p\log p\}$. With probability $1-p^2\exp(-2qnB^2/(2\sqrt{p}B)^2)=1-p^2\exp(-qn/2p)$, we have that 
\[
n^{-1}\sum_{i=1}^n\y_i\y_i^T\preceq e2qB^2\leq 6qB^2.
\] 
To conclude, observe that $p^2\exp(-qn/2p)\leq \exp(-qn/4p)$ as $qn\geq 8p\log p$.
\end{proof}

Next, we define subexponential and subgaussian norms of random variables.
\begin{definition}[Orlicz norms] For a scalar random variable Orlicz-$a$ norm is defined as
\[
\|X\|_{\psi_{a}}=\sup_{p\geq 1}p^{-1/a}(\E[|X|^p])^{1/p}
\]
Orlicz-$a$ norm of a vector $\x\in\R^d$ is defined as
\[
\|\x\|_{\psi_{a}}=\sup_{\vb\in \Sc^{d}} \|\vb^T\x\|_{\psi_{a}}
\]
We define subexponential norm as the function $\te{\cdot}$ and subgaussian norm as the function $\tsub{\cdot}$.
\end{definition}

The following result directly follows from subexponential Chernoff bound.
\begin{corollary} [Subgaussian vector length]\label{cor sub len} Let $\ab\in\R^d$ be a vector with i.i.d. zero-mean subgaussian entries with unit variance and maximum $\tsub{\cdot}$ norm $K$. Suppose $t\geq 2K$, then
\[
\Pro(\tn{\ab}^2\geq 2dt)\leq 2\exp(-cdt/K)
\]
\end{corollary}

\begin{lemma}[Spectral norm bound for random activation]\label{gen spec} Let $\x_i\in\R^{d_1},\y_i\in\R^{d_2}$ be i.i.d.~isotropic subgaussian vectors with subgaussian norms at most $\sqrt{K_x},\sqrt{K_y}$ respectively. Let $\A,\B$ be arbitrary matrices and set $\ab_i=\A\x_i,\bb_i=\B\y_i$. 
Suppose $d_1/K_x,d_2/K_y\geq C$ for some constant $C>0$ and set $q=\max\{1,8n^{-1}d_1d_2\log d_1d_2\}$. With probability $1-\exp(-qn/d_1d_2)-n\exp(-cd_1/K_x)-n\exp(-cd_2/K_y)$, we have that
\[
\|n^{-1}\sum_{i=1}^n (\x_i\bt\y_i)(\x_i\bt\y_i)^T\|\leq 6q\|\A\|^2\|\B\|^2.
\]
\end{lemma}
\begin{proof} First, we upper bound $\tn{\x_i\bt\y_i}$ probabilistically. Applying Corollary \ref{cor sub len}, we have that, for each $\x_i$ (similarly $\y_i)$, with probability $1-\exp(-cd_1/K_x)$
\[
\tn{\x_i}\leq \sqrt{2d_1}
\]
Setting $E=\text{event}\{\tn{\x_i}\leq \sqrt{2d_1}\}$, observe that for any vector $\vb$, moments of the conditioned random variable obey
\[
\E[|\vb^T\x_i|^p\bgl E]\leq\Pro(E)^{-1}\E[|\vb^T\x_i|^p]\leq \sqrt{2}\E[|\vb^T\x_i|^p]
\]
which implies conditional random variable obeys $\tsub{\x_i\bgl E}\leq \sqrt{2K_x}$.

For the rest of the proof, we condition $\x_i,\y_i$ on the event that individually each of them have length at most $\sqrt{2d_1},\sqrt{2d_2}$. This occurs with probability $1-n(\exp(-cd_1/K_x)-\exp(-cd_2/K_y))\leq 1-n\exp(-cd_1/K_x)-n\exp(-cd_2/K_y)$. This way, the truncated $\x_i,\y_i$ have subgaussian norm at most $\sqrt{2K_{x/y}}$. Furthermore, $\bSi(\x_i), \bSi(\y_i)\preceq \sqrt{2}\Iden$. This implies
\[
\bSi(\ab_i)\preceq \sqrt{2}\|\A\|^2\Iden,~\bSi(\bb_i)\preceq \sqrt{2}\|\B\|^2\Iden\implies \bSi(\ab_i\bt\bb_i)\preceq 2\|\A\|^2\|\B\|^2\Iden.
\] Finally, length of $\ab_i\bt\bb_i$ obeys
\[
\tn{\ab_i\bt\bb_i}\leq \|\A\|\tn{\x_i}\|\B\|\tn{\y_i}\leq \|\A\|\|\B\|2\sqrt{d_1d_2}.
\]
Now, still conditioned on $\ell_2$ bounds, we apply matrix Chernoff to obtain
\[
\Pro(\|\sum_{i=1}^n (\x_i\bt\y_i)(\x_i\bt\y_i)^T\|\geq 6 \|\A\|^2\|\B\|^2qn ) \leq d_1 d_2\exp(-\frac{qn}{4d_1d_2})
\]
where we used the fact that $\|\A\|,\|\B\|$ cancels out in the exponent and $2e\leq 6$. To conclude, use the definition of $q$, to get $ d_1 d_2\exp(-\frac{qn}{4d_1d_2})\leq \exp(-\frac{qn}{8d_1d_2})$.
\end{proof}

\subsection{Proof of Theorem \ref{rand act main thm}: Random activations}
The following theorem characterizes the effect of a chain of random activations.
\begin{theorem} \label{thm rand act}Let $\ob\in\R^{h_D}$ be the output layer vector. $\{\Vb_{i-1}\in \R^{h_i\times h_{i-1}}\}_{i=0}^{D}$ be matrices with $h_0=p$ and  define $\Vb_{D}=\diag(\ob)$. Let $\{\rb_i\in \R^{h_i}\}_{i=1}^D$ be independent vectors with i.i.d.~Rademacher entries. Let $\g$ be an isotropic subgaussian vector. For some $0\leq a\leq D$, consider the vectors defined as
\begin{align}
&\eta_{a}=\rb_{a}\bd\dots\Vb_1(\rb_1\bd(\Vb_0 \x))\\
&\theta_{a}=\rb_{a+1}\bd(\Vb_{a+1}^T\dots(\rb_{D-1}\bd(\Vb_{D-1}^T(\rb_D\bd \ob))))
\end{align}
Let $\rmin{\cdot}$, $\rmax{\cdot}$ denote the smallest and largest row length of the input matrix. Given $a>0$, define the quantities 
\begin{align}
&\gamma_{i,j}=\prod_{k=i}^j\|\Vb_k\|^2,~ \alpha_{i,j}=\begin{cases}\prod_{k=i}^j\rmax{\Vb_k}^2~\text{if}~i,j<a\\\prod_{k=i}^j\rmax{\Vb_k^T}^2~\text{if}~i,j>a\end{cases},~\beta_{i,j}=\begin{cases}\prod_{k=i}^j\rmin{\Vb_k}^2~\text{if}~i,j<a\\\prod_{k=i}^j\rmin{\Vb_k^T}^2~\text{if}~i,j>a\end{cases}
\end{align}
and define $\balpha_i=\prod_{j\neq i}\alpha_{j,j}$ (similarly for $\babeta_i,\bgamma_i$).

Conditioned on everything but $\x,\rb_{D}$, subgaussian norms of $\eta_a=\eta_a(\x),\theta_a=\theta_a(\rb_D)$ satisfies the following properties:
\begin{itemize}
\item $\tsub{\eta_a}^2\leq \tsub{\x}^2\gamma_{0,a-1}$.
\item $\tsub{\theta_a}^2\leq c\gamma_{a+1,D}$.
\item $\te{\theta_a\bt\eta_a}^2\leq c\tsub{\x}^2\bgamma_a$
\end{itemize}

Furthermore, covariance obeys
\begin{itemize}
\item $\alpha_{0,a-1}\succeq\bSi(\eta_a)\succeq \beta_{0,a-1}$.
\item $\alpha_{a+1,D}\succeq\bSi(\theta_a)\succeq \beta_{a+1,D}$.
\item $\balpha_a\succeq \bSi(\theta_a\bt\eta_a)\succeq \babeta_a$.
\end{itemize}
\end{theorem}

\begin{proof} We first show the result for $\tsub{\eta_a}$. The proof is by induction. First, using the fact that subgaussian norm is (at most) scaled by spectral norm, $\tsub{\eta_{i+1}}\leq \|\Vb_i\|\tsub{\eta_i}=\|\rb_i\bd\Vb_i\|\tsub{\eta_i}$. Inductively, this implies $\tsub{\eta_i}\leq \prod_{j=0}^{i-1} \|\Vb_i\|\tsub{\x}$. For $\theta_i$, we use the same argument combined with the fact that $\tsub{\rb_D\bd\ob}\leq \sqrt{c}\|\ob\|_\infty=\sqrt{c}\|\Vb_{D}\|$.

To show subexponentiality, we will apply the Hanson-Wright Corollary \ref{asym han}. Observe that $\eta_a=\A\x,\theta_a=\B\rb_D$ where $\A,\B$ are multiplications of intermediate $\Vb_i,\rb_i$'s and have bounded spectral norms. Conditioned on $\A,\B$, applying standard Hanson-Wright Lemma on $\theta_a^T\Vb_a\eta_a$ we have that 
\begin{align}
\Pro(|\theta_a^T\Vb_a\eta_a|\geq t)&=\Pro(|\rb_D^T\B\Vb_a\A\x|\geq t)\leq \exp(-c\min\{\frac{t^2}{\tsub{\x}^2\|\A\Vb_a\B\|_F^2},\frac{t}{\tsub{\x}\|\A\Vb_a\B\|})\\
&\leq \exp(-c\min\{\frac{t^2}{\tsub{\x}^2\|\A\|^2\|\B\|^2\|\Vb_a\|_F^2},\frac{t}{\tsub{\x}\|\A\|\|\B\|\|\Vb_a\|}\})\\
&\leq \exp(-c\min\{\frac{t^2}{\phi^2\|\Vb_a\|_F^2},\frac{t}{\phi\|\Vb_a\|}\})\\
\end{align}
where we used the definition that $\phi:=\tsub{\x}\prod_{i\neq a, 0\leq i \leq D} \|\Vb_i\|\geq \|\A\|\|\B\|\tsub{\x}$. Now, that we obtain the mixed tail bound, applying Lemma \ref{orlicz} and observing $\theta_a^T\A\eta_a=(\theta_a\bt\eta_a)^T\abb$, this implies that
\[
\te{\theta_a\bt\eta_a}\leq \sqrt{c'}\phi
\] for some absolute constant $c'>0$.

Next, we focus on the covariance. First, observe that thanks to $\rb_{a-1},\rb_a$, entries of $\eta_a,\theta_a$ are zero mean with independent signs, hence their covariance and $\bSi(\theta_a\bt\eta_a)$ are diagonal. With this, we will lower and upper bound the covariance. Without losing generality, we prove the minimum eigenvalue by induction. Suppose $\eta_i$ obeys our bound and consider $\eta_{i+1}$. Denote $j$th row of $\Vb_i$ by $\vb_{i,j}$. Observe that
\[
\E[(\vb_{i,j}^T\eta_i)^2]=\vb_{i,j}^T\E[\eta_i\eta_i^T]\vb_{i,j}\succeq \vb_{i,j}^T(\min_{k\leq h_{i}}\E[\eta_{i,k}^2]\Iden)\vb_{i,j}\geq \tn{\vb_{i,j}}^2\prod_{j=0}^{i-1}\rmin{\Vb_j}^2\geq \prod_{j=0}^{i}\rmin{\Vb_j}^2.
\]
This finishes the proof of minimum eigenvalue. Identical upper bound with $\rmax{\cdot}$ applies to maximum. To address $\theta_i$, we follow the same strategy combined with the fact that $\theta_{D}=\rb_D\bd\ob=\Vb_{D}^T\rb_D$. Finally, the covariance of $\eta_a\bt\theta_a$ is obtained by Kronecker producting the covariance matrices $\bSi(\eta_a)\bt\bSi(\theta_a)$ and its eigenvalues are given by the multiplication of eigenvalues of individual covariances i.e. $\lambda_i(\bSi(\eta_a))\lambda_j(\bSi(\theta_a))$.

\end{proof}

\noindent The next theorem is our main result on learning with random activations and can be specialized to prove Theorem \ref{rand act main thm}.
\begin{theorem}[Proof of Theorem \ref{rand act main thm}] Consider the random activation model described in Definition \ref{randact} and recall the definitions in Theorem \ref{thm rand act}. Suppose $\Cc$ is convex and closed, and data points $\{\x_i\}_{i=1}^n\sim \x$ are i.i.d.~isotropic subgaussian vectors. Define the network condition number
\[
\kappa=\kappa(\W)=\frac{\tsub{\x}^2\bgamma_{\ell}}{\babeta_{\ell}},
\] and set $\upsilonb=C\kappa\log^2(C\kappa)$ for some constant $C>0$. Suppose $q=\max\{1,n^{-1}h_{\ell+1}h_{\ell}\log h_{\ell+1}h_{\ell}\}$.
\[
n\geq \upsilonb^4(\omega_n(\Tc)+t)^2.
\]
Pick $\mu=\frac{1}{6q\bgamma_{\ell}}$. Starting from an arbitrary point $\W=\W_0$, projected gradient descent iterations
\[
\W_{i+1}=\Pc_{\Cc}(\W_i-\mu \grad{\W_i})
\]
obey
\[
\tn{\W_{i+1}-\Ws}\leq (1-\frac{1}{q\upsilonb^{4}})^i\tn{\W-\Ws}
\]
with probability $1-\exp(-qn/h_{\ell+1}h_{\ell})-n\exp(-\order{h_{\ell}/\tsub{\x}^2})-n\exp(-\order{h_{\ell+1}})-\exp(-n/\upsilonb^2)-2\exp(-\order{\min\{t\sqrt{n},t^2\}})$.
\end{theorem}
\begin{proof} We first write the gradient for a single sample $\x$ which comes with random activations $\{\rb_j\}_{i=1}^D$. Once we characterize the behavior of single sample, we will follow up by averaging to obtain ensemble gradient of samples $\{\x_i\}_{i=1}^n$.

The gradient iteration with random activations has a much simpler form compared to Theorem \ref{good thm}. First, observe that, since all other layers are fixed, the input to $\ell$th layer is given by $\hat\x_i=\rb_{\ell-1}\bd(\Vb_{i-1}\rb_{i-2}\dots\Vb_0\x_i)$. Similarly, output vector collapses to 
\[
\obh_i=(\ob^T \rb_D\bd (\Vb_{D-1}\cdot\diag(\rb_{\ell}))^T=(\diag(\rb_{\ell}) \Vb_{\ell+1}^T\dots \Vb_{D-1}^T)\rb_D\bd\ob.
\]
With this, the gradient of the $i$th label with respect to $\ell$th layer at $\Vb_{\ell}=\Ub$ is given by
\begin{align}
\nabla\Lc_i(\Ub)&=\frac{\pa (\obh_i^T\Ub\xh-\obh_i^T\Ws\xh)^2}{\pa \Ub}\\
&=(\obh_i^T\Ub\xh-\obh_i^T\Ws\xh)\obh_i\bt \xh\\
&=(\obh_i\bt \xh)(\obh_i\bt \xh)^T\text{vec}(\Ub-\Ws)\\
&=(\obh_i\bt \xh_i)(\obh_i\bt \xh_i)^T\zb\\
&=\y_i\y_i^T\zb
\end{align}
where $\zb=\text{vec}(\Zb)=\text{vec}(\Ub-\Ws)$ and $\y_i=\obh_i\bt \xh_i$. Denoting $\Y=[\y_1~\dots~\y_n]$, the population gradient iteration is given by
\begin{align}
\zb_{\tau+1}=(\Iden-\mu n^{-1}\Y\Y^T)\zb
\end{align}
We simply need to argue the properties of $n^{-1}\Y\Y^T$ in a similar fashion to Theorems \ref{spec bounds} and \ref{rsv hessian main}. In particular, Theorem \ref{thm rand act} shows that columns $\y_i$ are subexponential with norm at most $\sqrt{\bgamma_{\ell}}=\tsub{\x}\prod_{i\neq \ell}\|\Vb_i\|$. Consequently, we first apply Lemma \ref{rsv mean} to obtain that a lower bound on the restricted eigenvalue. In particular, from Theorem \ref{thm rand act} we have
\[
\E\tn{\Y^T\vb}^2\geq \babeta_a
\]
Now we apply Theorem \ref{rsv mean}. Let us set the parameters: Subexponential norm scaled by minimum singular value is 
\[
\Kb=\order{\sqrt{\bgamma_a/\babeta_a}}
\]
so that $\upsilonb^{-1}=\upsilon=\order{\Kb^2\log^24\Kb}^{-2}=\order{\kappa(\W)\log^2\kappa(\W)}$. Hence if $n\geq Ct\upsilonb^4\omega_n^2(T)$, for all $\vb\in \Tc$, with $1-\exp(-n/\upsilonb^2)-2\exp(-c\min\{t\sqrt{n},t^2\})$ probability, we have that
\[
n^{-1}\tn{\Y^T\vb}^2\geq {\babeta_a }{\upsilonb^{-3}}.
\]
Next, applying Lemma \ref{gen spec} and Theorem \ref{thm rand act}, we obtain an upper bound on the spectral norm obeys
\[
n^{-1}\|\Y\|\leq 6q\bgamma_a.
\]
with probability $1-\exp(-qn/h_{\ell+1}h_{\ell})-n\exp(-\order{h_{\ell}/\tsub{\x}^2})-n\exp(-\order{h_{\ell+1}})$. Combining this with Lemma \ref{combine all} (where $H_2,H_3=0$), we conclude that
\[
\tn{\W_i-\Ws}^2\leq (1-\frac{\babeta_a \upsilonb^{-3}}{6q\bgamma_a})^i\tn{\W_0-\Ws}^2
\]
where the learning rate is $6q\bgamma_a$. Simplifying the convergence rate $1-\rho$, we obtain
\[
\rho=1-\frac{\babeta_a {\upsilonb^{-3}}}{6q\bgamma_a}\leq 1-\frac{1}{q\upsilonb^4}.
\]
\end{proof}


\section{Result on subexponential restricted singular value}
This section is dedicated to the understanding the properties of neural network Hessian along restricted directions. These restricted directions are dictated by the feasible ball $\Tc$. 
\subsection{Effective subexponentiality of data points}
In this section, we discuss why subexponentiality occurs in neural network gradient whether we are using standard activation functions or randomized activation. We utilize results from the recent work \cite{sol2017}.

\begin{corollary} [Asymmetric Hanson-Wright] \label{asym han}Let $\x\in\R^{d_1},\y\in\R^{d_2}$ be vectors with i.i.d.~subgaussian entries and assume $K_x,K_y$ are respective upper bounds on subgaussian norm of their entries respectively. Given $\A\in\R^{d_1\times d_2}$, we have
\[
\Pro(|\x^T\A\y|\geq t)\leq 2\exp(-c\min\{\frac{t^2}{K_1^2K_2^2\|\A\|_{F}^2},\frac{t}{K_1K_2\|\A\|}\}).
\]
\end{corollary}
\begin{proof} This directly follows from symmetric result \cite{rudelson2013hanson}. Observe that $\E[\x^T\A\y]=0$. Set $r=\sqrt{K_x/K_y}$. To symmetrize the multiplication, write $\B=[0~\A/2;~\A^T/2~0]\in\R^{(d_1+d_2)\times (d_1+d_2)}$, $\z=[\x/r~r\y]\in\R^{d_1+d_2}$ and apply the symmetric Hanson-Wright inequality on $\z^T\B\z=\x^T\A\y$. To conclude, observe that $\|\B\|_{F}^2=\|\A\|_F^2/2$ and $\|\B\|=\|\A\|/2$ and observe that $\tsub{\z}\leq \order{\sqrt{K_xK_y}}$.
\end{proof}
\begin{corollary} Consider $\x,\y$ from Lemma \ref{asym han}. $\te{\x\bt\y}\leq K_xK_y/\sqrt{c}$.
\end{corollary}
\begin{proof} Combining Lemma \ref{orlicz} and Corollary \ref{asym han}, for any unit vector $\abb=\text{vec}(A)$, we have that
\[
\Pro(|\x^T\A\y|\geq t)\leq \exp(-c\min\{\frac{t^2}{\sqrt{K_xK_Y}^4},\frac{t}{\sqrt{K_xK_Y}^2\|\A\|}\})\leq \exp(-c\min\{\frac{t^2}{K_x^2K_y^2},\frac{t}{K_xK_y}\})
\]
Hence, $\te{\x\bt\y}\leq K_xK_y$.
\end{proof}
Proof of the next lemma follows a similar argument to Lemma  of \cite{sol2017} but refines the final estimate.
\begin{lemma}\label{functional hs} Let $\g\sim\Nn(\Iden_p)$ and $h(\g)\in\R^h$ be an $L$-lipschitz function of $\g$. Then, given a matrix $\A$, we have that
\[
\Pro(|h(\g)^T\A\g-\E[h(\g)^T\A\g]|\geq t)\leq 2\exp(-c\min\{\frac{t^2}{L^2\tf{\A}^2},\frac{t}{L\|\A\|}\})
\]
Applying Lemma \ref{orlicz}, this implies
\[
\te{h(\g)\bt \g}\leq \order{L}
\]
\end{lemma}
\begin{proof}  We repeat the argument for the sake of completeness. The result is obtained by using Hanson-Wright inequality for random vectors exhibiting ``convex concentration property''. This property holds for $\h=[\sqrt{L}\g~h(\g)/\sqrt{L}]$ as i) $\h$ is $K=\sqrt{2L}$-Lipschitz function of $\g$ and ii) any univariate $1$-Lipschitz function of $\h$ is still a $\sqrt{2L}$ Lipschitz function of $\g$ which concentrates exponentially fast. In particular, observing
\[
h(\g)^T\A\g=(h(\g)/\sqrt{L})^T\A(\sqrt{L}\g)
\]
asymmetric version of main theorem of \cite{adamczak2015note} yields (in a similar fashion to Corollary \ref{asym han})
\[
\Pro(|h(\g)^T\A\g-\E[h(\g)^T\A\g]|\geq t)\leq 2\exp(-c\min\{\frac{t^2}{K^4\tf{\A}^2},\frac{t}{K^2\|\A\|}\})
\]
where constant $K$ is $\sqrt{2L}$.
\end{proof}
\begin{lemma}[Lemma $4.5$ of \cite{sol2017}] \label{orlicz}Assume a random variable obeys the condition
\[
\Pro(|x|\geq t)\leq 2\exp(-c\min\{t^2/a,t/b\})
\]
Then, its subexponential norm obeys $\te{x}\leq 9\max\{\sqrt{a/c},b/c\}$.
\end{lemma}

\subsection{Subexponential restricted eigenvalue}
This section provides our main results on restricted singular values of matrices with independent subexponential rows. This question is inherently connected to the work by Sivakumar et al. \cite{banerjee}. However, their results only apply to rows with i.i.d. subexponential entries whereas our bounds apply to subexponential rows that not necessarily contain independent entries. Unfortunately, this prevents us from utilizing their bounds.


\begin{theorem} \label{main subexp}Let $\{\x_i\}_{i=1}^d$ be independent subexponential vectors with $\te{\cdot}$ norm at most $K$. Suppose covariance of $\x$ satisfy $\bSi(\x) \geq \kappa\Iden_d$. Form $\X=[\x_1~\dots~\x_n]^T$. Define $\Kb=K/\sqrt{\kappa}$ and $\upsilon=\order{\somelg}^{-2}$. Given a subset of unit sphere $T$, with probability $1-\exp(-n\upsilon^2)$, we have that
\begin{align}
\inf_{\vb\in T}\tn{\X\vb}\geq \sqrt{{\kappa n}{\upsilon^3}}-cK\omega_n(T)\label{rsv equation}
\end{align}
\end{theorem}
\begin{proof} The proof follows from Proposition $5.1$ of \cite{TroppConvex} which is Mendelson's small ball method. First, we estimate the tail quantity
\[
Q_{\eps}(T,\x) =\inf_{\vb\in T} \Pro(|\x^T\vb|\geq \eps)
\]
This is based on Lemma \ref{some log} which yields the tail bound
\[
\Pro(Z\geq \sqrt{\kappa\upsilon})\geq \upsilon
\]
This implies that for $2\eps=\sqrt{\kappa\upsilon}$
\[
\eps Q_{2\eps}(T,\x)\geq \sqrt{\kappa\upsilon^3}/2.
\]
Next, we obtain the empirical width from Lemma \ref{emp width}. Setting $\y=n^{-1}\sum_{1\leq i\leq n} \x_i$, it yields
\[
\E[\sup_{\vb\in T} |\y^T\vb|]\leq K\order{\omega_n(T)/\sqrt{n}}.
\]
Applying Proposition $5.1$ of \cite{TroppConvex}, combination implies that with probability $1-\exp(-t^2/2)$
\[
\inf_{\vb\in T}\tn{\X^T\vb}\geq 0.5\sqrt{\kappa\upsilon}(\sqrt{n}\upsilon-t)-cK{\omega_n(T)}
\]
Now, we simplify the notation by setting $t=0.5\sqrt{n}\upsilon$ which yields
\[
\inf_{\vb\in T}\tn{\X^T\vb}\geq 0.25\sqrt{{\kappa n}{\upsilon^3}}-cK{\omega_n(T)}
\]
with $1-\exp(-n\upsilon^2/8)$ probability. Making $\upsilon$ smaller by a constant factor do not affect the results. Scaling $\upsilon$ by a factor of $1/\sqrt{8}$, we obtain 
\[
\inf_{\vb\in T}\tn{\X^T\vb}\geq \sqrt{{\kappa n}{\upsilon^3}}-cK{\omega_n(T)}
\]
with probability $1-\exp(-n\upsilon^2)$.
\end{proof}

\begin{corollary} \label{cor subexp}Consider the setup in Theorem \ref{main subexp}. Suppose $n\geq C\upsilon^{-4}\omega_n^2(T)$ for some absolute constant $C>0$ where $\upsilon=\order{\somelg}^{-2}$. Then with probability $1-\exp(-n\upsilon^2)$, we have that
\[
\inf_{\vb\in T}\tn{\X\vb}^2\geq {\kappa n}{\upsilon^3}
\]
\end{corollary}
\begin{proof} We study the condition in \eqref{rsv equation}
\[
\sqrt{{\kappa n}{\upsilon^3}}\geq 2cK\omega_n(T)
\]
This holds as soon as $n\geq4c^2\kappa^{-1}\upsilon^{-3}{K^2}\omega^2_n(T)$. Using the fact that $\upsilon^{-1}\geq \frac{K^2}{\kappa}$, the condition is implied by $n\geq c\upsilon^{-4}\omega^2_n(T)$. Finally, constant of $\upsilon$ can be made smaller to account for the $0.5$ multiplier of
\[
\sqrt{{\kappa n}{\upsilon^3}}-cK\omega_n(T)\geq 0.5\sqrt{{\kappa n}{\upsilon^3}}.
\] 
\end{proof}

The following lemma bounds the empirical width for subexponential measurments. It directly follows from well-known generic chaining tools \cite{talagrand2014gaussian}. In particular, we refer the reader to Theorem $3.5$ of \cite{dirksen2013tail}.
\begin{lemma} [Bounding empirical width] \label{emp width}Suppose $T\subset\Bc^d$ and $\x\in\R^d$ is a zero-mean subexponential vector with norm $\te{\cdot}$ at most $K$. Given $\{\x_i\}_{i=1}^n$ i.i.d. copies of $\x$, define the empirical average vector $\y=n^{-1}\sum_i \x_i$. We have that
\begin{align}
&\Pro(\sup_{\vb\in T} |\y^T\vb|\geq cK(\gamma_1(T,\ell_2)/n+\gamma_2(T,\ell_2)/\sqrt{n}+t))\leq 2\exp(-\min\{t,t^2\}n),\\
&\Pro(\sup_{\vb\in T} |\y^T\vb|\geq cK((\omega_n(T)+t)/\sqrt{n}))\leq 2\exp(-\order{\min\{t\sqrt{n},t^2\}}),\\
&\E[\sup_{\vb\in T} |\y^T\vb|]\leq cK\omega_n(T)/\sqrt{n}.
\end{align}
\end{lemma}
\begin{proof} Define the random process $X_{\vb}=\y^T\vb$. Using the fact that $\y$ is i.i.d. average and applying subexponential Chernoff bound, this process satisfies the mixed-tail increments as follows
\[
\Pro(|X_{\vb}-X_{\ub}|\geq t)=\Pro(|\y^T(\ub-\vb)|\geq t)\leq 2\exp(-c'\min\{\frac{nt^2}{K^2\tn{\ub-\vb}^2},\frac{nt}{K\tn{\ub-\vb}}\})
\]
Note that, mixed tail is with respect to scaled $\ell_2$ distances namely $d_1(\vb,\ub)=K\tn{\ub-\vb}/{n}$ and $d_2(\vb,\ub)=K\tn{\ub-\vb}/{\sqrt{n}}$. Hence, we can alternatively write
\[
\Pro(|X_{\vb}-X_{\ub}|\geq t)\leq 2\exp(-c'\min\{\frac{t^2}{d_2^2(\ub,\vb)},\frac{t}{d_1(\ub,\vb)}\})
\]
 Applying Theorem $3.5$ of \cite{dirksen2013tail} and Theorem $2.2.23$ of \cite{talagrand2014gaussian}, we have
\[
\E[\sup_{\vb\in T} |\y^T\vb|]\leq \E[\sup_{\vb\in T} |\y^T\vb|]\leq c(\gamma_1(T,d_1)+\gamma_2(T,d_2))=cK(\gamma_1(T,\ell_2)/n+\gamma_2(T,\ell_2)/\sqrt{n})
\]
Similarly, using $T\subset\Bc^{d}$, the following tail bound holds
\[
\Pro[\sup_{\vb\in T} |\y^T\vb|\geq cK(\gamma_1(T,\ell_2)/n+\gamma_2(T,\ell_2)/\sqrt{n})+ct\Delta(T)K]\leq 2\exp(-\min\{t,t^2\}n).
\]
Using the fact that $\Delta(T)\leq 1$ yields the first tail bound. To obtain perturbed width bounds, we let $S$ be a set satisfying $\clconv(S)\supset T$ and $\text{rad}(S)\leq C$ (recall Definition \ref{pert width}). First observe that
\[
\sup_{\vb\in T} |\y^T\vb|\leq \sup_{\vb\in S} |\y^T\vb|.
\]
Consequently
\[
\E[\sup_{\vb\in T} |\y^T\vb|]\leq \inf_{\clconv(S)\supset T,~\text{rad}(S)\leq C}\E[\sup_{\vb\in S} |\y^T\vb|]\leq \order{K\omega_n(T)/\sqrt{n}}
\]
Next, since $S$ has bounded radius, picking an $S$ approximating $\omega_n(T)\approx \omega(S)+\gamma_1(S)/\sqrt{n}$, we find
\[
 \Pro[\sup_{\vb\in S} |\y^T\vb|\geq cK\omega_n(T)/\sqrt{n}+cCtK]\leq 2\exp(-\min\{t,t^2\}n).
\]
which is the second advertised bound.
%
\end{proof}

In order to address $\ell_1$ norm and nuclear norm constraints, we make use of the following result that allows us to move from nonconvex set to convexified set.
\begin{lemma} \label{lem clconv}Suppose $Y_1$ is a subset of $\clconv(Y_0)$ which is the closure of the convex hull of $Y_0$. For any vector $\ab$
\[
\sup_{\vb\in Y_1}|\ab^T\vb|\leq \sup_{\vb\in Y_0}|\ab^T\vb|.
\]
\end{lemma}
\begin{proof} Using the fact that $Y_1\subset \clconv(Y_0)$, we immediately have that
\[
\sup_{\vb\in Y_1}|\ab^T\vb|\leq \sup_{\vb\in \clconv(Y_0)}|\ab^T\vb|
\]
Observe that any $\vb\in \text{conv}(Y_0)$ can be written as $\vb=\sum_{i\geq 1}\alpha_i\vb_i$ where $\alpha_i\geq 0$ and $\sum_{i\geq 1}\alpha_i=1$, $\vb_i\in Y_0$. Consequently
\[
\sup_{\vb'\in \clconv(Y_0)}|\ab^T\vb'|=\sup_{\vb\in \text{conv}(Y_0)}|\ab^T\vb|\leq\sum_{i\geq 1}\sup_{\vb_i\in Y_0}\alpha_i|\ab^T\vb_i|= \sup_{\vb\in Y_0}|\ab^T\vb|
\]
\end{proof}

\begin{lemma} [Lower bounding subexponential first moment] \label{some log}Suppose $\x\in\R^d$ is a zero-mean subexponential vector with norm at most $K$ and covariance obeying $\bSi(\x)\succeq \kappa\Iden$. Define $\Kb=K/\sqrt{\kappa}$ and the quantity $\upsilon=\order{\somelg}^{-2}$. For all $\ub\in\Sc^{d}$, we have that
\[
\Pro(|\x^T\ub|\geq \sqrt{\kappa\upsilon})\geq \upsilon.
\]
\end{lemma}
\begin{proof} Let $Z=|\x^T\ub|$. Our goal is to obtain an estimate on $\E[Z]$ and then applying Paley-Zygmund to find
\begin{align}
\Pro(Z\geq \theta \E[Z])\geq (1-\theta)^2\frac{\E[Z]^2}{\E[Z^2]}\label{paley z}
\end{align}
Let $\kappa_{\ub}=\E[Z^2]\geq \kappa$, $\Zbb=Z/\sqrt{\kappa_{\ub}}$ and $\Kb=K/\sqrt{\kappa_{\ub}}$. We will obtain a bound for $\Zbb$ and then scale it by $\sqrt{\kappa_{\ub}}$. 
\[
1= \E[\Zbb^2]=\E[\Zbb^2\bgl \Zbb\geq \tau]\Pro(\Zbb\geq \tau)+\E[\Zbb^2\bgl \Zbb\leq \tau]\Pro(\Zbb\leq \tau)
\]
Now, observe that
\[
\E[\Zbb^2\bgl \Zbb\leq \tau]\Pro(\Zbb\leq \tau)\leq \tau \E[\Zbb\bgl \Zbb\leq \tau]\Pro(\Zbb\leq \tau)\leq \tau \E[\Zbb]
\]
This implies
\[
\E[\Zbb]\geq \frac{1-\E[\Zbb^2\bgl \Zbb\geq \tau]\Pro(\Zbb\geq \tau)}{\tau}
\]
We next obtain a good value of $\tau$. Set $\tau=c_0\Kb\log 4\Kb$ for some constant $c_0>0$. Using subexponential tails and using $\Kb\geq 1/2$ (since variance is $1$), observe that
\[
\E[\Zbb^2\bgl \Zbb\geq \tau]\Pro(\Zbb\geq \tau)=\int_{z\geq \tau} 2z\exp(-cz/\Kb)dz=\order{\Kb(\Kb+\tau)\exp(-c\tau/\Kb)}\leq \order{\Kb^2\log(4\Kb) (4\Kb)^{-cc_0}}\leq 1/2.
\]
Consequently
\[
\E[\Zbb]\geq \frac{1-\Kb^{-1}}{c_0\Kb\log 4\Kb}\geq  \frac{1}{2c_0\Kb\log 4\Kb}\implies \frac{\E[\Zbb^2]}{\E[\Zbb^2]}=(4c_0^2\Kb^2\log^24\Kb)^{-1}:=4\upsilon_{\ub}.
\]
where $\upsilon_{\ub}=\order{(\Kb^2\log^2\Kb)^{-1}}$. Observing $\frac{\E[\Zbb^2]}{\E[\Zbb^2]}=\frac{\E[Z^2]}{\E[Z^2]}$, $\E[Z]=\sqrt{\kappa_{\ub}}\E[\Zbb]$, setting $\theta=1/2$ and substituting $\upsilon_{\ub}$ in \eqref{paley z}
\[
\Pro(\Zbb\geq  \sqrt{4\upsilon_{\ub}})\geq  \upsilon_{\ub}\implies \Pro(Z\geq  \sqrt{\kappa_{\ub}\upsilon_{\ub}})\geq  \upsilon_{\ub}.
\]
%
Now, using the fact that $\kappa\leq \kappa_{\ub}$, $\Kb\geq \Kb_{\ub}$ and $\upsilon\geq\upsilon_{\ub}$, for all $\ub$
\[
\Pro(Z\geq  \sqrt{\kappa\upsilon})\geq \Pro(Z\geq  \sqrt{\kappa_{\ub}\upsilon_{\ub}})\geq \upsilon_{\ub}\geq \upsilon.
\]
\end{proof}

\begin{lemma} [Worst case impact of expectation] \label{mean impact} Given set $\Tc$, let $\bSi\in\R^{p\times p}$ satisfy $\inf_{\vb\in \Tc} \vb^T\bSi\vb\geq \alpha$, $\e$ is a fixed vector, $\x$ is a random vector that satisfies $\sup_{\vb\in \Tc} |\x^T\vb|\leq \beta$. Then
\[
\inf_{\vb\in \Tc}\vb^T\bSi\vb+2\vb^T\e\vb^T\x+(\vb^T\e)^2\geq \alpha-\beta^2.
\]
Suppose $\sqrt{\alpha}\geq \sqrt{2}\beta$, then, the lower bound becomes $\alpha/2$. 
\end{lemma}
\begin{proof} Set $\theta=|\vb^T\e|$. Given $\vb\in \Tc$, we have
\[
\vb^T\bSi\vb+2\vb^T\e\vb^T\x+(\vb^T\e)^2\geq \alpha -2\theta |\x^T\vb|+\theta^2\geq \alpha -|\x^T\vb|^2\geq \alpha-\beta^2.
\]
If $\alpha\geq \sqrt{2}\beta$, then $\alpha-\beta^2\geq \alpha/2$.
\end{proof}

The main result of this section bounds RSV of matrices with i.i.d. subexponential rows possibly having nonzero means.
\begin{theorem}[Bounding RSV with mean] \label{rsv mean}Suppose we are given $n$ i.i.d.~vectors $\ab_i$ with subexponential norm at most $K$ (when centered) and covariance $\bSi_{\ab}\succeq  \kappa\Iden_d$. Form the matrix $\A=[\ab_1~\dots~\ab_n]^T$. Let $\Tc$ be a subset of unit sphere and recall the definition
\[
\sigma^2(\A,\Tc)=\min_{\ub\in \Tc} \sum_{i=1}^n (\ub^T\ab_i)^2
\]
Let $\Kb=K/\sqrt{\kappa}$, $\upsilon=\order{\somelg}^{-2}$. Suppose
\[
n>C\upsilon^{-4}(\omega_n(T)+t)^2.
\]
With probability $1-\exp(-n\upsilon^2\})-2\exp(-\order{\min\{t\sqrt{n},t^2\}})$, we have that
\[
\sigma^2(\A,\Tc)\geq \kappa \upsilon^3n.
\]
\end{theorem}
\begin{proof} This result follows by combining Theorem \ref{main subexp}, Lemma \ref{emp width} and Lemma \ref{mean impact}. 
The proof will be done in two steps. Set $\e=\sqrt{n}\E[\ab_i]$, set $\bSi=\sum_i(\ab_i-\e)(\ab_i-\e)^T$ and $\y=n^{-1/2}\sum_{i=1}^n(\ab_i-\e)$. Given $\vb\in \Tc$ we have that
\[
\tn{\A\vb}^2=\sum_i (\ab_i^T\vb)^2=\vb^T\bSi\vb+2\y^T\vb\e^T\vb+(\e^T\vb)^2
\]
which has the setup in Lemma \ref{mean impact}. Now set $\beta=n^{-1/2}\sup_{\vb\in \Tc}|\vb^T\sum_{i=1}^n(\ab_i-\e) |$.
Applying Lemma \ref{emp width}, we have that with probability $1-2\exp(-\order{\min\{t\sqrt{n},t^2\}})$
\[
\beta\leq c_1K(\omega_n(T)+t)
\]
Secondly, setting $\alpha=\inf_{\vb\in \Tc} \vb^T\bSi\vb$, applying Theorem \ref{cor subexp}, with probability $1-\exp(-n\upsilon^2)$
\[
\alpha^{1/2}\geq  \sqrt{{\kappa n}{\upsilon^3}}
\]
We require $\alpha^{1/2}\geq \sqrt{2}\beta$. This occurs because by initial assumption $n\geq \order{\upsilon^{-4}(\omega_n(T)+t)^2)}$. Combining this with $\upsilon^{-1}\geq K^2/\kappa$, we have $\kappa n \upsilon^3\geq 2c_1^2K^2(\omega_n(T)+t)^2$. Overall, with the desired probability
\[
\sigma^2(\A,\Tc)\geq 0.5\alpha\geq 0.5\kappa n\upsilon^3.
\]
Finally, adjust $\upsilon$ by a constant to discard the $0.5$ factors.
\end{proof}


\subsection{Proof of Theorem \ref{rest hess thm}: Main result on restricted eigenvalue}
Our main result is a probabilistic lower bound on the restricted eigenvalue of Hessian. Before stating the result, we define
\[
\Theta:=\Theta_{\sigma,\Ws}=\order{\frac{\lip^2\s_{\max}^2\kappa^{2}(\ob)\kappa^{h+2}(\Ws)}{\zeta(\s_{\min})}}~~~\text{and}~~~\upsilon=(\Theta\log^2(\Theta))^{-1}
\]
where the constant factor of $\Theta$ comes from Theorem \ref{rsv mean}. Based on these definitions, the result is stated below.
\begin{theorem} [RSV for Hessian] \label{rsv hessian main} Suppose $n>\order{\upsilon^{-4}(\omega_n(\Tc)+t)^2}$. Let $\bL=\lip^2\ob_{\max}^2\s_{\max}^2$. Given matrix $\Ws$ and Gaussian inputs $\{\x_i\}_{i=1}^n\sim \Nn(0,\Iden_p)^n$, with probability $1-\exp(-{n\upsilon^2})-2\exp(-\order{\min\{t\sqrt{n},t^2\}})$, we have that, all $\vb\in\Tc$ obeys
\[
\vb^TH_1\vb\geq   \frac{\zeta(\s_{\min})\ob_{\min}^2}{\kappa^{h+2}(\Ws)}\upsilon^3= \bL^2\Theta^{-1}\upsilon^3\geq \bL^2\upsilon^4.
\]
\end{theorem}
\begin{proof} The result is obtained by combining Theorem \ref{rsv mean} and Lemma \ref{min eigen}. First, Lemma \ref{min eigen} states
\[
\E[H_1]\succeq \ob^2_{\min}\zeta(\s_{\min})/\kappa^{h+2}=\lip^2\ob_{\max}^2\s_{\max}^2\Theta^{-1}=\bL^2\Theta^{-1}.
\]
Next, applying Lemma \ref{functional hs} and using the fact that $\x\rightarrow\ob\cdot\sigma'(\Ws\x)$ is $\ob_{\max}\lip\s_{\max}$ lipschitz, subexponential norm of $\rho(\Ws;\x_i)$ obeys,
\[
\te{\rho(\Ws;\x)}\leq c\bL.
\]
To apply Theorem \ref{rsv mean}, define 
\[
\Kb=\frac{c\bL}{\sqrt{\s_{\min}(\E[H_1])}}= c\sqrt{\Theta},
\]
and
$\upsilon=\order{\Kb\log 4\Kb}^{-2}=(\Theta\log^2(\Theta))^{-1}$. With this at hand, applying Theorem \ref{rsv mean}, we obtain that when $n\geq \order{(\omega_n(\Tc)+t)^2\upsilon^{-4}}$, with the desired probability
\[
\inf_{\vb\in \Tc}\vb^TH_1\vb\geq \order{\lip^2\ob_{\max}^2\Theta^{-1}\upsilon^3}
\]
\end{proof}
To show Theorem \ref{rsv hessian main}, we utilized the fact that Hessian is positive definite. In particular, Lemma \ref{min eigen} addresses this issue and provides a lower bound on the minimum eigenvalue of population Hessian.


\subsection{Subexponential set complexity}

In this section, we will introduce and analyze perturbed width which is a unified definition of set complexity. Recall that It is initially introduced in Definition \ref{pert width} and it has dependence on the number of samples $n$. To understand where perturbed width arises from, we introduce Talagrand's $\gamma_a$ functionals and associated helper definitions.
\begin{definition} [Admissible sequence \cite{talagrand2014gaussian}] Given a set $T$ an admissible sequence is an increasing sequence ($A_n$) of partitions of $T$ such that $|A_n| \leq N_n$ where $N_0=1$ and $N_n=2^{2^n}$ for $n\geq 1$.
\end{definition}
For the following discussion $\Delta(A_n(t))$, will be the diameter of the set $S\in A_n$ that contains $t$.
\begin{definition} [$\gamma_a$ functional \cite{talagrand2014gaussian}] \label{gamma functional}Given $a > 0$, and a metric space $(T,d)$ we define \[\gamma_a(T,d) = \inf \sup_{t\in T} \sum_{n\geq 0} 2^{n/a} \Delta(A_n(t)),\] where the infimum is taken over all admissible sequences.
\end{definition}
We will only consider $\ell_2$ norm in this work, so the letter $d$ will be dropped from $D$ and $\gamma_1,\gamma_2$ variables. We should remark that $\gamma_2(T)$ and Gaussian width $\omega(T)$ are trivially related. For some constants $C,c>0$ and for all sets $T$
\[
c\gamma_2(T)\leq \omega(T)\leq C\gamma_2(T).
\]
With this observation, perturbed width is a slight modification of $\gamma_2$ as it has the additional $\gamma_1/n$ term i.e.
\[
\min_{\clconv(S)\supseteq T}\omega(S)= \omega(T)\leq \min_{\clconv(S)\supseteq T}\omega(S)+\frac{\gamma_1(S)}{\sqrt{n}}=\order{\min_{\clconv(S)\supseteq T}\gamma_2(S)+\frac{\gamma_1(S)}{\sqrt{n}}}
\]

\subsubsection{Bounding $\gamma_1(\cdot)$ functional in terms of covering numbers}
Denote $\ell_2$ (or Frobenius)  $\eps$-covering number of a set $T$ by $N(T,\eps)$.

\begin{lemma}[$\gamma_2$ for well-covered sets] \label{simple sum}Suppose $T$ is an arbitrary subset of $S\subset\Bc^{n}$ that admits a covering number $N_\eps(S)\leq (\frac{B}{\eps})^s$ for some $B> 1, s\geq 0$. Then, for some absolute constant $C_\alpha>0$
\[
\gamma_\alpha(T)\leq C_\alpha(s\log C)^{1/\alpha}
\]
In particular, $\gamma_1(T)\leq \order{s\log C}$ and $\gamma_2(T)\leq \order{\sqrt{s\log C}}$.
\end{lemma}
\begin{proof} The proof directly follows from Lemma \ref{dudley lem} by observing
\[
\gamma_\alpha(T)\leq C_\alpha \int_0^\infty (s\log \frac{B}{\eps})^{1/\alpha}d\eps\leq C'_\alpha (s\log B)^{1/\alpha}.
\]
\end{proof}

\begin{lemma}\label{arbit reg} Suppose $T$ is an arbitrary subset of $S$ that admits a covering number $N_\eps(S)\leq (\frac{C}{\eps})^s$ for some $C\geq 2, s\geq 2$. Then
\[
\gamma_1(T)\leq  3\sqrt{s\log_2 s\log_2C}\gamma_2(T)+1.
\]
for some absolute constants $C_1,C_2>0$.
\end{lemma}
\begin{proof}  Let $A_n$ be an admissible sequence of $T$ achieving $\gamma_2(T)$ bound. Define
\[
S_a=\sup_{t\in T} \sum_{n\geq 0} 2^{n/a} \Delta(A_n(t))
\]
We will slightly modify $A_n$ without hurting $S_2$ too much and we will bound $S_1$. We construct admissible $B_n$ as follows. Pick an integer $n_0$ to be determined later. Below $n\leq n_0$, we will set $B_n=A_n$. Above $n>n_0$, elements of $B_{n+1}$ will be the regions corresponding to the tightest $\ell_2$-covering of the regions of $A_n$ of cardinality $2^{2^n}$.

Now, we proceed to understand the impact of this modification. Pick $b_n\in B_{n_0}$. Clearly $b\subset T\subset\Sc^{n}$. Covering $b$ with $2^{2^n}$ elements we obtain that covering radius $\eps_n$ satisfies
\[
2^{2^n}=(\frac{C}{\eps_n})^s\iff -s\log_2 \frac{\eps_n}{C}=2^n\iff \eps_n=C2^{-{2^n/s}}
\]
Hence
\[
2^{n}2^{-{2^n/s}+\log_2 C}=2^{\log_2 C+n-{2^n/s}}:=2^{-c_n}
\]
where $c_n={2^n/s}-n-\log_2 C$. Set $n_0=\log_2(s)+\log_2(\log_2(s))+\log_2(\log_2(C))+\log_2(10)$. Observe that if $n\geq  n_0+c$ for $c\geq 0$,
\begin{align}
2^{n}/s=2^c10\log_2 s\log_2(C)&\geq 10(2^{c}-1)+10\log_2 s\log_2(C)\\
&\geq 2.5\log_2 s+1.5\log_2(C)+5\log_2(C)+10(2^{c}-1)\\
&\geq \log_2 s+\log_2(\log_2(s))+\log_2(\log_2(C))+\log_2(C)+5+10(2^{c}-1)\\
&\geq n_0+\log_2(C)+1+10(2^{c}-1)\\
&\geq n_0+\log_2(C)+2c+1\\
&\geq n+\log_2(C)+c+1.
\end{align}
where we used the fact that $\log_2(x)\leq 1.5\log x\leq 1.5x$ for $x\geq 1$. This implies for $n\geq n_0+c$
\[
c_n={2^n/s}-n-\log_2 C\geq c+1.
\]

Consequently, for $n=n_0+c$, $2^{n} \Delta(B_n(t))\leq 2^{-c-1}$ so that
\[
\sum_{n\geq n_0}2^{n} \Delta(B_n(t))\leq 1.
\]
To proceed, we first observe
\[
\sup_{t\in T} \sum_{n\geq 0} 2^{n/a} \Delta(B_n(t))\leq \sup_{t\in T} \sum_{n\geq 0} 2^{n/a} \Delta(A_n(t))+1
\]
Secondly, we observe 
\[
\sum_{n\leq n_0} 2^{n} \Delta(A_n(t))\leq 2^{n_0/2} \sum_{n\leq n_0} 2^{n/2} \Delta(A_n(t))
\]
where $2^{n_0/2}=3\sqrt{s\log_2 s\log_2C}$. The combination yields
\[
\gamma_1(T)\leq 3\sqrt{s\log_2 s\log_2C}\gamma_2(T)+1
\]
\end{proof}

\subsubsection{Upper bound via Dudley Integral}

The following result related $\gamma_\alpha$ sum to integration over covering numbers. We believe this is a standard result however we state the proof for completeness.
\begin{lemma} \label{dudley lem}Let $N(\eps)$ be the $\eps$ covering number of the set $T$ with respect to $\ell_2$ distance. Then
\[
\gamma_\alpha(T)\leq C_\alpha \int_{0}^\infty \log^{1/\alpha}(N(\eps)) d\eps
\]
where $C_\alpha$ depends only on $\alpha>0$.
\end{lemma}
\begin{proof} Let $e_n$ be the tightest cover size for $2^{2^n}$ points. One can construct an admissable sequence from tightest $e_n$ covers by cartesian producting them and forming the sequence by recursive intersections (for each $i$, intersect partitionings that correspond to the $e_j$-covers for $1\leq j\leq i$). To be precise, let $B_n$ be partition of $T$ induced by an $e_n$ cover of $T$. Given $\{B_i\}$'s, we define $A_n$ inductively as
\[
A_n=\{X\cap Y~\bgl ~X\in A_{i-1},~Y\in B_n\}
\]
This ensures that $A_n$ is admissable. First of all, size of $A_n$ obeys
\[
|A_n|=\prod_{i=0}^n |B_n|=\prod_{i=0}^n 2^{2^i}\leq 2^{2^{n+1}}.
\]
Observe that this implies the following $\gamma_{\alpha}$ upper bound. We can use $A_n$ as the $n+1$th admissable set. Clearly $\Delta(X)\leq e_{n-1}$ for all $X\in A_n$. Pick $C_0=T$ and $C_{i+1}=A_i$.
\begin{align}
\gamma_{\alpha}(T)&=e_0+\sum_{i=1}^\infty 2^{i/\alpha}e_{i-1}\\
&\leq e_0+2^{1/\alpha}\sum_{i=0}^\infty 2^{i/\alpha}e_{i}\leq C_{\alpha}\sum_{i=0}^\infty 2^{i/\alpha}e_{i}
\end{align}
 Hence, we have that $\gamma_{\alpha}(T)\leq\order{ \sum_{n\geq 0} 2^{n/\alpha} e_n}$. Next, we relate this sum to the integral via
\begin{align}
\int_{0}^\infty \log^{1/\alpha} N(\eps) d\eps&=\int_{0}^{e_0} \log^{1/\alpha} N(\eps) d\eps\\
&= \sum_{n=1}^\infty \int_{\eps=e_{n+1}}^{e_{n}} \log^{1/\alpha} N(\eps) d\eps\\
&\geq  (e_0-e_1)\log^{1/\alpha}2+\sum_{n=1}^\infty \int_{\eps=e_{n+1}}^{e_{n}} \log^{1/\alpha} N(e_{n}) d\eps\\
&= (e_0-e_1)\log^{1/\alpha}2+\sum_{n=1}^\infty (e_n-e_{n+1}) \log^{1/\alpha} N(e_{n})\\
&= (e_0-e_1)\log^{1/\alpha}2+\sum_{n=2}^\infty e_n (2^{n/\alpha}-2^{(n-1)/\alpha})+e_1\log^{1/\alpha}2^2\\
&\geq C_\alpha \sum_{n\geq 0} 2^{n/\alpha} e_n
\end{align}
Overall, these yield $\gamma_{\alpha}(T)\leq C_\alpha \int_{0}^{e_0} \log^{1/\alpha} N(\eps) d\eps$.
\end{proof}
%
%

\subsubsection{Bounding perturbed with for specific regularizers}

This section provides perturbed width bounds for specific constraint sets. Gaussian width term is already very well understood. Here, we show how $\gamma_1(\cdot)$ term can be approximated well for constraints of interest.

The following lemma states standard results on covering numbers of subspace, sparse, low-rank constraints. This will help us get perturbed width bounds for nonconvex sets as well as convex sets.
\begin{lemma} [$\eps$-covers of simple sets] \label{eps covers} Over the space $\R^{h\times p}$, unit ball $\Bc^{h\times p}$, set of $s$ sparse matrices and set of rank $r$ matrices, and $d$ dimensional subspaces have the following $\ell_2$ (i.e.~Frobenius norm) covering numbers.
\begin{itemize}
\item Unregularized: $T=\Bc^{h\times p}$ \cite{Vers}: $\log N(T,\eps)\leq ph\log (\frac{3}{\eps})$.
\item Sparse: $T=\{\W\in\Bc^{h\times p},~\|\W\|_0\leq s\}$ \cite{Vers,Cha}: $\log N(T,\eps)\leq s\log (\frac{6hp}{\eps s})$.
\item Low-rank: $T=\{\W\in\Bc^{h\times p},~\text{rank}(\W)\leq r\}$ \cite{candes2011tight}: $\log N(T,\eps)\leq (p+h+1)r\log (\frac{9}{\eps})$.
\item Subspace: $T$ is linear subspace with $\text{dim}(S)=d$ \cite{Vers}: $\log N(T,\eps)\leq d\log (\frac{3}{\eps})$.
\end{itemize}
\end{lemma}

Merging Lemma \ref{eps covers} with Lemma \ref{arbit reg}, we have the following upper bounds on $\gamma_1(T)$ for regularizers of interest. We present both convex and nonconvex constraints in a similar fashion to Table \ref{table summary}.
\begin{lemma}[$\gamma_1$ functionals of specific sets] Let $\Tc$ be the tangent ball as described in \eqref{feasible ball}. We have the following upper bounds on $\gamma_1(\Tc)$ for different regularizers $\Rc$'s for the set $\Cc=\{\Rc(\W)\leq \Rc(\Ws)\}$.
\item {\bf{Unregularized:}} $\gamma_1(\Tc)\leq \order{ph}$.
\item {\bf{$\ell_1$ regularized:}} Suppose $\|\Ws\|_0\leq s$. Then, $\inf_{\clconv(S)\supset \Tc}\gamma_1(S)\leq\order{s\log \frac{6ph}{s}}$.
\item {\bf{Sparsity constraint:}} Suppose $\|\Ws\|_0\leq s$. Then, $\gamma_1(\Tc)\leq \order{s\log \frac{6ph}{s}}$.
\item {\bf{Nuclear norm regularized:}} Suppose $\text{rank}(\Ws)\leq r$. Then, $\inf_{\clconv(S)\supset S}\gamma_1(S)\leq \order{r(p+h)}$.
\item {\bf{Rank constraint:}} Suppose $\text{rank}(\Ws)\leq r$. Then, $\gamma_1(\Tc)\leq \order{r(p+h)}$.
\item {\bf{Subspace constraint:}} $\text{dim}(\Cc)=d$. Then, $\gamma_1(\Tc)\leq \order{d}$.
\item {\bf{Arbitrary regularization:}} For any feasible ball $\Tc$, we have $\gamma_1(\Tc)\leq\order{ \omega(\Tc)\sqrt{ph\log p}}$.
\end{lemma}
\begin{proof} First, let us focus on the listed sets except $\ell_1$, nuclear norm and arbitrary regularization constraints which will be handled later. All remaining sets have good covering bounds i.e. $\log N(T,\eps)\leq s\log\frac{C}{\eps}$ and Lemma \ref{eps covers} is applicable.
Consequently, applying Lemma \ref{simple sum}, we obtain the bounds
\[
\gamma_2^2(T),~\gamma_1(T)\leq \order{s\log C}.
\]
Substituting the $s,C$ information yields the result via
\begin{itemize}
\item Set $s=ph$, $C=3$ for unregularized.
\item Set $s=k$, $C=\frac{6hp}{k}$ for $k$ sparse.
\item Set $s=r(p+h+1)$, $C=9$ for $r$ rank.
\item Set $s=d$, $C=3$ for subspace.
\end{itemize}
 Now, we focus on the convex $\ell_1$ and nuclear norm constraints. $\ell_1$ proof is strictly simpler hence we will focus on nuclear norm.  Following similar argument to \cite{pilanci2014randomized}, we first use the fact that
 \[
 \Tc\subset \{\tf{\W}\leq 1\bgl \|\W\|_\star\leq 2\sqrt{r}\tf{\W}\}=C_{2\sqrt{r}}.
 \]
 Next, via Lemma \ref{inclusion}, the set $C_{2\sqrt{r}}$ is superset by the low-rank set
 \[
 C_{2\sqrt{r}}\subset\clconv(\{\Ub~\bgl~\text{rank}(\Ub)\leq 4r,~\tf{\Ub}\leq 3\})=R_{4r,3}
 \]
Consequently, we obtain
\[
\inf_{\clconv(S)\supset T} \gamma_1(S)\leq \omega_1(R_{4r,3})=\order{r(p+h)}
\]
 Identical argument applies to $\ell_1$ and $\|\cdot\|_0$ pair. Finally, to show the result for arbitrary constraint, apply Lemma \ref{arbit reg} and use the fact that $T\subset\Bc^{h\times p}$.
\end{proof}

The following lemma is a restatement of Lemma $13$ of \cite{pilanci2014randomized}.
\begin{lemma} \label{inclusion}Given $s$-sparse $\Ws$, consider the $\ell_1$ norm feasible ball
\begin{align}
T_{\ell_1}=\Tc=\Bc^{h\times p}\bigcap \cl\left(\left\{\alpha\Ub\in\R^{hp}~\bgl~\|\Ws+\Ub\|_1\leq \|\Ws\|_1,~\alpha\geq 0\right\}\right)
\end{align}
We have that $T_{\ell_1}\subset \clconv(\{\Ub~\bgl~\|\Ub\|_0\leq 4s,~\tf{\Ub}\leq 3\})$. Similarly, consider a rank $r$ matrix $\Ws$ and its nuclear norm feasible ball
\begin{align}
\Tc_{\star}=\Bc^{h\times p}\bigcap \cl\left(\left\{\alpha\Ub\in\R^{hp}~\bgl~\|\Ws+\Ub\|_\star\leq \|\Ws\|_\star,~\alpha\geq 0\right\}\right)
\end{align}
We have that $T_{\star}\subset \clconv(\{\Ub~\bgl~\text{rank}(\Ub)\leq 4r,~\tf{\Ub}\leq 3\})$.
\end{lemma}

\section{Equivalence of CNNs and projected fully-connected network}
The overall degrees of freedom of this model is same as $\kb=\{\kb_i\}_{i=1}^k$ and is equal to $kb$. Within our framework, we need to project $\W$ to its constraint space. We will now argue that, the projected gradient iterations are {\emph{exactly}} the convolutional gradient iterations. The lemma below illustrates this.
\begin{proof}[Proof of Lemma \ref{conv space}] It is clear that $\Cc$ is a linear subspace as addition and scaling of convolutional weight matrices stays a convolutional weight matrix. The dimension of the space follows from the fact that $\FC(\cdot)$ operation is bijective and $\kb$ spans a $kb$ dimensional subspace. More formally, define $kb$ matrices $\{\M^{i,l}\}_{(i,l)=(0,0)}^{(k-1,b-1)}\in\R^{h\times p}$ with entries parametrized as
\begin{align}
\M^{i,l}_{j_1r+j_2,k}=\begin{cases}1~\text{if}~j_1=i~\text{and}~k=j_2s+l\\0~\text{else}\end{cases}\label{conv project}
\end{align}
$\M^{i,l}$ picks the $l$th entry of the $i$th kernel. It is clear that $\M^{i,j}$ are orthogonal to each other (due to non-overlapping support) and can represent all convolutional weight matrices. Hence $\text{dim}(\Cc)=kb$.
\end{proof}
\begin{proof}[Proof of Lemma \ref{conv grad}] The proof follows from the structure of $\Cc$. First, let us again write the gradient with respect to $(i,l)$th row
\[
\gradf{\W}=(f_{\FC}(\W)-\y)\sigma'(\w_{i,l}^T\x)\x
\]
Similarly gradient of $f_{CNN}$ with respect to $i$th kernel is given by
\[
\gradc{\kb}=\sum_{l=1}^r(f(\kb)-\y)\sigma'(\kb_i^T\x_{l})\x_{l}
\]
Denote $f_{\FC}(\W)-\y=f(\kb)-\y=L$ and $\sigma'(\kb_i^T\x_{j,l})=\sigma'(\w_{i,l}^T\x)=a_{i,l}$ which simplifies the notation to
\[
\gradf{\W}_{i,l}=La_{i,l}\x, ~\gradc{\kb}_i=L\sum_{l=1}^r a_{i,l}\x_{l}
\]
Observe that weight sharing occurs between $\{\gradf{\W}_{i,l}\}_{l=1}^r$. Hence, we will connect $\Pc_{\Cc}(\sum_{l=1}^r\gradf{\W}_{i,l})$ to $\gradc{\kb}_i$.
Following the basis construction of \eqref{conv project}, projection of $\gradf{\W}$ is given by summing up the inner products with basis matrices $\M^{i,j}$ i.e.
\[
\Pc_{\Cc}(\gradf{\W})=\frac{1}{r}\sum_{(i,j)=(0,0)}^{(k-1,b-1)}\li\M^{i,j},\gradf{\W}\ri\M^{i,j}
\]
where $r=\tf{\M^{i,j}}^2$ is the normalization. Inner product with $\M^{i,j}$ ensures that we average the entries of $\gradf{\W}$ that corresponds to the $j$th entry of $i$th kernel. Letting $\e_j$ be the $j$th element of standard basis, we have
\[
\frac{1}{r}\li\M^{i,j},\gradf{\W}\ri\M^{i,j}=\FC(\e_j\e_j^T\gradc{\kb}_i)
\]
Summing these up for all $i,j$ we obtain
\[
\Pc_{\Cc}(\gradf{\W})=\frac{1}{r}\sum_{(i,j)=(0,0)}^{(k-1,b-1)}\li\M^{i,j},\gradf{\W}\ri\M^{i,j}=\frac{1}{r}\sum_{i,j}\FC(\e_j\e_j^T\gradc{\kb}_i)=\frac{1}{r}\FC(\gradc{\kb})
\]
which completes the proof.
To show equivalence of the gradient iterations, we make use of the fact that $\Cc$ is a linear subspace hence projection of the sum is equal to the sum of the projections.
\end{proof}
\section{Proof of Lemma \ref{lemma rad}}

\begin{proof} Let $\rb=\{\rb_i\}_{i=1}^n$ be i.i.d. Rademacher random variables. Set $\Tc=\{\W\in\R^{h\times p}\bgl\W\in\Cc,~\|\W\|\leq \alpha\}$. We are interested in the expected supremum
\[
\text{Rad}(\Fc)=n^{-1}\E_{\{\x_i\}_{i=1}^n}[\E_{\rb}[\sup_{\ob,\W\in\Tc} \sum_{i=1}^n\rb_i\ob^T \sigma(\W\x_i)]].
\]
Define the variable $s(\ob,\W\x_i)=\ob^T\sigma(\W\x_i)-\E[\ob^T\sigma(\W\x_i)]$ and set $e(\ob,\W)=\E[\ob^T\sigma(\W\x_i)]$. First observe that given $\z_1,\z_2$
\[
\ob^T\sigma(\W\z_1)-\ob^T\sigma(\W\z_2)\leq L\tn{\ob}\|\W\|\tn{\z_1-\z_2}
\]
which implies $\ob^T\sigma(\W\x)$ is $L\tn{\ob}\|\W\|$ Lipschitz function of $\x$. This implies, for any $\ob,\W$
\[
\Pro(|\ob^T\sigma(\W\x)-\E[\ob^T\sigma(\W\x)]|\geq t)\leq 2\exp(-\frac{t^2}{2L\|\W\|^2\tn{\ob}^2})
\]
or alternatively $\tsub{s(\ob,\W\x)}\leq L\|\W\|\tn{\ob}\leq LR_{\W}R_{\ob}:=\bL$. We will split the analysis into two parts by writing
\begin{align*}
\E_{\{\x_i\}_{i=1}^n}[\E_{\rb}[\sup_{\ob,\W\in\Tc} \sum_{i=1}^n\rb_i\ob^T \sigma(\ob,\W\x_i)]]&\leq \E_{\{\x_i\}_{i=1}^n}[\E_{\rb}[\sup_{\ob,\W\in\Tc} \sum_{i=1}^n\rb_i\ob^T s(\ob,\W\x_i)+\sup_{\ob,\W\in\Tc} \sum_{i=1}^n\rb_ie(\ob,\W)]]\\
&=\E_{\rb}[\E_{\{\x_i\}_{i=1}^n}[\sup_{\ob,\W\in\Tc} \sum_{i=1}^n\rb_i\ob^T s(\ob,\W\x_i)]]+\E_{\rb}[\sup_{\ob,\W\in\Tc} \sum_{i=1}^n\rb_ie(\ob,\W)]]
\end{align*}
We first bound the $e(\ob,\W)$ term. First, recalling $f(\x)=\ob^T\sigma(\W\x)$, observe that $|f(\x)-f(0)|\leq \bL_{\x}$ where $\bL_{\x}=L\tn{\W\x}\tn{\ob}$. This implies
\[
{f(0)}-\bL_{\x}\leq f(\x)\leq  \bL_{\x}+{f(0)}
\]
which implies 
\[
{f(0)}-\E[\bL_{\x}]\leq \E[f(\x)]\leq  \E[\bL_{\x}]+{f(0)}
\]
Clearly $\E[\tn{\W\x}]\leq \sqrt{\E[\tn{\W\x}^2]}=\sqrt{\sum_{i=1}^h \tn{\w_i}^2}\leq\sqrt{h}\|\W\|$. This yields
\[
f(0)-\bL\sqrt{p}\leq  \E[\tn{\e(\ob,\W)}]\leq f(0)+\bL\sqrt{p}
\]
Let $s(\rb)=\sum_i \rb_i$. Let $s(\rb)_+,s(\rb)_-$ denote $\max(s(\rb),0)$ and $\min(s(\rb),0)$.
\begin{align}
\E_{\rb}[\sup_{\ob,\W\in\Tc} \sum_{i=1}^n\rb_ie(\ob,\W)]&=\E[\sup_{\ob,\W\in\Tc}s(\rb)e(\ob,\W)]\\
&=\E[\sup_{\ob,\W\in\Tc}s(\rb)_+e(\ob,\W)]-\E[\inf_{\ob,\W\in\Tc}-s(\rb)_-e(\ob,\W)]\\
&\leq  \E[s(\rb)_+](\bL+f(0))+\E[s(\rb)_-](f(0)-\bL)
\end{align}
Using the fact that $\E[s(\rb)_+]=-\E[s(\rb)_-]=\E[|s(\rb)|]/2\leq \sqrt{n}/2$, we find
\begin{align}
\E_{\rb}[\sup_{\ob,\W\in\Tc} \sum_{i=1}^n\rb_ie(\ob,\W)]\leq \E[|s(\rb)|]L\|\W\|\tn{\ob}\sqrt{h}\leq \sqrt{nh}L\|\W\|\tn{\ob}=\sqrt{nh}\bL.\label{rad mean bound}
\end{align}
To address the zero-mean $s(\ob,\W\x_i)$ component, we carry out a standard covering argument. Let $\{\W_i\}_{i\geq 1}\subset \Tc$ be an $R_{\W}\eps$ cover for the set $\Tc$ and $\{\ob_i\}_{i\geq 1}$ be $R_{\ob}\eps/\sqrt{h}$ cover of $R_{\ob}\Bc^h$. Let cover sizes be $N_W$ and $N_o$ respectively and let $N_\eps=N_WN_o$. Since $s(\ob,\W\x_i)$ is zero-mean and subgaussian, conditioned on $\rb_i$, 
\[
s(\ob,\W)=\sum_{i=1}^n\rb_is(\ob,\W\x_i)
\]
is sum of $n$ zero-mean random variables with subgaussian norm at most $\bL$. This implies
\[
\Pro(|s(\ob,\W)|\geq t\bL \sqrt{n})\leq 2\exp(-t^2/2)
\]
Setting $t'=\order{\sqrt{\log N_\eps}}+t$ and union bounding over all $\ob_i,\W_j$ pairs we obtain that with $1-2\exp(-t^2/2)$ probability, all elements of the cover satisfies
\[
|s(\ob_i,\W_i)|\leq \bL\sqrt{n} (\order{\sqrt{\log N_\eps}}+t).
\]
This implies
\[
\E[\sup_{\ob_j,\W_i}|s(\ob_j,\W_i)|]\leq \order{\bL\sqrt{n \log N_\eps}}
\]
What remains is doing the perturbation argument to extend this bound to elements that are not inside the cover. Pick $\ob,\W$ from the constraint set. Let $\hat\ob,\hat\W$ be their closest neighbors from the corresponding covers. Letting $P_W=|s(\hat\ob,\W-\hat\W)|$ and $P_o=|s(\ob-\hat\ob,\W)|$, we will write
\[
|s(\ob,\W)|\leq |s(\hat\ob,\hat\W)|+P_o+P_W
\]
Form the data matrix $\X=[\x_1~\dots~\x_n]$. We have that $\E[\|\X\|]\leq \sqrt{n}+\sqrt{p}\leq 2\sqrt{\max\{n,p\}}$. Consequently, for any $\W$ and its neighbor $\hat\W$, we obtain
\[
\E[\tf{(\W-\hat\W)\X}]\leq\tf{\W-\hat\W}2\sqrt{\max\{n,p\}}\implies \E[\sum_i\tn{(\W-\hat\W)\x_i}]\leq \tf{\W-\hat\W}2\sqrt{n\max\{n,p\}}
\]
Consequently
\[
P_W\leq |\sum_{i=1}^n\hat\ob^T (\sigma(\W\x_i)-\sigma(\hat\W\x_i))|\leq \sum_i\tn{\hat\ob}L\tn{(\W-\hat\W)\x_i}\leq 2\eps R_{\ob}R_{\W} L\sqrt{n\max\{n,p\}}
\]
Similarly, using $\tf{\W}\leq \sqrt{h}\|\W\|$, we have
\begin{align}
P_o&\leq \sum_{i=1}^n\tn{\hat\ob-\ob} \tn{\sigma(\W\x_i)}\leq h^{-1/2}\eps R_{\ob}L\sum_{i=1}^n \tn{\W\x_i}\\
&\leq h^{-1/2}\eps R_{\ob}L\tf{\W}2\sqrt{n\max\{n,p\}}\\
&\leq R_{\ob}LR_{\W}\eps2\sqrt{n\max\{n,p\}}.
\end{align}
Combining these estimates, we obtain $P_W+P_o\leq 4\eps \bL\sqrt{n\max\{n,p\}}$. Overall, for fixed $\rb$, we have
\[
\E_{\x_i}[\sup_{\tn{\ob}\leq R_{\ob},\W\in\Tc}|s(\ob,\W)|]\leq \bL\sqrt{n}(\order{\sqrt{\log N_\eps}}+4\eps\sqrt{\max\{n,p\}})
\]
This is also true for expectation over $\rb$ which implies 
\begin{align}
\E_{\rb}[\E_{\{\x_i\}_{i=1}^n}[\sup_{\ob,\W\in\Tc} \sum_{i=1}^n\rb_i\ob^T s(\ob,\W\x_i)]]\leq \bL\sqrt{n}(\order{\sqrt{\log N_\eps}}+4\eps\sqrt{\max\{n,p\}}).\label{rad nonzero first}
\end{align}
Picking $\eps=C\max\{n,p\}^{-1/2}$, we obtain (recall $s$ is set $\Cc$'s dimension, $(B/\eps)^s$ is covering)
\[
N_\eps = N_oN_W\leq \left({\order{\sqrt{h\max\{n,p\}}}}\right)^h \left(1+\frac{B}{R_{\W}}\sqrt{\max\{n,p\}}\right)^s
\]
which implies 
\begin{align}
 \log N_\eps &\leq \order{h\log \max\{n,p\}}+\order{s(\log (1+\frac{B}{R_{\W}})+\log\max\{n,p\})}\\
 &\leq \order{(h+s)\log \max\{n,p\}}+\order{s\log (1+\frac{B}{R_{\W}})}
\end{align}
Substituting this to \eqref{rad nonzero first}, we find
\begin{align}
\E_{\rb}[\E_{\{\x_i\}_{i=1}^n}[\sup_{\ob,\W\in\Tc} \sum_{i=1}^n\rb_i\ob^T s(\ob,\W\x_i)]]\leq \bL\sqrt{\order{n((h+s)\log \max\{n,p\}+s\log (1+\frac{B}{R_{\W}}))}}.\label{rad nonzero second}
\end{align}
Cumulatively (combining the mean term \eqref{rad mean bound} and zero-mean term \eqref{rad nonzero second}) and normalizing by $n$, we obtain the advertised Rademacher complexity bound
\[
\text{Rad}(\Fc)\leq \bL \order{\frac{(h+s)\log \max\{n,p\}+s\log (1+\frac{B}{R_{\W}})}{n}}^{1/2}.
\]
\end{proof}

\section{Nonlinearity requirements for well-conditioned Hessian}\label{section hessian nonlinear}

\begin{assumption} [Nonlinearity over an interval] Let $g$ be a standard Gaussian and define $\eta(x)=\sigma'(xg)$ for $x\in\R$. Given a range $[\alpha,\beta]$ define $\theta_1,\theta_2$ as
\begin{align}
&\theta_1=\inf_{\alpha\leq x,y\leq \beta}\frac{1}{2}(\var[\eta(x)]+\var[\eta(y)]-\sqrt{(\var[\eta(x)]-\var[\eta(y)])^2+4\E[g\eta(x)]^2\E[g\eta(y)]^2})\\
&\theta_2=\inf_{\alpha\leq x\leq \beta} \var[g\eta(x)]-\E[g^2\eta(x)]^2.
\end{align}
We define $\zeta$ as the minimum i.e. $\zeta=\zeta(\alpha,\beta)=\min\{\theta_1,\theta_2\}$.
\end{assumption}

\begin{lemma} [Orthogonal weight matrix] \label{orthogonal min eigen}Let $\Ws\in\R^{h\times p}$ have orthogonal rows. Suppose singular values of $\Ws$ lie between $[\alpha,\beta]$ for some scalars $\beta\geq \alpha>0$. Then, given $\x\sim\Nn(0,\Iden_p)$, we have that
\[
\bSi(\sigma'(\Ws\x)\bt\x)\succeq  \zeta(\alpha,\beta)\Iden_{hp}.
\]
\end{lemma}
\begin{proof} Let $\Ws=\bSi\Vb^T$ where $\bSi$ is diagonal and $\Vb^T$ have orthonormal rows. Let $\x^1=\Vb^T\x\sim\Nn(0,\Iden_h)$. Also let $\Qb$ be the completion of $\Vb$ to orthonormal basis and let $\xh={\Qb}^T\x$. Let $\s=\text{diag}(\bSi)$. Hence 
\[
\y=\sigma'(\Ws\x)\bt\x=\sigma'(\bSi\x^1)\bt\x=\sigma'(\s\bd \x^1)\bt\x
\]
Consider the $i,j$th submatrix of $\bSi(\y)\in\R^{hp\times hp}$ of size $p\times p$ which is given by
\[
\bSi(i,j)=\E[\sigma'(\s_i\bd \x^1_i)\sigma'(\s_j\bd \x^1_j)\x\x^T]=\Qb\E[\sigma'(\s_i\bd \x^1_i)\sigma'(\s_j\bd \x^1_j)\xh\xh^T]\Qb^T
\]
where we used the fact that $\x=\Qb\xh$. Defining $\y'=\sigma'(\s\bd \x^1)\bt\xh$ and forming unitary matrix $\bar\Qb=\diag(\Qb)\in\R^{hp\times hp}$ this implies
\[
\bSi=\bar\Qb\bSi(\y'){\bar\Qb}^T.
\]
Hence eigenvalue spectrum of $\bSi$ and $\bSi(\y')$ are identical. Now, focusing on $\xh$ and letting $\xh=[\x_1~\x_2]$, $\bSi(\y')=\bSi(\sigma'(\s\bd \x^1)\bt\xh)$ can be written as a $2\times 2$ block matrix $\M$ where $\M_{1,1}=\bSi[\sigma'(\s\bd \x^1)\bt\x^1]$, $\M_{1,2}=\M_{2,1}^T=\E[(\sigma'(\s\bd \x^1)\bt\x^1)(\sigma'(\s\bd \x^1)\bt\x^2)^T]$ and $\M_{2,2}=\bSi[\sigma'(\s\bd \x^1)\bt\x^2]$. Since $\x^2$ is independent of $\x^1$, $\M_{1,2}=\M_{2,1}=0$.

To estimate $\M_{2,2}$ we use
\[
\bSi[\sigma'(\s\bd \x^1)\bt\x^2]=\bSi(\sigma'(\s\bd \x^1))\bt\bSi(\x^2)=\bSi(\sigma'(\s\bd \x^1))\bt\Iden_{p-h}
\]
The minimum singular value of $\bSi(\sigma'(\s\bd \x^1))$ can be lower bounded by writing $\eta=\sigma'(\s\bd \x^1)$ and
\[
\E[\eta\eta^T]-\E[\eta]\E[\eta]^T=\diag(\var(\eta))
\]
which yields $\lambda_{\min}(\M_{2,2})\succeq \min_{i}\var(\eta_i)\succeq \theta_1(\s)$.

Finally, Lemma \ref{m11} shows that $\lambda_{\min}(\M_{1,1})\succeq \min\{\theta_1(\s),\theta_2(\s)\}$. Since $\M_{1,1}$ and $\M_{2,2}$ are block diagonal, we obtain the result.
\end{proof}
\begin{lemma} [Entries of covariance] \label{lemma covar} Let $\bSi=\bSi(\y)=\E[\y\y^T]-\E[\y]\E[\y]^T$ where $\y=\sigma'(\s\bd\x)\bt \x$ where $\x\sim\Nn(\Iden)$. Define the vector $\eta=\sigma'(\s\bd\x)$. Let $\bSi(i,j)\in\R^{h\times h}$ be the $i,j$th submatrix of $\bSi$. We have that
\begin{enumerate}
\item If $k\not\in \{i,j\}$ or $l\not\in \{i,j\}$ and $k\neq l$: $\bSi(i,j)_{k,l}=0$.
\item If $k=l\not\in \{i,j\}$, $i\neq j$: $\bSi(i,j)_{k,l}=\E[\eta_i]\E[\eta_j]$.
\item If $k=l\neq i$, $i= j$: $\bSi(i,j)_{k,l}=\E[\eta_i^2]$.
\item If $k=i$, $l=j$, $i= j$: $\bSi(i,j)_{k,l}=\E[\eta_i^2\x_i^2]-\E[\eta_i\x_i]^2$.
\item If $k=j$, $l=i$, $i\neq j$: $\bSi(i,j)_{k,l}=\E[\eta_i\x_i]^2$.
\item If $k=i$, $l=j$, $i\neq j$: $\bSi(i,j)_{k,l}=0$.
\item If $k=l=i$, $i\neq j$, $i\neq j$: $\bSi(i,j)_{k,l}=\E[\eta_i\x_i^2]\E[\eta_j]$.
\end{enumerate}
\end{lemma}
\begin{proof} These statements all follows from basic properties such as independence and standard Gaussian moments. For the first case, suppose $k\not\in\{i,j\}$.
\[
\bSi(i,j)_{k,l}=\E[\eta_i\eta_j\x_l\x_k]=\E[\eta_i\eta_j\x_l]\E[\x_k]=0.
\]
In the second case, $\bSi(i,j)_{k,l}=\E[\eta_i\eta_j\x_l^2]=\E[\eta_i]\E[\eta_j]$. Third case, $\E[\eta_i^2\x_l^2]=\E[\eta_i^2]\E[\x_l^2]$. Fourth case yields $\E[\eta_i^2\x_i^2]-\E[\eta_i\x_i]^2$. Fifth yields, $\E[\eta_i\eta_j\x_j\x_i]-\E[\eta_i\x_j]\E[\eta_j\x_i]=\E[\eta_i\x_i]^2$. Sixth yields, $\E[\eta_i\eta_j\x_j\x_i]-\E[\eta_i\x_i]\E[\eta_j\x_j]=0$. The last case yields $\E[\eta_i\eta_j\x_i\x_i]=\E[\eta_i\x_i^2]\E[\eta_j]$.
\end{proof}
\begin{lemma} [Analyzing orthogonal weight matrix]\label{m11} Suppose $\x\sim \Nn(0,\Iden_h)$. Define the vector $\y=\eta(\s\bd\x)\bt \x$ where entries of $\s$ lie between $\alpha,\beta$. We have that
\[
\bSi(\y)\succeq \zeta(\alpha,\beta).
\]
\end{lemma}
\begin{proof} 
We first study the $(i,j)$th submatrix of $\bSi=\bSi(\y)=\E[\y\y^T]-\E[\y]\E[\y]^T$ given by
\[
\bSi(i,j)=\E[\eta_i\eta_j\x\x^T]-\E[\eta_i\x]\E[\eta_j\x]^T
\]
Entries of $\bSi(i,j)$ are given by
\[
\bSi(i,j)_{k,l}=\E[\eta_i\eta_j\x_k\x_l]-\E[\eta_i\x_k]\E[\eta_j\x_l]
\]
Observe that $\eta$ has independent entries and only $(\eta_i,\x_i)$ pairs are dependent. Straightforward calculations based on independence and zero-mean in Lemma \ref{lemma covar} reveal that only nonzero entries of $\bSi(i,j)$ are its diagonal and $\bSi(i,j)_{j,i}$. Let us write 
\[
\bSi(i,j)=\bSih(i,j)+\bSit(i,j)
\]
where $\bSih(i,j)$ contain the diagonal entries and $\bSit(i,j)$ contains the $\bSi(i,j)_{j,i}$ entry for $i\neq j$. We first focus on analyzing the singular values of $\bSih\in\R^{h\times h}$ which is composed of $h^2$ blocks with nonzero diagonals. Later on, we argue that, impact of $\bSit$ can be seen as a perturbation on $\bSih$ to obtain $\bSi$.

Lemma \ref{bsih min} shows that $\bSih\succeq \Lambda$ where $\Lambda$ is a diagonal matrix with entries described above. On the other hand, $\bSit$ is a very sparse matrix. For each pair $i\neq j$, we form the $2\times 2$ submatrix of $\bSih+\bSit$ at the entries $[\bSi(i,i)_{j,j}~\bSi(i,j)_{j,i};~\bSi(j,i)_{i,j}~\bSi(j,j)_{i,i}]$. It is easy to verify that nonzero entries of $\bSit$ fall on distinct rows and columns. This ensures that eigenvalues of $\bSih+\bSit$ are the union of eigenvalues of individual submatrices. This $2\times 2$ submatrix of covariance is equal to
\[
\Sb(i,j)=[\bSi(i,i)_{j,j}~\bSi(i,j)_{j,i};~\bSi(j,i)_{i,j}~\bSi(j,j)_{i,i}]=[\var(\eta_i)~\E[\g_i\eta_i]\E[\g_j\eta_j];~\E[\g_i\eta_i]\E[\g_j\eta_j]~\var(\eta_j)]
\]
Consequently, Eigenvalues of $\bSih+\bSit$ are, 
\begin{itemize}
\item lower bounded by diagonal elements of $\Lambda$ if they don't lie on a $2\times 2$ submatrix as elements of $\bSit$ (which is at least $\min\{\theta_1,\theta_2\}$),
\item otherwise lower bounded by the eigenvalues of $\Sb(i,j)$ which is given by,
\[
\lambda_{\min}(\Sb(i,j))\geq  \frac{1}{2}(\var(\eta_i)+\var(\eta_j)-\sqrt{(\var(\eta_i)-\var(\eta_j))^2+4\E[\g_i\eta_i]^2\E[\g_j\eta_j]^2})\Iden\geq \theta_1.
\]
\end{itemize}
The combination implies $\lambda_{\min}(\Lambda+\bSit)\geq \min\{\theta_1,\theta_2\}$.
\end{proof}
\begin{lemma} [Minimum eigenvalue of diagonals]\label{bsih min} Consider the setup of Lemma \ref{m11}. Let $\bSih$ be the nonzero entries of covariance obtained by taking the diagonal entries of $\bSi(i,j)$ for all $i,j$. We have that
\[
\sigma_{\min}(\bSih)\succeq \Lambda\succeq \Iden \min\{\theta_1,\theta_2\},
\]
where $\Lambda$ is a diagonal matrix with entries $\Lambda(i,i)_{i,i}=\var[\g_i\eta_i]-\E[\g_1^2\eta_1]^2$ and for $i\neq j$, $\Lambda(i,i)_{j,j}=\var[\eta_i]$.
\end{lemma}
\begin{proof} To analyze $\bSih$, we will write it as sum of $h$ matrices of size $h\times h$. The $i$th matrix $\M_i\in\R^{h\times h}$ will correspond to the submatrix corresponding to entries $\{i,i+h,i+2h,\dots,i+(h-1)h\}\times \{i,i+h,i+2h,\dots,i+(h-1)h\}$. Defining the operation that maps $\M_i$ to the corresponding submatrix of $\bSi$ as $\map(\cdot)$, we write $\bSih=\sum_{i=1}^h\map(\M_i)$. Finally, since $\M_i$'s correspond to nonoverlapping entries, $\map(\M_i)$'s are orthogonal and eigenvalues of $\bSih$ is simply the set of eigenvalues of $\{\map(\M_i)\}_{i=1}^h$.

Consequently, without losing generality, we analyze the eigenvalue of $\M=\M_1$. First, let us write the entries of $\M$. Following from Lemma \ref{lemma covar}
\begin{itemize}
\item $\M_{1,1}=\E[\g_1^2\eta_1^2]-\E[\g_1\eta_1]^2$
\item For $i\neq 1$: $\M_{1,i}=\E[\g_1^2\eta_1\eta_i]-\E[\g_1\eta_1]\E[\g_1]\E[\eta_i]=\E[\g_1^2\eta_1]\E[\eta_i]$
\item For $i\neq 1$: $\M_{i,i}=\E[\g_1^2\eta_i^2]-\E[\g_1\eta_i]=\E[\g_1^2]\E[\eta_i^2]$.
\item For $i,j\neq 1$: $\M_{i,i}=\E[\g_1^2\eta_i\eta_j]-\E[\g_1\eta_i]\E[\g_1\eta_j]=\E[\g_1^2]\E[\eta_i]\E[\eta_j]$.
\end{itemize}
Now, we will decompose $\M$ into $4$ components namely $\M=[\M_{1,1}~\M_{1,2:};\M_{2:,1}~\M_{2:,2:}]$. Set $\e=\E[\eta_{2:}]$. First observe that
\begin{align}
&\M_{1,2:}^T=\M_{2:,1}=\E[\g_1^2\eta_1]\e\\
\end{align}
Next, for $i\neq 1$, using $\E[\g_1^2]\E[\eta_i^2]=\E[\g_1]^2(\var[\eta_i]+\E[\eta_i]^2)$, we decompose $\M_{2:,2:}=\Db+\Cb$ where $\Db$ is a diagonal matrix with diagonal entries $\text{diag}(\E[\g_1^2]\var[\eta_{2:}])$ and
\[
\Cb =\E[\g_1]^2\e\e^T.
\]
Now, observe that $\Db$ is already positive semidefinite by definition. Next, we will show that remainder components are PSD as well. Define the quantity 
\[
a=\frac{\E[\g_1^2\eta_1]^2}{\E[\g_1^2]}
\]
Observe that the matrix
\begin{align}
&[a~\M_{1,2:};\M_{2:,1}~\Cb]=[\sqrt{a}~\sqrt{\E[\g_1^2]}\e]^T[\sqrt{a}~\sqrt{\E[\g_1^2]}\e]\succeq 0.
\end{align}
Now, subtracting this from remaining component, we obtain
\begin{align}
[\M_{1,1}~\M_{1,2:};\M_{2:,1}~\M_{2:,2:}]&=[\M_{1,1}-a~0_{2:}^T;0_{2:}~\M_{2:,2:}-\Cb]+[a~\M_{1,2:};\M_{2:,1}~\Cb]\\
&\succeq [\M_{1,1}-a~0_{2:}^T;0_{2:}~\Db] \succeq \Iden_{h} \min_{i\leq h}\{\M_{1,1}-a,\Db_{i,i}\}
\end{align}
Overall we showed that, $\M$ is lower bounded by a diagonal matrix in terms of PSDness. This diagonal matrix has first entry
$\M_{1,1}-a=\E[\g_1^2\eta_1^2]-\E[\g_1\eta_1]^2-\E[\g_1^2\eta_1]^2$,
and the remaining entries are
\[
\E[\eta\eta^T]-\E[\eta]\E[\eta]^T=\diag(\var(\eta)).
\]
Combining all $\M_i$'s to form $\Lambda$, we achieve the advertised result.
\end{proof}

\subsection{Covariance bound for nonorthogonal weight matrix}
The next lemma addresses the minimum eigenvalue of the covariance of $\rho(\Ws;\x)$ which is crucial for ensuring the expected Hessian $\E[H_1]$ is positive definite. We do this by borrowing Lemma $\text{D.}6$ of \cite{zhong2017recovery} and making some adjustments and improvements for our purposes. 
\begin{definition}\label{rand quant} Define $\ob_{\min}=\min_{i=1}^h|\ob_i|,~\ob_{\max}=\max_{i=1}^h|\ob_i|$. Denoting $i$th largest singular value by $\s_i(\cdot)$ define
\[
\Lambda(\Ws)=\prod_{i=1}^h(\s_i(\Ws)\big/\s_{\min}(\Ws))
\]
\end{definition}
\begin{lemma} [Minor variation of Lemma $\text{D.}6$ of \cite{zhong2017recovery}]\label{min eigen}Suppose $\sigma(\cdot)$ satisfies Assumption \ref{actassume} with $\theta=\s_{\min}(\W^*)$. Then, 
\[
\s_{\min}(\cov(\rho(\Ws;\x)))\geq \ob^2_{\min}\Lambda^{-1} \zeta(\s_{\min})/\kappa^{2}\succeq \ob^2_{\min}\zeta(\s_{\min})/\kappa^{h+2} .
\]
where $\kappa=\kappa(\Ws)$.
\end{lemma}
\begin{proof} The proof of this lemma directly follows that of \cite{zhong2017recovery}. The only caveat is that we are interested in covariance rather than correlation matrix $\E[\rho(\Ws;\x)\rho(\Ws;\x)^T]$ which includes the mean. The proof for covariance work in the exact same manner, however, we need to slightly modify one of the estimates in the proof of Lemma $\text{D.}6$ to account for $\var(f)$ rather than $\E[f^2]$. In particular, Lemma $\text{D.}6$ of \cite{zhong2017recovery} considers the function $f=f(\x)=\rho(\Ws;\x)\vbb$ for some vector $\vbb\in\R^{hp}$
and lower bounds $\E[f^2]$.

We will simply show that same strategy lower bounds the variance and rest of the proof is identical. The challenge is the fact that $\Ws\x$ does not have i.i.d. entries and we overcome this issue by transforming the expectation integral from a Gaussian vector with dependent entries to a Gaussian vector with independent entries. Given a vector $\vbb$, we study
\[
\var(f)=\var[\rho(\Ws;\x)\vbb]
\]
Set $e=\E[f(\x)]=\E[\rho(\Ws;\x)\vbb]$ and $\s=\Ws\x$. Let $\Ws$ have right singular vectors ${\mtx{R}}$ so that $\Ws=\Ub{\mtx{R}}$ for some $\Ub\in\R^{h\times h}$. Let $\xh={\mtx{R}}\x$ so that $\Ws\x=\Ub\xh$. To avoid repetition, we will provide the argument when rows of $\Vb$ (matricized $\vbb$) is spanned by ${\mtx{R}}^T$ i.e. $\Vb=\Vb{\mtx{R}}^T{\mtx{R}}$; but the exact same idea can be adapted for general $\Vb$. Let $\Vb_{R}=\Vb=\Vb{\mtx{R}}^T$. Let $\s_{\min}=\s_{\min}(\Ub)=\s_{\min}(\Ws)$.
\begin{align}
\var(f)&=\E[(\rho(\Ws;\x)\vbb)^2-\E[\rho(\Ws;\x)\vbb]^2]=\E[(\rho(\Ws;\x)\x^T\vb_i-e)^2]\\
&=\E[(\sigma'(\Ws\x)\Vb\x-e)^2]=\E[(\sigma'(\Ub\xh)\Vb_{R}\xh-e)^2]~~~\text{where}~~~\xh\sim\Nn(0,\Iden_h)\\
&=\int_{\xh} (2\pi)^{-h/2}( \sigma'(\Ub\xh)^T\Vb_{R}\xh-e)^2\exp(-\tn{\xh}^2/2)d\xh\\
&=\int_{\s=\Ub\xh} (2\pi)^{-h/2}(\sigma'(\s)^T\Vb\Ub^\dagger\s-e)^2\frac{\exp(-\tn{\Ub^\dagger\s}^2/2)}{|\text{det}(\Ub^\dagger)|}d\s~~~\text{where}~~~\s\sim\Nn(0,\Ub\Ub^T)\\
&\geq \int_{\s} (2\pi)^{-h/2}(\sigma'(\s)^T\Vb\Ub^\dagger\s-e)^2\frac{\exp(-\s_{\min}^{-2}\tn{\s}^2/2)}{|\text{det}(\Ub^\dagger)|}d\s\\
&\geq \int_{\z=\s/\s_{\min}} (2\pi)^{-h/2}(\sigma'(\s_{\min}\z)^T\Vb\Ub^\dagger\s_{\min}\z-e)^2\frac{\exp(-\tn{\z}^2/2)\s_{\min}^{h}}{|\text{det}(\Ub^\dagger)|}d\z
\end{align}
Recalling singular values of $\Ub$ are same as $\Ws$ and Definition \ref{rand quant}, we have $\Lambda^{-1}=\frac{\s_{\min}^{h}}{|\text{det}(\Ub^\dagger)|}=\frac{\s_{\min}^{h}}{\prod_{i=1}^h\s_i(\Ws)}$. Consequently, defining $\Pb=\s_{\min}\Vb\Ub^\dagger$ and ${\bar{\pb}}=\text{vec}(\Pb)$. We have $\tf{\Pb}^2\geq \tf{\Vb}^2/\kappa^{2}(\Ws)$. Consequently
\begin{align}
\var(f)&\leq \Lambda^{-1} \int_{\z\sim\Nn(0,\Iden)} (2\pi)^{-h/2}(\sigma'(\s_{\min}\z)^T\Vb\Ub^\dagger\z-e)^2\exp(-\tn{\z}^2/2)d\z\\
&=\Lambda^{-1} \E_{\z\sim\Nn(0,\Iden)}[(\sigma'(\s_{\min}\z)^T\Pb \z-e)^2]\\
&\geq \Lambda^{-1}\var_{\z\sim\Nn(0,\Iden)}[\sigma'(\s_{\min}\z)^T\Pb \z]\\
&=\Lambda^{-1}\var_{\z\sim\Nn(0,\Iden)}[\rho(\Iden;\z)^T{\bar{\pb}}]\\
&\geq\Lambda^{-1} \sigma_{\min}(\cov(\rho(\Iden;\z)\rho(\Iden;\z)^T))\tn{\bar{\pb}}^2\\
&\geq \Lambda^{-1}\kappa^{-2}\sigma_{\min}(\cov(\rho(\Iden;\z)\rho(\Iden;\z)^T))\tn{\vbb}^2
\end{align}
This way we related covariance of $\rho(\Ws;\x)$ to the covariance of $\rho$ with identity matrix which is bounded in Lemma $D.4$ of \cite{zhong2017recovery}. We remark that \cite{zhong2017recovery} states the bound for $\E[\rho(\Iden;\z)\rho(\Iden;\z)^T]$ but $\cov(\cdot)$ obeys the same. In fact, this can be concluded by specializing Lemma \ref{orthogonal min eigen} to the identity weight matrix where all singular values are identical.
%
\end{proof}

\subsection{Softplus nonlinearity}\label{softplus appendix}
\begin{lemma} Consider the softplus function $\sigma(x)=\log(1+\exp(x))$. For some $C>0$ and for all $\theta>C$, $\zeta(\theta)>0.05$.
\end{lemma}
\begin{proof} We will use the fact that softplus is a ReLU approximation. Denote $\zeta$ corresponding to ReLU and softplus by $\zeta_R$ and $\zeta_S$ respectively. From \cite{zhong2017recovery}, we know that $\zeta_{R}(\theta)>0.09$ for all $\theta$. We will show that $|\zeta_{R}(\theta)-\zeta_{S}(\theta)|<0.04$ for $\theta>C$.

Observe that $\sigma'(x)=\frac{1}{1+\exp(-x)}$ which implies $\sigma'(x)\rightarrow 0$ as $x\rightarrow-\infty$ and $\sigma'(x)\rightarrow 1$ as $x\rightarrow\infty$. Let $\mu(x)$ be the standard step function: $\mu(x)=(\text{sign}(x)+1)/2$. 
Observe that $|\sigma'(x)^a-\mu(x)|\leq a\exp(-|x|)$ for integers $a\geq 1$. Let $R=\sqrt{\theta}$. Define
\[
\text{diff}(a,b)=\E[\sigma'(\theta g)^ag^b]-\E[\mu(\theta g)^ag^b],~\text{sum}(a,b)=\E[\sigma'(\theta g)^ag^b]+\E[\mu(\theta g)^ag^b]
\]
where $a,b$ will be integers in $\{0,1,2\}$. For any $\text{diff}(a,b)$ term, we write that
\[
\text{diff}(a,b)=\int_{|x|>\frac{R}{\theta}}(\sigma'(\theta x)^a-\mu(x))x^bp_{\Nn(0,1)}(x)dx+\int_{|x|\leq \frac{R}{\theta}}(\sigma'(\theta x)^a-\mu(x))x^bp_{\Nn(0,1)}(x)dx.
\]
We have the following bounds for right side. Using $|\sigma'(\theta x)^a-\mu(x)|\leq 2$,
\[
\left|\int_{|x|\leq \frac{R}{\theta}}(\sigma'(\theta x)^a-\mu(x))x^bp_{\Nn(0,1)}(x)dx\right|\leq \order{\frac{R}{\theta}}^{b+1}
\]
For the $|x|\geq R$ component, we have
\begin{align}
\left|\int_{|x|\leq \frac{R}{\theta}}(\sigma'(\theta x)^a-\mu(x))x^bp_{\Nn(0,1)}(x)dx\right|&\leq aC_0 \int_{R/\theta}^{\infty}\exp(-\theta x)x^bdx\\
&\leq a\theta^{-(b+1)}C_0\int_{R}^\infty\exp(-x)x^bdx
\end{align}
Since $R=\sqrt{\theta}$ can be chosen large enough, we can ensure $\exp(-x)x^b<\exp(-x/2)$. Hence we obtain $\int_{R}^\infty\exp(-x)x^bdx\leq2\exp(-R/2)$. Combining
\begin{align}
\left|\E[\sigma'(\theta g)^ag^b]-\E[\mu(\theta g)^ag^b]\right|&\leq \order{\frac{R}{\theta}}^{b+1}+ a2\exp(-R/2)\theta^{-(b+1)}C_0.\\
&=\order{\theta}^{-(b+1)/2}+ a2\exp(-\sqrt{\theta}/2)\theta^{-(b+1)}C_0:=f(a,b,\theta).
\end{align}
By definition, $\zeta_{R}(\theta)-\zeta_{S}(\theta)$ can be written in terms of $\text{diff}(a,b)$ and $\text{sum}(a,b)$. For instance,
\[
(\var[\sigma'(\theta g)]-\E[\sigma'(\theta g)g]^2)-(\var[\mu(g)]-\E[\mu(g)g]^2)=\text{diff}(2,0)-\text{diff}(1,0)\text{sum}(1,0)-\text{diff}(1,1)\text{sum}(1,1).
\]
Since $\text{sum}(a,b)$ terms are $\order{1}$, we obtain
\[
|\zeta_R(\theta)-\zeta_S(\theta)|\leq \order{f(2,0,\theta)+f(1,0,\theta)+f(1,1,\theta)+f(2,2,\theta)+f(1,1,\theta)+f(1,2,\theta)}.
\]
To conclude, use the fact that $f(a,b,\theta)\rightarrow 0$ as $\theta\rightarrow \infty$ hence for some $C>0$ and for all $\theta>C$, $|\zeta_R(\theta)-\zeta_S(\theta)|<0.04$ as desired
%
%
\end{proof}

\end{document}